\documentclass{article}

\usepackage[preprint]{neurips_2026}

\usepackage[utf8]{inputenc} 
\usepackage[T1]{fontenc}    
\usepackage{hyperref}       
\usepackage{url}            
\usepackage{booktabs}       
\usepackage{amsfonts}       
\usepackage{nicefrac}       
\usepackage{microtype}      
\usepackage{xcolor}         

\usepackage{microtype}
\usepackage{graphicx}
\usepackage{subcaption}
\usepackage{booktabs} 

\usepackage{amsmath}
\usepackage{amssymb}
\usepackage{mathtools}
\usepackage{amsthm}

\usepackage[capitalize,noabbrev]{cleveref}

\theoremstyle{plain}
\newtheorem{theorem}{Theorem}[section]

\newtheorem{lemma}[theorem]{Lemma}

\theoremstyle{definition}

\theoremstyle{remark}

\usepackage[textsize=tiny]{todonotes}

\usepackage{url}            

\usepackage{amsfonts}       

\usepackage{bm}
\usepackage{algorithm}
\usepackage{algorithmic}
\usepackage{multirow}
\usepackage{array}
\usepackage{tabularx}
\usepackage{wrapfig}
\usepackage{xcolor}
\usepackage{enumitem}
\newcolumntype{H}{>{\setbox0=\hbox\bgroup}c<{\egroup}@{}}
\renewcommand{\eqref}[1]{(\ref{#1})}

\usepackage{makecell}

\renewcommand{\algorithmicrequire}{\textbf{Input:}}
\renewcommand{\algorithmicensure}{\textbf{Output:}}

\title{GraphVec: Cross-Domain Graph Vectorization for Graph-Level Representation Learning}

\author{%
  Qi Feng\qquad Jicong Fan\\
  School of Data Science\\
  The Chinese University of Hong Kong, Shenzhen, Guangdong 518172, China\\
  \texttt{qifeng@link.cuhk.edu.cn, fanjicong@cuhk.edu.cn} \\
}

\begin{document}

\maketitle



\begin{abstract}
Learning universal graph representations across heterogeneous domains is difficult because graph datasets differ in topology, node-attribute semantics, feature dimensions, and even attribute availability. We propose \textbf{GraphVec}, a language-model-free graph vectorization model that maps diverse graphs into transferable fixed-dimensional embeddings for graph-level tasks. Instead of directly using incomparable raw node attributes, GraphVec constructs multi-scale global graphs over all nodes in each dataset and extracts spectral embeddings to obtain domain-agnostic relational features. To make these spectral features comparable across datasets, we introduce a density-maximization mean alignment algorithm over orthogonal transformations and prove its monotonic convergence. GraphVec further combines a GIN--Graph Transformer backbone with a multi-layer reference distribution module, which preserves node-level distributional information beyond standard pooling. We also provide a generalization error bound for the proposed model. Experiments on 13 datasets with more than 15 comparison methods demonstrate that GraphVec consistently outperforms strong graph pretraining baselines in cross-domain few-shot graph classification and graph clustering. Beyond graph-level tasks, GraphVec also yields strong node-level representations, achieving competitive performance on few-shot node classification against representative graph prompt learning methods.
\end{abstract}

\section{Introduction}
\label{sec:1}
\vspace{-5pt}
Graph data is a fundamental and widely prevalent form of structured data, representing entities as nodes and their relationships as edges. It plays a crucial role in diverse domains, including social networks, biological systems, citation networks, recommendation systems, and knowledge graphs. Given its ability to model complex relational patterns, graph data analysis has become a key focus in machine learning and data mining. In node-level tasks, the training set is usually a single but large graph, on which each node represents a sample. Node-level tasks include
node embedding or representation \citep{grover2016node2vec,cai2018comprehensive}, node classification \citep{kipf2016semi}, node clustering \citep{wang2023overview}, link prediction \citep{martinez2016survey}, etc. For example, node classification might involve categorizing users in a social network, while link prediction could be used to recommend new connections.

On the other hand, graph-level tasks operate on entire graphs, where a dataset is composed of numerous graphs, and each graph is treated as a sample. Graph-level tasks address broader challenges such as graph comparison \citep{kobler2012graph}, representation learning \citep{sun2020infograph}, classification \citep{xu2019powerful}, clustering \citep{ijcai2024p417}, generation \citep{liao2019efficient}, etc. Graph comparison often relies on graph kernels \citep{gartner2003graph,vishwanathan2010graph,shervashidze2011weisfeiler} or distances \citep{bunke1997relation,zeng2009comparing,memoli2011gromov,bento2018family} or deep learning methods \citep{sun2024mmd} to measure similarity between different graphs, while graph representation learning aims to encode entire graphs into compact, informative embeddings for downstream tasks \citep{you2020graph,you2021graph,sun2024lovasz}. Representative graph-level representation methods range from substructure-based graph embeddings such as graph2vec \citep{narayanan2017graph2vec} to self-supervised mutual-information and contrastive approaches such as InfoGraph, GraphCL, automated graph contrastive learning, and GCC \citep{sun2019infograph,you2020graph,you2021graph,qiu2020gcc}.
Graph classification, for example, is critical in chemistry for predicting molecular properties \citep{gilmer2017neural,NEURIPS2024_GRDL}, whereas graph generation enables the creation of novel structures, such as drug-like molecules in computational biology \citep{hoogeboom2022equivariant}.

Most of the aforementioned methods are dataset-specific. That means, for one dataset, we have to train a new model, e.g., a graph neural network, to solve the corresponding problem. This leads to the following two limitations. First, training a model from scratch is very time-consuming, and it requires model selection and parameter tuning, which brings inconvenience to practical applications. Second, knowledge from historical data or tasks in the same domain or similar domains cannot be exploited well. These limitations motivate graph pretraining, which has attracted increasing attention following the success of pretraining in natural language processing and computer vision. A representative line of work is graph prompt learning \citep{sun2022gppt, fang2023universal, sun2023all, liu2023graphprompt, fu2025edge}, which aims to adapt a pretrained graph model to different downstream tasks by reducing the gap between pretraining objectives and task-specific objectives. However, most graph prompt methods are still developed and evaluated within single-dataset settings, where different tasks share the same graph distribution or feature space. Recently, graph foundation models (GFMs) \citep{liu2025graph} have also become an increasingly prominent area of research in graph data analysis due to their ability to pre-train on diverse datasets to enhance performance across multiple tasks and domains. Existing GFMs can be roughly grouped into LLM-based methods \citep{liu2023one, konggofa, xia2024anygraph} and non-LLM-based methods \citep{sun2025riemanngfm, yuan2025much}. Emerging studies \citep{galkin2023towards, zheng2023you} indicate that GFMs exhibit strong generalization capabilities, even when applied to previously unseen graph structures.


One key challenge in pretraining graph models is that graph patterns from different domains exhibit significant variation \citep{galkin2023towards}, which is evident in both structural and feature representations. 
For example, in molecular graphs \citep{yang2016revisiting}, the structure encodes 3D spatial arrangements and atomic bonds, 
while node features represent chemical properties. 
Conversely, in social networks \citep{dwivedi2023benchmarking}, the structure reflects user connections, 
and node features correspond to user profiles.
These distribution differences make it difficult for a single model to learn domain-agnostic representations. 
A promising approach involves transforming both graph structures and node features into textual formats, 
then employing large language models (LLMs) to derive unified representations \citep{fatemi2023talk,liu2023one,tang2024graphgpt,wang2024can}. 
Another approach is to improve existing graph learning paradigms \citep{liu2025graph} through innovations in the aspects of the backbone \citep{rong2020self}, pretraining \citep{qiu2020gcc,you2020graph,yu2025samgpt}, and adaptations \citep{fu2025edge,yu2025gcot,wang2025multi}. 


Despite recent progress in graph pretraining, several important limitations remain. First, LLM-based graph pretraining methods convert graphs into textual descriptions, which can discard fine-grained topological patterns and node-feature information~\citep{yu2025samgpt}. They also incur substantial computational costs due to the large scale of LLMs. Second, many existing pretrained graph models are primarily designed for node-level tasks~\citep{zhao2024graphany,wei2024llmrec,wang2025multi}, while graph-level tasks have received comparatively less attention~\citep{yu2025gcot,fu2025edge}. Third, it is still challenging to train a single graph model that can handle heterogeneous graph datasets from diverse domains and generalize reliably to unseen domains and downstream graph-level tasks.

This work proposes a graph representation model for graph-level tasks across diverse domains. Figure \ref{fig_flowchart} shows the flowchart. Our contributions are:
Our contributions are summarized as follows.
\begin{itemize}[noitemsep,leftmargin=10pt, partopsep=0pt]
\vspace{-15pt}
    \item We propose GraphVec, a language-model-free cross-domain graph vectorization framework for graph-level tasks. GraphVec maps graphs from heterogeneous domains into a shared fixed-dimensional representation space without assuming comparable raw node attributes.
    \item We introduce a global multi-graph feature construction strategy that converts domain-specific node attributes into multi-scale relational spectral features. This provides a unified input representation for graphs with different feature dimensions, semantics, or missing attributes.
    \item We develop a density-maximization orthogonal alignment algorithm for spectral features and establish monotonic improvement and asymptotic stationarity of the alignment objective.
    \item We design a multi-layer reference distribution learning mechanism for distribution-aware graph representation. This mechanism preserves information that is typically discarded by standard graph-level pooling.
    \item We provide extensive experiments on cross-domain few-shot graph classification and graph clustering, together with ablations, robustness tests, scalability analysis, and comparisons with recent graph pretraining and prompt-tuning methods.
\end{itemize}
\vspace{-5pt}

\section{Related Work}
\vspace{-7pt}
\paragraph{Graph pretraining}
Many studies have explored the ``pre-train and adaptation" paradigm for graph models, leveraging message-passing or transformer-based GNN backbones. These approaches typically employ contrastive or generative self-supervised learning for pretraining, followed by fine-tuning a subset of model parameters to adapt to downstream tasks or datasets \citep{liu2025graph}. Contrastive methods, including  GCC \citep{qiu2020gcc}, InfoGraph \citep{sun2019infograph}, DGI \citep{velivckovic2019deep}, SimGRACE \citep{xia2022simgrace}, and GCOPE \citep{zhao2024all}, maximize agreement between augmented views to learn transferable representations, while generative methods \citep{hou2022graphmae,hou2023graphmae2} pre-train via graph reconstruction or property prediction. Recently, graph prompt tuning \citep{sun2022gppt, fang2023universal, sun2023all, liu2023graphprompt, fu2025edge} has emerged to bridge the pretraining–downstream gap, and many recent GFMs adopt this paradigm \citep{yuan2025much, yu2025non}. However, these works mainly emphasize adaptation, leaving the problem of learning unified graph representations underexplored. Due to space limitations, we defer more discussion of \textbf{Graph pretraining} and the related work on \textbf{GFMs} to Appendix \ref{app_related_work}.

\vspace{-5pt}
\section{Methodology}
\vspace{-5pt}
\subsection{Cross-Domain Graph Vectorization: Problem Setup}
\vspace{-5pt}
First of all, the major notations used in this paper are shown in Table \ref{tab:notation}.
Let $\mathcal{D}=\{\mathcal{G}_1,\mathcal{G}_2,\ldots,\mathcal{G}_M\}$ be a union of $M$ datasets of labeled graphs from $M$ different domains, where 
$\mathcal{G}_j=\{(G^{(j)}_1,y_1^{(j)}),(G^{(j)}_2,y_2^{(j)}),\ldots,(G^{(j)}_{N_j},y_{N_j}^{(j)})\}$. 
Here, each graph $G^{(j)}_{i}$ is denoted as $G^{(j)}_{i}=(\mathbf{A}_i^{(j)}, \mathbf{X}_i^{(j)}, y_i^{(j)})$, where $\mathbf{A}_i^{(j)}\in\mathbb{R}^{n_i^{(j)}\times n_i^{(j)}}$ denotes the adjacency matrix, $\mathbf{X}_i^{(j)}\in\mathbb{R}^{n_i^{(j)}\times d^{(j)}}$ denotes the node attribute matrix, $n_i^{(j)}$ denotes the number of nodes, $d^{(j)}$ denotes the number of attributes, and $y_i^{(j)}$ denotes the graph label. Our goal is to use $\mathcal{D}$ to train a cross-domain graph representation model, denoted as
\begin{equation}
    F:\mathbb{G}\rightarrow \mathbb{R}^r
\end{equation}
to represent any graph from the space $\mathbb{G}$ as an $r$-dimensional vector that is useful in downstream tasks such as graph classification, where $\mathbb{G}$ denotes the set of all graphs in the form of $(\mathbf{A},\mathbb{\mathbf{X}})$. Therefore, $F$ serves as a universal graph representation model.

To learn $F$ from $\mathcal{D}$, we need to address these challenges:
\vspace{-5pt}
\begin{itemize}[noitemsep,leftmargin=2pt]
    \item \textbf{Attributes inconsistency} The node attributes of graphs from different domains are different and not comparable at all. Thus the node attributes in $\mathcal{D}$ cannot be fed into $F$ directly.
    \item \textbf{Attributes absence} Many graph datasets do not contain node attributes, making them very different from graph datasets with node attributes. The heuristic method of constructing node attributes, such as using node degrees, does not comply with the semantic attributes of other datasets.  
    \item \textbf{Information loss in pooling} Although there have been a few advanced graph pooling methods \citep{liu2022graph}, converting nodes' embeddings into a single vector cannot fully utilize the information.
\end{itemize}
\vspace{-5pt}

\begin{figure*}[t]
    \centering
    \includegraphics[width=0.9\linewidth]{ 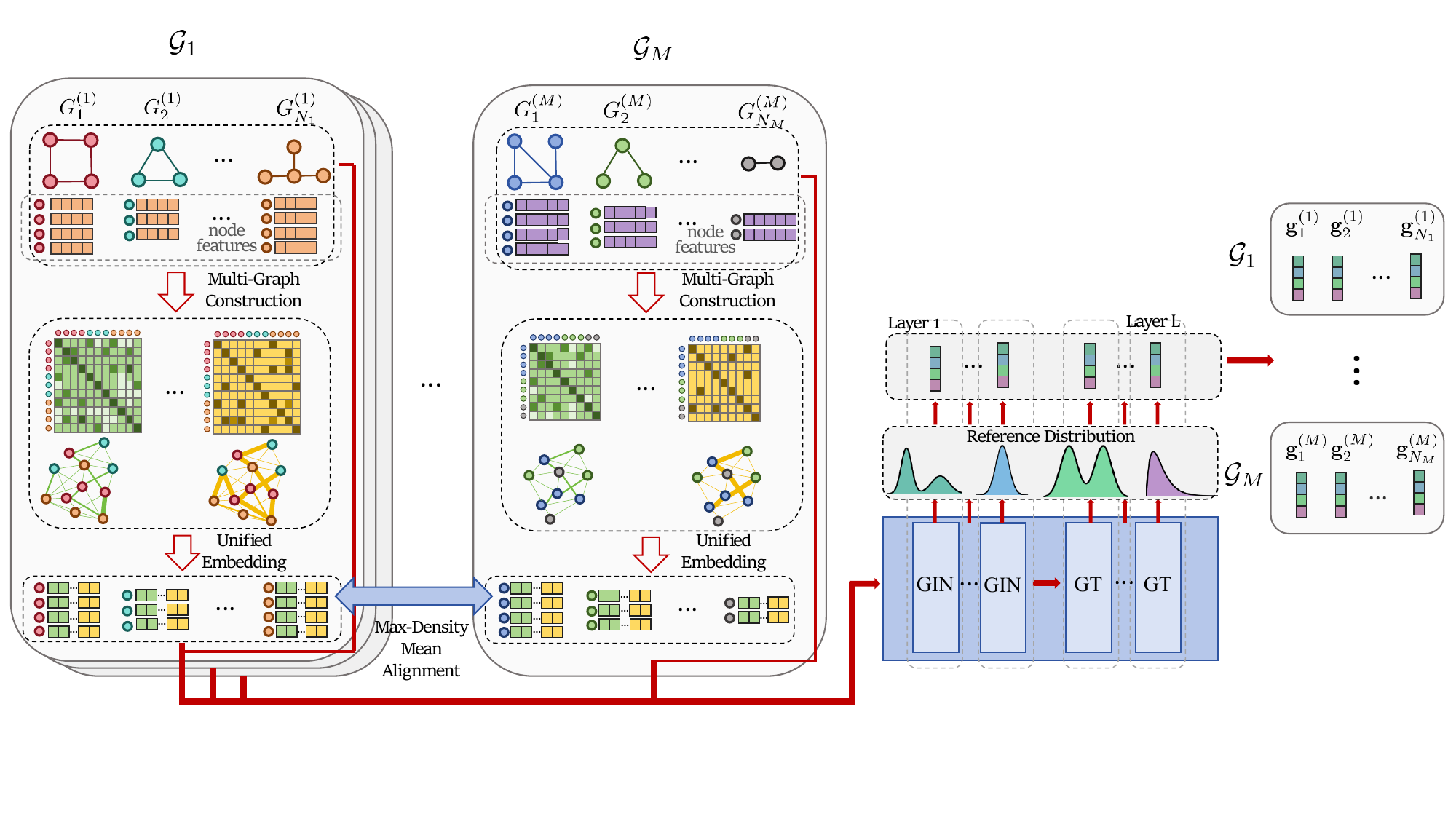}
    \vspace{-20pt}
    \caption{Flow-chart of GraphVec. $\mathcal{G}_1,\ldots,\mathcal{G}_M$ are $M$ datasets from different domains. The model represents each graph $G_{i}^{(j)}$ as a single vector $\mathbf{g}_i^{(j)}$ that can be used in graph-level downstream tasks.}
    \label{fig_flowchart}
    \vspace{-10pt}
\end{figure*}

\subsection{Domain-Agnostic Relational Features via Global Multi-Graphs}\label{sec_graph}
\vspace{-5pt}
As mentioned, the node features of different domain graph datasets may vary significantly in semantics and dimensions. 
To capture domain-invariant features, we focus on the relationships among nodes across the entire dataset rather than the original features. The reason is that in many machine learning problems, using the relationships between samples or a graph constructed from the dataset can provide effective solutions. For instance, in spectral clustering \citep{ng2001spectral}, we use a similarity graph rather than the original features; in kernel support vector machine \citep{cortes1995support}, we can use a Gaussian kernel matrix, which is a similarity matrix of the data points.  

\textbf{Multi-Scale Global Relational Graphs}~~ For each graph dataset $\mathcal{G}_j$, $j\in[M]$, we propose to construct a similarity graph over all nodes in the dataset using a Gaussian kernel function, i.e., 
\begin{equation}\label{eq_K_construct}
\mathbf{K}_{\lambda}^{(j)} = \left[\exp \left( -\frac{\left\| \mathbf{x}_u - \mathbf{x}_{v} \right\|_2^2}{2\lambda\mu^2} \right)\right]_{u,v=1}^{\bar{N}_j}, \quad\bar{N}_j=\sum_{i=1}^{N_j}{n_i^{(j)}}
\end{equation}
where $\mathbf{x}_i$ is the $i$-th row of $\mathbf{X} :=  \big|\big|_{i=1}^{N_j}\mathbf{X}_i^{(j)}$ (vertical concatenation), 
$\mu$ is the mean of the pairwise distances between all nodes in the dataset, and 
$\lambda$ controls the bandwidth of the kernel. This setting ensures \textit{translation, rotation, and scaling invariance}, which is important to extract comparable features across diverse datasets. $\mathbf{K}_{\lambda}^{(j)}$ is the adjacency matrix of this global graph of the nodes in $\mathcal{G}_j$.
Note that a single $\mathbf{K}_{\lambda}^{(j)}$ exploits partial information of the node attributes of $\mathcal{G}_j$ and the optimal setting of $\lambda$ remains an open problem. Therefore, we use a number of different values for $\lambda$, e.g. $\lambda_1,\lambda_2,\ldots,\lambda_Q$, to construct \textbf{multiple global graphs} for the nodes in $\mathcal{G}_j$:
\begin{equation}\label{eq_K_construct_Q}
    \boldsymbol{\mathcal{K}}^{(j)}:=\left\{\mathbf{K}_{\lambda_1}^{(j)}, \mathbf{K}_{\lambda_2}^{(j)},\ldots,\mathbf{K}_{\lambda_Q}^{(j)}\right\},\quad j\in[M]
\end{equation}
Note that the diversity of $\boldsymbol{\mathcal{K}}^{(j)}$ can be further enhanced if more kernel families, e.g., $k(\mathbf{x}_u,\mathbf{x}_v)=\exp(-\alpha\|\mathbf{x}_u-\mathbf{x}_v\|_1)$, are considered. 

\vspace{-5pt}
\textbf{Multi-Scale Relational Spectral Embedding}~~ For $\mathbf{K}_{\lambda_q}^{(j)}$, we compute $\bar{d}$-dimensional node embeddings using singular value decomposition (SVD):
\begin{equation}\label{eq_Z_KSVD}
    \mathbf{Z}_{\lambda_q}^{(j)} = \mathbf{U}_{\bar{d}}\boldsymbol{\Sigma}_{\bar{d}}^{1/2},\quad ~\mathbf{K}^{(j)}_{\lambda_q}=\mathbf{U}\boldsymbol{\Sigma}\mathbf{V}^\top
\end{equation}
where $\mathbf{U}_{\bar{d}}\in\mathbb{R}^{\bar{N}_j\times \bar{d}}$ is composed of the first $\bar{d}$ columns of $\mathbf{U}$ and $\boldsymbol{\Sigma}_{\bar{d}}$ is a diagonal matrix consisting of the first (largest) $\bar{d}$ singular values.
Then the final node feature matrix is obtained by concatenating embeddings from all scales, i.e.,
{\small
\begin{equation}\label{eq_ZZZ}
    \begin{aligned}
        \mathbf{Z}^{(j)} &= \left[\mathbf{Z}_{\lambda_1}^{(j)},\ldots,\mathbf{Z}_{\lambda_Q}^{(j)}\right]
                   =\left[
                   \begin{matrix}
                       \mathbf{Z}^{(j)}_1\\
                       \vdots\\
                       \mathbf{Z}^{(j)}_{N_j}
                   \end{matrix}\right] \in \mathbb{R}^{\sum_{i=1}^{N_j}{n_i^{(j)}}\times Q\bar{d}},
    \end{aligned}
\end{equation}
}
where $j\in[M]$.
For datasets without node attributes, we generate node attributes using the truncated SVD of the self-looped adjacency matrix, i.e.,
\begin{equation}\label{eq_wofeat}
    \mathbf{X}_i^{(j)} = \mathrm{SVD} \left( \mathbf{A}_i^{(j)}+\mathbf{I}_{n_i^{(j)}} \right) \in \mathbb{R}^{n_i\times \bar{d}}
\end{equation}
where $\mathrm{SVD}(\cdot)$ returns the singular vectors corresponding to the top-$\bar{d}$ singular values, similar to \eqref{eq_Z_KSVD}. Then we apply \eqref{eq_K_construct}, \eqref{eq_K_construct_Q}, \eqref{eq_Z_KSVD}, and \eqref{eq_ZZZ} to $\mathbf{X}_i^{(j)}$ to generate unified node embeddings. For large datasets, the Nyström approximation \citep{williams2000using} can be employed to accelerate the computation of the kernel matrix and SVD.

\subsection{Convergent Density-Maximization Alignment for Spectral Features}
\vspace{-5pt}
In SVD, individual singular vectors have arbitrary signs \citep{bro2008resolving}. This sign ambiguity may make the embeddings of two similar graphs very different, leading to significant difficulties in both the training and testing stages. Moreover, if two singular values are the same, the order of the corresponding singular vectors cannot be determined, which further increase the difficulty in learning. In machine learning, to ensure learnability and generalization, we require that the training samples and the testing samples are from the same distribution, or at least, their means are similar. Therefore, we proposed to align the mean embeddings of different graphs via maximizing the density.

Specifically, consider the SVD embeddings of $M\times Q$ graphs generated by the method in Section \ref{sec_graph}, for each $\lambda_q$, we compute the mean vectors of the embedding matrices $\mathbf{Z}_{\lambda_q}^{(j)}$ as $\boldsymbol{\mu}_j=\frac{1}{\bar{N}_j}\mathbf{Z}_j^\top\boldsymbol{1}_{\bar{N}_j}$, where $j\in [M]$ and we have dropped the subscript $\lambda_q$ to simplify the notation for the following operations. 
For each graph $j$, we introduce an orthonormal matrix $\mathbf{R}_j\in\mathbb{R}^{\bar{d}\times \bar{d}}$, which will transform $\boldsymbol{\mu}_j$ as $\mathbf{R}_j\boldsymbol{\mu}_j$, $j\in[M]$, which means $\mathbf{R}_j\mathbf{Z}_{\lambda_q}^{(j)}$ is equivalent to $\mathbf{Z}_{\lambda_q}^{(j)}$ in preserving the information of $\mathbf{K}_{\lambda_q}^{(i)}$. We align all mean vectors using the corresponding orthonormal matrices by maximizing the density of the mean vectors. 
The density of each mean vector can be calculated by the kernel density estimation \citep{parzen1962estimation}:
\begin{equation}
  \hat{p}(\boldsymbol{\mu})=\frac{1}{M}\sum_{j=1}^M\frac{1}{(2\pi h)^{\bar{d}/2}}\exp\left(-\frac{\|\boldsymbol{\mu}-\boldsymbol{\mu}_j\|^2}{2h}\right)
\end{equation}
where we use the Gaussian kernel with hyperparameter $h$. 
Let $\mathcal{R}$ be the set of all orthonormal matrices of size $\bar{d}\times \bar{d}$, i.e., $\mathcal{R}=\{\mathbf{R}\in\mathbb{R}^{\bar{d}\times \bar{d}}: \mathbf{R}^\top\mathbf{R}=\mathbf{I}_{\bar{d}} \}$.
Then we maximize the total density of the $M$ mean vectors: 
{\small
\begin{equation}\label{eq_opt_mean_align_0}
\mathop{\text{maximize}}_{\mathbf{R}_j\in\mathcal{R}, j\in[M]}~\frac{1}{M}\sum_{i=1}^M\sum_{j=1}^M\frac{1}{(2\pi h)^{\bar{d}/2}}\exp\left(-\frac{\|\mathbf{R}_i\boldsymbol{\mu}_i-\mathbf{R}_j\boldsymbol{\mu}_j\|^2}{2h}\right)
\end{equation}
}
Letting $\gamma=\frac{1}{2h}$, \eqref{eq_opt_mean_align_0} is equivalent to the following problem
\begin{equation}\label{eq_opt_mean_align_1}
\begin{aligned}
  \mathop{\text{maximize}}_{\mathbf{R}_j\in\mathcal{R}, j\in[M]}~&\frac{1}{M}\sum_{i=1}^M\sum_{j=1}^M\exp\left(-\gamma\|\mathbf{R}_i\boldsymbol{\mu}_i-\mathbf{R}_j\boldsymbol{\mu}_j\|^2\right)\triangleq \mathcal{L}\left(\{\mathbf{R}_j\}_{j=1}^M\right)
  \end{aligned}
\end{equation}

\renewcommand{\algorithmicrequire}{\textbf{Input:}}
\renewcommand{\algorithmicensure}{\textbf{Output:}}

\begin{figure}[t]
\vspace{-20pt}
\begin{minipage}[t]{0.48\textwidth}

\begin{algorithm}[H]
\caption{Max-Density Mean Alignment}
\label{alg_DMMA}
\small{
\begin{algorithmic}[1]
\REQUIRE
$\boldsymbol{\mu}_1,\boldsymbol{\mu}_2,\ldots,\boldsymbol{\mu}_M$; $\gamma>0$; $\eta>0$; $T$.
\STATE Initialization: $\mathbf{R}_j^{(0)}=\mathbf{I}_{\bar{d}}$, $\forall j\in[M]$
\STATE $w_{ij}=\exp(-\gamma(\|\boldsymbol{\mu}_i\|^2+\|\boldsymbol{\mu}_j\|^2))$, $(i,j)\in[M]\times[M]$
\FOR{$t=1$ to $T$}
\STATE $k_{ij}^{(t-1)}=\exp(2\gamma\langle \mathbf{R}_i^{(t-1)}\boldsymbol{\mu}_i,\mathbf{R}_j^{(t-1)}\boldsymbol{\mu}_j \rangle)$, $(i,j)\in[M]\times[M]$
\FOR{$i=1$ to $M$}
\STATE $\mathbf{H}_i^{(t)}=\sum_jw_{ij}k_{ij}^{(t-1)}\mathbf{R}_j^{(t-1)}\boldsymbol{\mu}_j\boldsymbol{\mu}_i^\top+\eta\mathbf{R}_i^{(t-1)}$
\STATE SVD: $\mathbf{H}_i^{(t)}=\mathbf{U}_i\mathbf{S}_i\mathbf{V}^\top_i$
\STATE $\mathbf{R}_{i}^{(t)}=\mathbf{U}_i\mathbf{V}^\top_i$
\ENDFOR
\ENDFOR
\ENSURE $\mathbf{R}_1,\mathbf{R}_2,\ldots,\mathbf{R}_M$
\end{algorithmic}
}
\end{algorithm}
\end{minipage}
\hfill
\begin{minipage}[t]{0.48\textwidth}

\begin{algorithm}[H]
\small{
\begin{algorithmic}[1]
    \renewcommand{\algorithmicrequire}{\textbf{Input:}}
    \caption{pretraining model}\label{alg_pretrain}
    \REQUIRE {Train datasets: $\mathcal{D} = \{\mathcal{G}_j\}_{j=1}^M$}
    \STATE Compute $\{\mathbf{Z}^{(j)}\}_{j=1}^M$ using \eqref{eq_Z_KSVD} and \eqref{eq_ZZZ}
    \STATE $\{\mathbf{R}_j^{\lambda_q}\}_{j=1}^M \gets$ Algorithm~\ref{alg_DMMA}
    \STATE $\mathbf{Z}_{\lambda_q}^{(j)}\leftarrow \mathbf{Z}_{\lambda_q}^{(j)}{\mathbf{R}_j^{(\lambda_q)}}^\top,~~q\in[Q],~j\in[M]$
    \REPEAT
    \FOR{$j = 1$ to $M$}
    \FOR{$i=1$ to $S$}
    \STATE{Sample $\{\mathbf{G}_i^{(j)} = (\mathbf{Z}_i^{(j)}, \mathbf{A}_i^{(j)})\}_{i\in\mathcal{B}}$}
    \STATE{$\mathbf{g}_{i}^{(j)}\gets F_{\mathcal{W}, \mathcal{V}, \gamma}\big(\mathbf{A}_i^{(j)}, \mathbf{Z}_i^{(j)}\big)$, $i\in\mathcal{B}$}
    \STATE$\mathcal{W}\gets \mathcal{W} -\alpha_1\nabla_{\mathcal{W}}\mathcal{L}\left(\mathcal{W},\mathcal{V}, \gamma\right)$\\
    $\mathcal{V}\gets \mathcal{V}-\alpha_1\nabla_{\mathcal{V}}\mathcal{L}\left(\mathcal{W},\mathcal{V}, \gamma\right)$\\
    $\gamma\gets\gamma-\alpha_2\nabla_{\gamma}\mathcal{L}\left(\mathcal{W},\mathcal{V}, \gamma\right)$
    \ENDFOR
    \ENDFOR
    \UNTIL{Convergence conditions are met}
    \ENSURE Pretrained model $F_{\mathcal{W},\mathcal{V}, \gamma}$
\end{algorithmic}}
\end{algorithm}
\end{minipage}
\vspace{-10pt}
\end{figure}

\vspace{-5pt}
The optimization is non-trivial due to the orthonormal constraints and the exponential functions. We propose an efficient algorithm in Algorithm \ref{alg_DMMA}. The detailed derivation for the algorithm is introduced in Appendix \ref{app_alg_derivation}. Theorem \ref{theorem_convergence} provides a convergence guarantee for the optimization.
\begin{theorem}[Monotonic convergence of Algorithm~\ref{alg_DMMA}]
\label{theorem_convergence}
Let $\nu:=\max_i\|\boldsymbol\mu_j\|_2$ and suppose $\eta>4\gamma\nu^4M^{3/2}
+\nu^2M$. 
Then $\{L(\{\mathbf R_j^{(t)}\}_{j=1}^M)\}_{t\geq0}$ is monotonically non-decreasing and convergent. Moreover, $\sum_{j=1}^M
\|\mathbf R_j^{(t)}-\mathbf R_j^{(t-1)}\|_F^2
\rightarrow 0$ as $t\rightarrow\infty$.
Every accumulation point of $\{\mathbf R_j^{(t)}\}_{j=1}^M$ is a stationary point of the constrained maximization problem.
\end{theorem}

Once $\mathbf{R}_1,\ldots,\mathbf{R}_M$ are optimized, we modify the embeddings of the $M\times Q$ global graphs as
\begin{equation}\label{align_embedding}
    \mathbf{Z}_{\lambda_q}^{(j)}\leftarrow \mathbf{Z}_{\lambda_q}^{(j)}{\mathbf{R}_j^{(\lambda_q)}}^\top,\quad q\in[Q],~j\in[M],
\end{equation}
where $\mathbf{R}_i^{(\lambda_q)}$ denotes the $\mathbf{R}_j$ we obtained for the kernel embeddings given by the $q$-th kernel function. Recalling \eqref{eq_ZZZ} and using \eqref{align_embedding}, we here obtain the modified embeddings $\mathbf{Z}_1^{(j)},\ldots,\mathbf{Z}_{N_j}^{(j)}$, $j\in[M]$.
It is worth noting that Algorithm \ref{alg_DMMA} can also be applied to the generated attributes by \eqref{eq_wofeat} of graphs without inherent node attributes. 
\vspace{-5pt}
\subsection{Local--Global Node Encoding with GIN and Graph Transformer}
\vspace{-5pt}
To design a universal graph representation model $F$, 
we incorporate two main components: a GIN module $f$ \citep{xu2019powerful} and a graph transformer (GT) \citep{rampavsek2022recipe} module $g$.
We build a GIN encoder followed by a graph transformer encoder $g_{\psi}\circ f_{\theta}\left(\cdot\right)$, where $\theta$ and $\psi$ are the parameters. 
The GIN encoder specializes in learning
local representations of the structure of a node's immediate neighborhood, 
while the transformer computes all pairwise node interactions, 
enabling global reasoning through attention mechanisms. The node representations of $G_i^{(j)}$ obtained from the model can be formulated as
\begin{equation}
\begin{aligned} 
    \mathbf{H}_i^{(j)}=g_{\psi}\circ &f_{\theta}\left(\mathbf{A}_i^{(j)},\mathbf{Z}_i^{(j)}\right) \Big\|f_\theta\left(\mathbf{A}_i^{(j)},\mathbf{Z}_i^{(j)}\right)
\end{aligned}
\end{equation}
where $i\in[N_j]$ and $j\in[M]$.
For convenience, we let $\mathcal{W}=\{\psi,\theta\}$, which is the set of all parameters of the GIN and GT.

\vspace{-5pt}
\subsection{Hierarchical Reference Distribution Encoding}
\vspace{-5pt}
Standard graph-level representations are usually obtained by pooling node embeddings into a single vector, which mainly captures aggregate statistics and may discard rich distributional information in the node representation space. 
Reference distribution learning characterizes a graph by comparing its node-embedding distribution with learnable prototype distributions~\citep{NEURIPS2024_GRDL}. 
We extend this idea from a single readout module to a \emph{layer-wise distributional encoding mechanism}: reference distributions are attached to multiple GIN and GT layers, so that each graph is represented by its similarities to learnable distributions at different structural ranges and abstraction levels. 
This hierarchical design helps preserve local, mid-level, and global node-distribution patterns for cross-domain graph vectorization.

Specifically, suppose for a certain layer $l \in [L]$ in the backbone, we have $R$ reference discrete distributions $\{\mathbf{V}^{(l)}_1, \mathbf{V}^{(l)}_2\ldots, \mathbf{V}^{(l)}_R\}\triangleq \mathcal{V}^{(l)}$, each $\mathbf{V}^{(l)}_b\in\mathbb{R}^{m\times d}$ is a learnable set of $m$ reference nodes and can be interpreted as a discrete reference distribution in the latent space.
To obtain the graph representations from node embedding matrix $\mathbf{H}_i^{(j,l)}$, 
we measure the similarity between the graph $G_i^{(j)}$ and the reference distributions $\{\mathbf{V}^{(l)}_b\}_{b=1}^R$. 
Letting $\xi$ be a similarity measure between two distributions, the similarity between the graph $G_i^{(j)}$ and the reference distribution $\mathbf{V}^{(l)}_b$ is
\begin{equation}
    s_{i}^{(j,l)} = \xi\left(\mathbf{H}_i^{(j,l)}, \mathbf{V}^{(l)}_b\right),\quad b\in[R]
\end{equation}
We let $\xi$ be the negative kernelized Maximum Mean Discrepancy (MMD) \citep{gretton2012kernel} to be the similarity measure $\xi$ and have 
\begin{equation}
    \begin{aligned}
        &s_{i,b}^{(j,l)} =-\text{MMD}\left(\mathbf{H}_i^{(j,l)}, \mathbf{V}^{(l)}_b\right) \\
        &= -\Big\Vert\frac{1}{n_i} \sum_{p=1}^{n_i} \phi\left(\mathbf{h}_p^{(j,l)}\right)-\frac{1}{m} \sum_{q=1}^m \phi\left(\mathbf{v}_q^{(b,l)}\right)\Big\Vert_2 \\
        &= -\Big[\frac{1}{n_i^2} \sum_{p, p^{\prime} \in [n_i]} k\left(\mathbf{h}_p^{(j,l)}, \mathbf{h}_{p^{\prime}}^{(j,l)}\right)
        +\frac{1}{m^2} \sum_{q, q^{\prime} \in [m]} k\left(\mathbf{v}_{q}^{(b,l)}, \mathbf{v}^{(b,l)}_{q^{\prime}}\right)
-\frac{2}{m n_i} \sum_{p \in [n_i], q \in [m]} k\left(\mathbf{h}_p^{(j,l)}, \mathbf{v}_{q}^{(b,l)}\right)\Big]^{\frac{1}{2}}
    \end{aligned}
\end{equation}
 where $\phi$ is the high-dimensional feature map induced by a kernel function, $\mathbf{h}_p^{(j)}$ is the $p$-th row of $\mathbf{H}_i^{(j)}$, $\mathbf{v}_{q}^{(b)}$ is the $q$-th row of $\mathbf{V}_b$, 
 and $k(\mathbf{x}, \mathbf{x}^\prime) = \mathrm{exp}\left(-\gamma\|\mathbf{x} - \mathbf{x}^\prime\|^2\right)$ is the Gaussian kernel with a learnable parameter $\gamma$. 
We apply the RD module after every layer in both the GT and GIN backbones, allowing the final graph embedding to capture hierarchical information from neighbors at different ranges. The final graph embedding $\mathbf{g}_i^{(j)}\in\mathbb{R}^r$ of graph $i$ in dataset $j$ combines the similarity vector $\mathbf{s}_i^{(j)}=\Big\Vert_{l\in[L]} \mathbf{s}_i^{(j,l)}$  from different layers with a readout vector $\mathbf{p}_i^{(j)}$, i.e.,
\begin{equation}
    \mathbf{g}_i^{(j)} = \mathbf{s}_i^{(j)}\big|\big|\mathbf{p}_i^{(j)},\quad i\in[N_j],j\in[M]
\end{equation}
where $\mathbf{p}_i^{(j)} = \text{READOUT}(\mathbf{H}_i^{(j)})$ is obtained using a graph-level pooling operation.

\vspace{-10pt}
\subsection{Supervised and Unsupervised Cross-Dataset Pretraining}
\vspace{-5pt}
Different datasets $\mathcal{G}_1,\mathcal{G}_2,\ldots,\mathcal{G}_M$ may contain varying numbers of classes. 
To unify the training framework across different datasets and avoid changing classifiers during training, we adopt the supervised contrastive loss (SCL) \citep{oord2018representation}. Therefore, we minimize the following loss
\begin{equation}\label{eq_SCL}
\begin{aligned}
    \mathcal{L}_{\text{SCL}}(\mathcal{W},\mathcal{V}, \gamma)= -\sum_{j=1}^M\frac{1}{N_j}\sum_{i=1}^{N_j} \frac{1}{|C(i)|}\sum_{u\in C(i)} \log\left(p_{iu}^{(j)}\right)
\end{aligned}
\end{equation}
where $p_{iu}^{(j)}={\exp\left(\zeta(\mathbf{g}_i^{(j)},\mathbf{g}_u^{(j)})\right)}/{\sum_{v\neq i}^{N_j} \exp\left(\zeta(\mathbf{g}_i^{(j)},\mathbf{g}_v^{(j)})\right)}$, $\mathbf{g}_{i}^{(j)}=F_{\mathcal{W},\mathcal{V}, \gamma}(\mathbf{A}_i^{(j)},\mathbf{Z}_i^{(j)})$, $\zeta(\mathbf{u,\mathbf{v}})=\frac{\mathbf{u}^\top \mathbf{v}}{\varrho\|\mathbf{u}\|\cdot\|\mathbf{v}\|}$, $C(i)$ denotes the set of samples from the same class as $\mathbf{g}_i^{(j)}$, 
and $\varrho$ is a temperature hyperparameter. 
This objective encourages graphs from the same class to be close in the embedding space 
while pushing apart samples from different classes. See Algorithm \ref{alg_pretrain}.

GraphVec can also be trained in an unsupervised pretraining setting, where positive pairs are constructed through augmentations without using graph labels. Let $P(i)$ be the set of positive samples of $G_i$ obtained by augmentation, the unsupervised contrastive loss (USL) can be represented as 
\begin{equation}\label{eq_UCL}
\begin{aligned}
    &~\mathcal{L}_{\text{UCL}}(\mathcal{W},\mathcal{V}, \gamma) = -\sum_{j=1}^M\frac{1}{N_j}\sum_{i=1}^{N_j} \frac{1}{|P(i)|}\sum_{u\in P(i)} \log\left(p_{iu}^{(j)}\right)
\end{aligned}
\end{equation}

\vspace{-10pt}
\subsection{Downstream Inference with Aligned Graph Embeddings}
\vspace{-5pt}
When applying the pretrained model to graph-level downstream tasks, the output embeddings generated by the model can be directly utilized as input features for other machine learning models. Specifically, in few-shot graph classification, the mean alignment algorithm need to be performed on both train graphs and test graphs, the detailed formulation and algorithm are in Appendix \ref{test_detail}.

\vspace{-7pt}
\subsection{Generalization Analysis of GraphVec}
\vspace{-5pt}
Providing a theoretical guarantee for the generalization ability of GraphVec, i.e., its performance on unseen test datasets, is crucial yet challenging, primarily due to the model's inherent complexity in objective and architecture. 
Since the loss defined in \eqref{eq_SCL} cannot intuitively reflect the model performance, we consider a general metric learning loss $\ell\in[0,1]$. An example is $\ell({G_u,G_v})=(1-\varrho C_{uv}\zeta(\mathbf{g}_u,\mathbf{g}_v))/2$,
where $C_{uv}=1$ if $G_u$ and $G_v$ are in the same class and $C_{uv}=-1$ otherwise.
\begin{theorem}\label{theorem_geb}
Denote $\vartheta_1$ the number of layers of each of the $Q$ GINs, $\vartheta_2$ the number of layers of the GT, $\kappa_1$ the MLP depth in each GIN, and $\kappa_2$ the MLP depth in the GT. Let $\mathbf{W}$ be the weight matrix in a layer of the networks. Let $\tilde{\mathbf{Z}}$ be the whole input data matrix of the GINs and denote $\beta=\|\tilde{\mathbf{Z}}\|_F$. Let $\varsigma=\max _{(i,j)\in[N]\times[M]}\|\mathbf{A}_{i}^{(j)}\|_2$. Suppose $\ell$ is $\tau$-Lipschitz continuous and the attention maps in GT are $\mu$-Lipschitz continuous. Denote $\mathcal{L}(F)=\mathbb{E}_{G,G\sim\mathbb{G}}[\ell(G,G')]$. Then with probability $1-\delta$ over the training dataset $\mathcal{D}$, the following inequality holds
\begin{equation*}
\begin{aligned}
     {\mathcal{L}}(F) \leq \frac{1}{MN(N-1)}\sum_{j=1}^M\sum_{u\neq v}\ell(G_{u}^{(j)},G_{v}^{(j)})
     +\frac{16+48\tau L_F\beta Q\bar{d}\sqrt{\ln{(2Q\bar{d})}}\ln(MN/2)}{MN}+\sqrt{\frac{\ln(1/\delta)}{2MN}}
\end{aligned}
\end{equation*}
where    
    $L_F=\left(4\sqrt{\frac{\gamma R}{{n}}}+\frac{1}{\sqrt{n}}\right)
\left(\varsigma^{\vartheta_1}\max_{q\in[Q]}\prod_{j=1}^{\kappa_1}\|\mathbf{W}_j^{\text{GIN}_q}\|_2\right)\left(\mu^{\vartheta_2}\prod_{j}^{\kappa_2}\|\mathbf{W}_j^{\text{GT}}\|_2\right)$.
\end{theorem}
The theorem has the following implications.
\vspace{-5pt}
\begin{itemize}[noitemsep,leftmargin=10pt]
    \item When the total number of training graphs $MN$ is larger, the bound is tighter, which is further verified by the experiments in Figure \ref{fig:number_pretrain}. Note that if we use the unsupervised contrastive loss to train the model, due to the data augmentation (though the samples are not independent), the generalization could be stronger.
    \item Although $\beta$ often scales with $\sqrt{n}$, we have a factor $\tfrac{1}{\sqrt{n}}$ in $L_F$. This means that the number of nodes in each graph does not have a significant impact on the generalization, provided that the spectral norms of $\mathbf{A}_i^{(j)}$ increase slowly with $n$. As a result, our model will generalize well to both small graphs (e.g., ENZYMES) and large graphs (e.g., REDDIT), as shown by Tables \ref{tab:fewshot_res} and \ref{tab:fewshot_results2}.
    \item Since $L_F$ scales with $\mathcal{O}(\sqrt{\gamma R})$, we could use a relatively large $R$ to enrich the final vector representation for each graph, thereby improving the expressiveness. Moreover, $L_F$ is not very sensitive to $\gamma$, which is learned adaptively.
\end{itemize}

\vspace{-10pt}
\section{Experiments}
\vspace{-10pt}
\begin{table*}[ht]
    \centering
        \caption{50-shot graph classification performance comparison with different pretrained models. We color the \textcolor{red}{\textbf{best}} and \textcolor{orange}{\textbf{second best}} models. The compared numbers of in-domain experiments are from EdgePrompt \citep{fu2025edge}. We only demonstrate the most competitive results reported in EdgePrompt here due to space limitations. Full comparison can be found in Appendix \ref{app_fullres}.}
    \label{tab:fewshot_res}
    \resizebox{\textwidth}{!}{
\begin{tabular}{c|c|c|c|c|c|c|c}
\toprule pretraining & Tuning Methods & ENZYMES & DD & NCI1 & NCI109 & Mutagenicity  &Average\\ 

\midrule \multirow{8}{*}{in-dataset} 
 & Classifier Only \citep{you2020graph} & $27.07_{ \pm 1.04}$ & $61.77_{ \pm 2.40}$ & $61.27_{ \pm 3.64}$ & $62.12_{ \pm 1.10}$ & $67.36_{ \pm 0.71}$  &$55.92$
\\
 & GraphPrompt \citep{liu2023graphprompt}& $26.87_{ \pm 1.47}$ & $62.58_{ \pm 1.84}$ & $62.45_{ \pm 1.52}$& $62.41_{ \pm 0.69}$ & $68.03_{ \pm 0.78}$  &$56.47$
\\
 & ALL-in-one \citep{sun2023all}& $25.73_{ \pm 1.18}$ & $65.16_{ \pm 1.47}$ & $58.52_{ \pm 1.59}$ & $62.01_{ \pm 0.66}$ & $64.43_{ \pm 1.00}$  &$55.17$
\\
 & GPF \citep{fang2023universal}& $28.53_{ \pm 1.76}$ & $65.64_{ \pm 0.70}$ & $61.45_{ \pm 3.13}$ & $61.90_{ \pm 1.26}$ & $67.19_{ \pm 0.74}$  &$56.94$
\\
 & GPF-plus \citep{fang2023universal}& $27.33_{ \pm 2.01}$ & $67.20_{ \pm 1.56}$& $61.61_{ \pm 2.89}$ & $62.84_{ \pm 0.23}$ & $67.69_{ \pm 0.64}$  &$57.33$
\\
 & EdgePrompt \citep{fu2025edge}& $29.33_{ \pm 2.30}$ & $63.97_{ \pm 2.14}$ & $62.02_{ \pm 3.02}$ & $62.02_{ \pm 1.03}$ & $67.55_{ \pm 0.85}$  &$56.98$
\\
 & EdgePrompt+ \citep{fu2025edge}& $3 2 . 6 7_{ \pm 2 . 5 3}$& $67 . 72_{ \pm 1.62}$& $\textcolor{orange}{\mathbf{67.07}}_{ \pm 1 . 9 6}$& $\textcolor{orange}{\mathbf{66.53}}_{ \pm 1.30}$& ${{68.31}}_{ \pm 1 .36}$ &$60.46$
\\ 
\cline{2-8}

& GraphCL \citep{you2020graph} & $30.50_{ \pm 1.16}$ & $62.89_{ \pm 2.19}$ & $62.49_{ \pm 1.95}$ & $61.68_{ \pm 0.93}$ & $66.62_{ \pm 1.87}$  &$56.84$ \\
 & GeMax \citep{sun2024learning}& $31.20 _{ \pm 5.92}$& $56.79_{ \pm 1.12}$& $59.62_{ \pm 1.41}$& $59.55_{\pm 0.85}$& $65.18_{ \pm 1.98}$ &$54.46$
\\

\midrule
 \multirow{4}{*}{cross-dataset}&{GCN \citep{kipf2016semi}} & $43.33_{ \pm 1.05}$ & $65.84_{ \pm 2.77}$ & $61.36_{ \pm 2.00}$& $62.17_{ \pm 0.66}$& $60.46_{ \pm 1.75}$  &$58.63$\\
&{BRIDGE \citep{yuan2025much}} & $36.67_{ \pm 5.96}$ & $64.95_{ \pm 3.38}$ & $63.50_{ \pm 2.27}$& $61.78_{ \pm 1.63}$& $65.12_{ \pm 2.83}$  &$58.40$\\
&{GFT \citep{wang2024gft}} & $34.61_{ \pm 3.12}$ & $56.00_{ \pm 1.77}$ & $59.16_{ \pm 6.25}$& $60.50_{ \pm 2.71}$& $67.82_{ \pm 3.18}$  &$55.61$\\
&{RiemannGFM \citep{sun2025riemanngfm}} & $34.27_{ \pm 1.72}$ & $68.74_{ \pm 1.31}$ & $55.10_{ \pm 2.24}$& $59.86_{ \pm 1.30}$& $62.56_{ \pm 4.04}$  &$56.11$\\
\midrule
\multicolumn{2}{c|}{\textbf{GraphVec}} & $\textcolor{red}{\mathbf{51.00}}_{ \pm 3.22}$ & $\textcolor{red}{\mathbf{75.94}}_{ \pm 2.70}$ & $\textcolor{red}{\mathbf{67.32}}_{ \pm 1.51}$& $\textcolor{red}{\mathbf{67.90}}_{ \pm 1.67}$& $\textcolor{red}{\mathbf{68.57}}_{ \pm 1.62}$  &$\textcolor{red}{\mathbf{66.14}}$\\
\multicolumn{2}{c|}{\textbf{Unsupervised GraphVec}} & $\textcolor{orange}{\mathbf{48.33}}_{ \pm 2.36}$ & $\textcolor{orange}{\mathbf{74.02}}_{ \pm 1.26}$ & $66.11_{ \pm 2.30}$& $64.34_{ \pm 2.38}$& $\textcolor{orange}{\mathbf{68.38}}_{ \pm 2.88}$  &$\textcolor{orange}{\mathbf{64.23}}$\\
\bottomrule
\end{tabular}}
\vspace{-10pt}
\end{table*}

\subsection{Few-Shot Graph Classification}
\vspace{-5pt}
\paragraph{Datasets and Baselines}
Following \citep{fu2025edge}, we use five datasets from TUDataset \citep{Morris+2020}, including ENZYMES, DD, NCI1, NCI109, and Mutagencity, to conduct few-shot graph classification experiments. 
We evaluate our methods against baselines under two distinct settings:

\textbf{1) In-dataset setting:} In this setting, the training and testing sets are partitioned within the same dataset, and the model is fine-tuned using few-shot samples. We employ  SimGRACE \citep{xia2022simgrace}, the most competitive pretraining strategy reported in EdgePrompt for pretraining and adopt seven different tuning mechanisms, including prompt-tuning methods such as GraphPrompt, All-In-One \citep{sun2023all}, GPF \citep{fang2023universal}, and GPF-plus \citep{fang2023universal} as well as standard classifier training. We also compare GraphVec with 2 graph representation methods, including the classical method GraphCL \citep{you2020graph} and the recent method GeMax \citep{sun2024learning}.

\textbf{2) Cross-dataset setting:} This setting involves fine-tuning and testing on datasets that were unseen during the pretraining phase. Specifically, we adopt a leave-one-out strategy where each of the five datasets serves as the downstream target for testing, while the remaining four datasets are leveraged for pretraining. In cross domain setting we compared GraphVec with 3 recent strong GFM baselines, GFT \citep{wang2024gft}, BRIDGE \citep{yuan2025much}, and RiemannGFM \citep{sun2025riemanngfm}, and 1 classic GNN \citep{kipf2016semi}. GraphVec is pretrained and evaluated under this setting.

To further evaluate the transferability of GraphVec in different domains, we conduct more experiments on 4 social network datasets, including COLLAB, REDDIT-BINARY, IMDB-BINARY, IMDB-MULTI and 3 computer vision datasets including Letter-med, COIL-RAG and Cuneiform using GraphVec pretrained on 5 bio-chemical datasets mentioned above. Our GraphVec is compared with 2 graph prompt based methods: ProNoG \citep{yu2025non} and EdgePrompt+ \citep{fu2025edge}, 2 LLM-based GFMs: OFA \citep{liu2023one} and GOFA \citep{konggofa}, and 4 non-LLM-based GFMs: GFT \citep{wang2024gft}, SAMGPT \citep{yu2025samgpt}, BRIDGE \citep{yuan2025much}, and RiemannGFM \citep{sun2025riemanngfm}. For a fairer graph-level comparison, we additionally equip GFT and RiemannGFM with DiffPool \citep{ying2018hierarchical} as an advanced pooling mechanism.
More details about the settings and datasets can be found in Appendix \ref{exp_set} and \ref{datasets}. 

\paragraph{Results}

\begin{table*}[ht]
\vspace{-10pt}
\centering
\caption{Few-shot graph classification on social network and computer vision datasets. 5-shot and 1-shot settings are used on COIL-RAG and Cuneiform, due to insufficient samples in some classes.}
\label{tab:fewshot_results2}
\vspace{-5pt}
\resizebox{0.99\textwidth}{!}{
\begin{tabular}{l *{7}{c}}
\toprule
Dataset & COLLAB & REDDIT-B & IMDB-B & IMDB-M & Letter-med & COIL-RAG & Cuneiform \\
& 50-shot & 50-shot & 50-shot & 50-shot & 50-shot & 5-shot & 1-shot \\
\midrule
ProNoG \citep{yu2025non}& $46.88 _{\pm 3.14}$ & $74.33 _{\pm 2.05}$ & $60.8 _{\pm 5.19}$ & $40.53 _{\pm 0.66}$ & $56.98 _{\pm 5.83}$ & $34.97 _{\pm 7.62}$ & $10.00 _{\pm 7.92}$ \\
EdgePrompt+\citep{fu2025edge} & $\mathbf{68.76 }_{\pm 1.60}$ & $74.60 _{\pm 1.60}$ & ${\mathbf{71.17}} _{\pm 1.07}$ & $46.60 _{\pm 0.50}$ & $74.66 _{\pm 1.69}$ & $5.60 _{\pm 0.21}$ & $18.22 _{\pm 0.72}$ \\ \midrule
\text{OFA \citep{liu2023one}}

& $33.07 _{\pm 0.43}$

& $50.02 _{\pm 0.99}$

& $51.06 _{\pm 0.67}$

& $34.33 _{\pm 0.50}$

& $20.19 _{\pm 0.37}$

& $20.48 _{\pm 0.79}$

& $19.94 _{\pm 0.67}$ \\
GOFA \citep{konggofa}

& $38.31 _{\pm 1.02}$

& $\mathrm{OOM}$

& $49.49 _{\pm 0.84}$

& $34.67 _{\pm 0.61}$

& $11.11 _{\pm 0.25}$

& $3.84 _{\pm 0.18}$

& $6.25 _{\pm 0.31}$ \\ \midrule

\text{SAMGPT \citep{yu2025samgpt}}

& $67.20 _{\pm 0.91}$

& $55.77 _{\pm 0.90}$

& $49.50 _{\pm 0.50}$

& $46.00 _{\pm 6.30}$

& $56.90 _{\pm 0.74}$

& $57.34 _{\pm 1.66}$

& $52.00 _{\pm 6.96}$ \\

GFT+DiffPool \citep{wang2024gft}

& $63.44 _{\pm 1.35}$

& $\mathrm{OOM}$

& $65.83 _{\pm 1.12}$

& $46.30 _{\pm 0.94}$

& $37.83 _{\pm 0.71}$

& $7.68 _{\pm 0.42}$

& $10.34 _{\pm 0.39}$ \\
RiemannGFM+DiffPool \citep{sun2025riemanngfm}

& $40.45 _{\pm 0.88}$

& $\mathrm{OOM}$

& $55.53 _{\pm 0.72}$

& $34.31 _{\pm 0.57}$

& $45.76 _{\pm 0.63}$

& $12.45 _{\pm 0.55}$

& $11.05 _{\pm 0.48}$ \\
BRIDGE \citep{yuan2025much} & $54.52 _{\pm 3.73}$ & $57.80 _{\pm 8.79}$ & $50.20 _{\pm 6.27}$ & $36.60 _{\pm 5.48}$ & $35.40 _{\pm 2.09}$ & $27.02 _{\pm 6.90}$ & $23.07 _{\pm 9.42}$ \\
\midrule
GraphVec & $\underline{68.09} _{\pm 2.99}$ & $\mathbf{81.52} _{\pm 1.50}$ & $\underline{68.39} _{\pm 4.06}$ & $\mathbf{46.70} _{\pm 0.99}$ & $\mathbf{85.60} _{\pm 1.44}$ & $\mathbf{74.20} _{\pm 0.77}$ & $\mathbf{55.86} _{\pm 8.15}$ \\ 
Unsupervised GraphVec& $64.82 _{\pm 1.87}$ & $\underline{80.62} _{\pm 2.18}$ & $66.61 _{\pm 1.66}$ & $45.00 _{\pm 1.81}$ & $\underline{80.63} _{\pm 1.87}$ & $\underline{72.22} _{\pm 1.66}$ & $\underline{53.79} _{\pm 5.97}$ \\
\bottomrule
\end{tabular}}
\vspace{-5pt}
\end{table*}
The results are reported in Table \ref{tab:fewshot_res} and Table \ref{tab:fewshot_results2}.  As shown in Table \ref{tab:fewshot_res}, our GraphVec consistently outperforms all baseline methods in the supervised pretraining paradigm and achieved second best in the unsupervised pretraining paradigm. Compared with in-domain baselines with different fine-tuning strategies under the cross-dataset setting, our model still achieves competitive performance. In social network and computer vision datasets that have a significant gap between pretraining datasets in both semantics and structure, our GraphVec consistently outperforms other baselines. This performance demonstrates that our pretrained model effectively learns generalizable graph embeddings across different domains without relying on delicately designed tuning methods. It also highlights the model's capability to capture features from diverse domains while maintaining a strong generalization ability to new domains. Compared to the supervised pretrained model, the unsupervised pretrained model exhibits only a slight decrease in accuracy except for COLLAB and Letter-med. The unsupervised GraphVec still outperforms all four baselines on 4 of the 7 datasets.

\vspace{-5pt}
\subsection{Graph Clustering}
\vspace{-5pt}
\begin{table*}[ht]
\centering
\vspace{-5pt}
\caption{Graph clustering performance on ENZYMES, NCI1, COLLAB, REDDIT-BINARY, REDDIT-MULTI. The comparison numbers are from AMGC \citep{yang2025towards}. }
\label{tab:cluster_res}
\resizebox{\textwidth}{!}{
\begin{tabular}{l|ccH|ccH|ccH|ccH|ccH}
\toprule
Method & \multicolumn{3}{c|}{ENZYMES} & \multicolumn{3}{c|}{NCI1} & \multicolumn{3}{c|}{COLLAB} & \multicolumn{3}{c|}{REDDIT-BINARY} & \multicolumn{3}{c}{REDDIT-MULTI} \\
 & ACC & NMI & ARI & ACC & NMI & ARI & ACC & NMI & ARI & ACC & NMI & ARI & ACC & NMI & ARI \\
\midrule
InfoGraph +KM & $22.1 _{\pm 1.0}$ & $2.4 _{\pm 0.5}$ & $1.3 _{\pm 0.5}$ & $54.1 _{\pm 2.2}$ & $1.3 _{\pm 1.1}$ & $0.9 _{\pm 0.9}$ & $59.6 _{\pm 1.8}$ & $14.4 _{\pm 3.0}$ & $6.6 _{\pm 2.3}$ & $51.3 _{\pm 2.1}$ & $2.3 _{\pm 0.4}$ & $0.6 _{\pm 0.2}$ & $20.3 _{\pm 0.9}$ & $0.5 _{\pm 0.2}$ & $0.0 _{\pm 0.0}$ \\
InfoGraph +SC & $23.8 _{\pm 0.5}$ & $4.6 _{\pm 0.7}$ & $2.2 _{\pm 0.4}$ & $54.9 _{\pm 1.7}$ & $0.9 _{\pm 0.6}$ & $1.0 _{\pm 0.8}$ & $60.9 _{\pm 2.5}$ & $15.4 _{\pm 3.3}$ & $9.3 _{\pm 3.5}$ & $50.8 _{\pm 1.3}$ & $1.6 _{\pm 0.6}$ & $0.6 _{\pm 0.0}$ & $24.7 _{\pm 1.3}$ & $4.8 _{\pm 0.6}$ & $3.2 _{\pm 0.6}$ \\
GraphCL +KM & $21.5 _{\pm 0.2}$ & $1.6 _{\pm 0.1}$ & $0.9 _{\pm 0.1}$ & $55.4 _{\pm 1.7}$ & $0.5 _{\pm 0.3}$ & $1.0 _{\pm 0.9}$ & $58.0 _{\pm 1.2}$ & $17.8 _{\pm 2.0}$ & $11.3 _{\pm 0.6}$ & $51.9 _{\pm 3.3}$ & $3.4 _{\pm 1.2}$ & $0.2 _{\pm 0.0}$ & $25.3 _{\pm 0.9}$ & $5.3 _{\pm 0.3}$ & $4.3 _{\pm 0.6}$ \\
GraphCL +SC & $25.3 _{\pm 0.3}$ & $4.8 _{\pm 0.4}$ & $2.0 _{\pm 0.3}$ & $50.8 _{\pm 1.6}$ & $0.6 _{\pm 0.6}$ & $1.1 _{\pm 0.8}$ & $57.8 _{\pm 0.6}$ & $17.0 _{\pm 1.3}$ & $10.1 _{\pm 0.7}$ & $55.9 _{\pm 2.1}$ & $3.2 _{\pm 1.0}$ & $0.3 _{\pm 0.2}$ & $27.3 _{\pm 1.3}$ & $5.4 _{\pm 0.8}$ & $4.2 _{\pm 1.1}$ \\
JOAO + KM & $21.7 _{\pm 0.4}$ & $4.9 _{\pm 0.4}$ & $2.1 _{\pm 0.2}$ & $51.1 _{\pm 0.4}$ & $0.4 _{\pm 0.2}$ & $0.1 _{\pm 0.0}$ & $58.3 _{\pm 1.5}$ & $18.7 _{\pm 2.6}$ & $11.1 _{\pm 1.8}$ & $54.3 _{\pm 2.9}$ & $4.2 _{\pm 1.8}$ & $0.8 _{\pm 0.3}$ & $26.6 _{\pm 0.6}$ & $3.6 _{\pm 1.2}$ & $2.5 _{\pm 0.2}$ \\
JOAO + SC & $24.4 _{\pm 1.4}$ & $3.2 _{\pm 0.7}$ & $1.7 _{\pm 0.8}$ & $51.5 _{\pm 3.0}$ & $0.9 _{\pm 1.2}$ & $0.4 _{\pm 1.2}$ & $58.2 _{\pm 0.9}$ & $17.1 _{\pm 2.1}$ & $10.6 _{\pm 0.8}$ & $55.9 _{\pm 1.2}$ & $6.7 _{\pm 2.0}$ & $1.4 _{\pm 0.6}$ & $25.6 _{\pm 0.6}$ & $2.5 _{\pm 0.2}$ & $3.4 _{\pm 0.3}$ \\
GLCC\citep{ju2023glcc}  & $24.4 _{\pm 1.4}$ & $3.2 _{\pm 0.7}$ & $1.7 _{\pm 0.8}$ & $60.9 _{\pm 2.3}$ & $5.3 _{\pm 1.9}$ & $3.6 _{\pm 2.6}$ & $60.3 _{\pm 0.6}$ & $18.2 _{\pm 1.3}$ & $12.1 _{\pm 0.9}$ & $\textcolor{orange}{\mathbf{67.6}} _{\pm 3.4}$ & $9.2 _{\pm 2.6}$ & $8.7 _{\pm 1.7}$ & $32.4 _{\pm 2.1}$ & $11.8 _{\pm 1.3}$ & $8.2 _{\pm 1.6}$ \\
AMGC\citep{yang2025towards} & $\textcolor{orange}{\mathbf{26.7}} _{\pm 2.0}$ & $\textcolor{orange}{\mathbf{5.2}} _{\pm 1.3}$ & $\textcolor{orange}{\mathbf{2.8}} _{\pm 0.7}$ & $\textcolor{orange}{\mathbf{62.7}} _{\pm 3.0}$ & $\textcolor{orange}{\mathbf{6.4}} _{\pm 1.9}$ & $\textcolor{orange}{\mathbf{6.4}} _{\pm 3.6}$ & $\textcolor{orange}{\mathbf{61.2}} _{\pm 1.0}$ & $\textcolor{orange}{\mathbf{20.5}} _{\pm 1.6}$ & $12.9 _{_{\pm 0.9}}$ & $64.3 _{\pm 1.9}$ & $\textcolor{orange}{\mathbf{12.1}} _{\pm 3.3}$ & $\textcolor{orange}{\mathbf{10.5}} _{\pm 2.7}$ & $\textcolor{orange}{\mathbf{35.5}} _{\pm 2.3}$ & $\textcolor{orange}{\mathbf{16.1}} _{\pm 0.9}$ & $\textcolor{orange}{\mathbf{12.0}} _{\pm 0.7}$ \\
\midrule
GraphVec & $\textcolor{red}{\mathbf{29.1}} _{\pm 0.4}$ & $\textcolor{red}{\mathbf{7.7}} _{\pm 0.3}$ & $\textcolor{red}{\mathbf{3.5}} _{\pm 0.2}$ & $\textcolor{red}{\mathbf{64.8}} _{\pm 0.0}$ & $\textcolor{red}{\mathbf{6.5}} _{\pm 0.0}$ & $\textcolor{red}{\mathbf{8.7}} _{\pm 0.0}$ & $\textcolor{red}{\mathbf{61.8}} _{\pm 0.0}$ & $\textcolor{red}{\mathbf{21.2}} _{\pm 0.0}$ & $\textcolor{red}{\mathbf{18.9}}_{\pm 0.0}$ & $\textcolor{red}{\mathbf{71.6}} _{\pm 0.0}$ & $\textcolor{red}{\mathbf{20.7}} _{\pm 0.0}$ & $\textcolor{red}{\mathbf{18.6}} _{\pm 0.0}$ & $\textcolor{red}{\mathbf{40.0}} _{\pm 0.2}$ & $\textcolor{red}{\mathbf{17.3}} _{\pm 0.1}$ & $\textcolor{red}{\mathbf{12.1}} _{\pm 0.2}$ \\
\bottomrule
\end{tabular}}
\end{table*}
\vspace{-5pt}
To further validate the superiority of our proposed model on graph-level tasks and quality of the graph embeddings obtained from GraphVec, we conducted experiments on graph clustering: applying spectral clustering \citep{ng2001spectral} to the graph representations produced by the pretrained model. As shown in Table \ref{tab:cluster_res}, our methods perform best. The results validate the generalization ability of GraphVec and the transferability of the graph embeddings obtained from the pretrained model on unseen domains and unseen structures. 

\textbf{More Results} Due to space limitations, we defer additional results on node classification, low-shot learning, raw-attribute semantic preservation, pretraining scale, clustering, ablations, robustness, and efficiency to Appendices \ref{app_nodecls}, \ref{app_1_shot}, \ref{app_attr_semantics}, \ref{app_impact}, \ref{app_more_clus}, \ref{app_abl}, \ref{app_robust}, and \ref{app_time}.

\vspace{-10pt}
\section{Conclusions}
\vspace{-10pt}
This paper presented GraphVec, a graph vectorization model trained on multiple source datasets for graph-level representation learning. GraphVec maps graphs into fixed-dimensional vectors that can be directly used for downstream tasks. We introduced a multi-graph construction method to generate consistent node embeddings across datasets and a reference distribution module to better exploit node-embedding information. Experiments on few-shot graph classification and graph clustering demonstrated the superiority of GraphVec over competing methods. One limitation is that GraphVec does not support zero-shot learning, as it does not use language models or textual information. Future work will incorporate cross-modal alignment during training.



\bibliography{reference}
\bibliographystyle{named}

\newpage
\appendix
\onecolumn
\section{Notations}
\begin{table}[h]
    \centering
    \begin{tabular}{c|l|c|l}
    \toprule
    Symbol & Meaning & Symbol & Meaning\\ \midrule
       $x$  & a real number & $\mathbf{x}$ & a vector \\
       $\mathbf{X}$ & a matrix &$\mathbf{I}_n$& identity matrix of size $n\times n$ \\ 
       $G$ & a graph & $\mathbf{g}$& vector representation of $G$ \\ 
       $\mathcal{G}$ & a set of graphs & $\mathcal{D}$ & a dataset \\
       $\|\mathbf{x}\|$& the Euclidean norm of $\mathbf{x}$ & $\|\mathbf{x}\|_1$& the $\ell_1$ norm of $\mathbf{x}$ \\ 
       $[M]$  & the set $\{1,2,\ldots,M\}$& $\mathbf{X}\Vert\mathbf{Y}$ or $[\mathbf{X},\mathbf{Y}]$ & vertical concatenation \\
       $\|\mathbf{X}\|_F$ & Frobenius norm of matrix & $\|\mathbf{X}\|_2$ & spectral norm of matrix\\
        \bottomrule
    \end{tabular}
    \caption{Notations}
    \label{tab:notation}
\end{table}

\section{More about Related Work}\label{app_related_work}
\paragraph{Language Model-Free GFMs}
Many studies have explored training GFMs using the ``pre-train and adaptation" paradigm, 
leveraging message-passing-based or transformer-based GNNs as backbones. 
These approaches typically employ contrastive or generative self-supervised learning for pretraining, 
followed by fine-tuning a subset of model parameters to adapt to downstream tasks or datasets \citep{liu2025graph}.
Contrastive methods \citep{qiu2020gcc, sun2019infograph, velivckovic2019deep, xia2022simgrace} typically aim to produce generalized graph representations through maximizing the agreement between different augmentations of the same instance.
For example, GraphCL \citep{you2020graph} designs four types of graph data augmentations to learn invariant representations under specialized perturbations.
GCOPE \citep{zhao2024all} employs a graph contrastive learning framework and introduces coordinators 
which are some virtual nodes that function as dynamic bridges between disparate graph datasets. 
Focused on node-level tasks, it effectively mitigates
negative transfer effects when pretraining graph models on cross-domain datasets. 
In the meantime, generative methods pre-train GNNs through graph reconstruction or property prediction.
For instance, GraphMAEs \citep{hou2022graphmae,hou2023graphmae2} employed the reconstruction of features with masking strategies. 
Recently, graph prompt tuning methods \citep{sun2022gppt, fang2023universal, sun2023all, liu2023graphprompt} have been proposed as an adaptation mechanism to bridge the gap between pretraining tasks and downstream tasks. 
GraphPrompt \citep{liu2023graphprompt} converts the pretraining task and downstream tasks to follow the same template based on subgraph similarity 
and uses learnable prompt vectors to implement different aggregation schemes for readout in different downstream tasks. 
EdgePrompt \citep{fu2025edge} manipulates input graphs by learning prompt vectors for edges and incorporates the
edge prompts through message passing in the pretrained GNN models. As an effective adaptation mechanism, these methods have been widely adopted in subsequent GFMs \citep{yuan2025much,yuan2025graver,wang2025multi,yu2024multigprompt,yu2025samgpt}. Recently, several notable Graph Foundation Models have been introduced. RiemannGFM \citep{sun2025riemanngfm} embeds nodes into a Riemann manifold using structural vocabulary of trees as circles, enable structral transferability across domains. TS-GNN \citep{finkelshtein2025equivariance} investigates symmetries that a graph foundation model must respect, which is mainly designed for node-level tasks. RAG4GFM \citep{wangrag4gfm} applies the Retrieval-Augmented Generation (RAG) paradigm to Graph Foundation Models, allowing them to dynamically access and integrate graph knowledge at inference time. SCORE \citep{wang2024towards} utilizes KGs as a unified topological structure to tackle diverse tasks. 

\paragraph{LLM-based GFMs} These models utilize the strong capacity of large language models to conduct graph analysis. 
For instance, GraphQA \citep{fatemi2023talk} converts graph connectivity into textual descriptions and uses LLMs to answer graph reasoning questions. By enriching these prompts with domain-specific context, GraphQA can effectively learn cross-domain structural representations, essentially serving as a structural GFM. Similar approaches include NLGraph \citep{wang2024can}, which tackles tasks like shortest path finding by translating graphs into text, demonstrating another viable pathway for unified structure learning.
For unifying node feature representations, the One For All (OFA) framework \citep{liu2023one} offers an innovative solution. It aggregates diverse graph datasets into a unified text-attributed graph (TAG) format, then leverages LLMs to jointly learn feature representations that transcend domain boundaries. This approach effectively bridges the gap between heterogeneous graph data sources.

\paragraph{GFM for Graph-Level Tasks}
As mentioned before, most of the existing GFMs are designed for node-level tasks. There are a few studies that focus on graph-level tasks across domains. For instance, \citet{chauhanfew} tries to pretrain GNNs on certain classes of a dataset and conduct few-shot classification on the remaining classes within the same dataset. \citet{hassani2022cross} adopts a meta-learning approach to learn model initialization for few-shot graph classification. These graph-level models are usually small and not general.
Some GFMs can be adapted to graph-level tasks. 
For instance, GraphPrompt \citep{liu2023graphprompt}, GraphPrompt+ \citep{yu2024generalized}, and EdgePrompt \citep{fu2025edge} use learnable prompts to adjust graph-level pooling for obtaining domain-adaptive graph embeddings.
Other works such as SAMGPT \citep{yu2025samgpt}, ProNoG \citep{yu2025non} MultiGPrompt \citep{yu2024multigprompt} , and BRIDGE \citep{yuan2025much} are also designed to effectively perform graph classification,
but they mainly build on well-designed node embedding and use simple global pooling to apply the model to graph-level tasks.

\section{Derivation of Algorithm \ref{alg_DMMA}}\label{app_alg_derivation}

Recall that we aim to solve
\begin{equation}
\label{eq:dmma_original_obj}
  \mathop{\mathrm{maximize}}_{\mathbf{R}_j\in\mathcal{R},\, j\in[M]}
  ~
  \frac{1}{M}\sum_{i=1}^M\sum_{j=1}^M
  \exp\left(
  -\gamma
  \|\mathbf{R}_i\boldsymbol{\mu}_i-\mathbf{R}_j\boldsymbol{\mu}_j\|_2^2
  \right)
  \triangleq
  \mathcal{L}\left(\{\mathbf{R}_j\}_{j=1}^M\right),
\end{equation}
where
\begin{equation}
\mathcal{R}
=
\left\{
\mathbf{R}\in\mathbb{R}^{\bar d\times \bar d}:
\mathbf{R}^\top\mathbf{R}=\mathbf{I}_{\bar d}
\right\}.
\end{equation}
The factor $1/M$ does not affect the maximizer. Therefore, for notational simplicity, we consider the equivalent unnormalized objective
\begin{equation}
\label{eq:dmma_L_def}
    L(\mathbf{R})
    =
    \sum_{i=1}^M\sum_{j=1}^M
    \exp\left(
    -\gamma
    \|\mathbf{R}_i\boldsymbol{\mu}_i-\mathbf{R}_j\boldsymbol{\mu}_j\|_2^2
    \right),
\end{equation}
where
\begin{equation}
\mathbf{R}:=\{\mathbf{R}_j\in\mathcal{R}\}_{j=1}^M .
\end{equation}
Since each $\mathbf{R}_j$ is orthogonal, we have
\begin{equation}
\label{eq:orthogonal_norm_preserve}
\|\mathbf{R}_j\boldsymbol{\mu}_j\|_2
=
\|\boldsymbol{\mu}_j\|_2 .
\end{equation}
Thus,
\begin{equation}
\label{eq:dmma_obj_rewrite}
\begin{aligned}
L(\mathbf{R})
&=
\sum_{i=1}^M\sum_{j=1}^M
\exp\left[
-\gamma
\left(
\|\boldsymbol{\mu}_i\|_2^2+
\|\boldsymbol{\mu}_j\|_2^2
\right)
\right]
\exp\left(
2\gamma
\langle
\mathbf{R}_i\boldsymbol{\mu}_i,
\mathbf{R}_j\boldsymbol{\mu}_j
\rangle
\right)  \\
&=
\sum_{i=1}^M\sum_{j=1}^M
w_{ij}
\exp\left(
2\gamma
\langle
\mathbf{R}_i\boldsymbol{\mu}_i,
\mathbf{R}_j\boldsymbol{\mu}_j
\rangle
\right),
\end{aligned}
\end{equation}
where
\begin{equation}
\label{eq:dmma_wij}
w_{ij}
:=
\exp\left[
-\gamma
\left(
\|\boldsymbol{\mu}_i\|_2^2+
\|\boldsymbol{\mu}_j\|_2^2
\right)
\right].
\end{equation}

Let $\mathbf{R}^{(t-1)}=\{\mathbf{R}_j^{(t-1)}\}_{j=1}^M$ be the current iterate. Define
\begin{equation}
\label{eq:dmma_kij}
k_{ij}^{(t-1)}
:=
\exp\left(
2\gamma
\left\langle
\mathbf{R}_i^{(t-1)}\boldsymbol{\mu}_i,
\mathbf{R}_j^{(t-1)}\boldsymbol{\mu}_j
\right\rangle
\right).
\end{equation}
The Euclidean gradient of $L$ with respect to $\mathbf{R}_i$ at $\mathbf{R}^{(t-1)}$ is
\begin{equation}
\label{eq:dmma_grad}
\nabla_{\mathbf{R}_i}L(\mathbf{R}^{(t-1)})
=
4\gamma
\sum_{j=1}^M
w_{ij}k_{ij}^{(t-1)}
\mathbf{R}_j^{(t-1)}
\boldsymbol{\mu}_j
\boldsymbol{\mu}_i^\top .
\end{equation}
For compactness, denote
\begin{equation}
\label{eq:dmma_Ai_def}
\mathbf{A}_i^{(t)}
:=
\sum_{j=1}^M
w_{ij}k_{ij}^{(t-1)}
\mathbf{R}_j^{(t-1)}
\boldsymbol{\mu}_j
\boldsymbol{\mu}_i^\top .
\end{equation}
Then
\begin{equation}
\label{eq:dmma_grad_Ai}
\nabla_{\mathbf{R}_i}L(\mathbf{R}^{(t-1)})
=
4\gamma \mathbf{A}_i^{(t)} .
\end{equation}

At iteration $t$, we maximize the following surrogate:
\begin{equation}
\label{eq:dmma_surrogate}
\begin{aligned}
\widetilde L_t(\mathbf{R})
=
\sum_{i=1}^M
\left\langle
\nabla_{\mathbf{R}_i}L(\mathbf{R}^{(t-1)}),
\mathbf{R}_i-\mathbf{R}_i^{(t-1)}
\right\rangle
-
2\gamma\eta
\sum_{i=1}^M
\|\mathbf{R}_i-\mathbf{R}_i^{(t-1)}\|_F^2 ,
\end{aligned}
\end{equation}
subject to $\mathbf{R}_i\in\mathcal{R}$ for all $i\in[M]$. Substituting \eqref{eq:dmma_grad_Ai} into \eqref{eq:dmma_surrogate}, we obtain
\begin{equation}
\label{eq:dmma_surrogate_expand}
\begin{aligned}
\widetilde L_t(\mathbf{R})
=
4\gamma
\sum_{i=1}^M
\left\langle
\mathbf{A}_i^{(t)},
\mathbf{R}_i-\mathbf{R}_i^{(t-1)}
\right\rangle
-
2\gamma\eta
\sum_{i=1}^M
\|\mathbf{R}_i-\mathbf{R}_i^{(t-1)}\|_F^2 .
\end{aligned}
\end{equation}
Because both $\mathbf{R}_i$ and $\mathbf{R}_i^{(t-1)}$ are orthogonal, we have
\begin{equation}
\label{eq:orthogonal_distance}
\|\mathbf{R}_i-\mathbf{R}_i^{(t-1)}\|_F^2
=
2\bar d
-
2\left\langle
\mathbf{R}_i^{(t-1)},
\mathbf{R}_i
\right\rangle .
\end{equation}
Therefore, after dropping constants independent of $\mathbf{R}$ and removing the positive factor $4\gamma$, maximizing \eqref{eq:dmma_surrogate_expand} is equivalent to
\begin{equation}
\label{eq:dmma_linear_subproblem}
\mathop{\mathrm{maximize}}_{\mathbf{R}_i\in\mathcal{R},\, i\in[M]}
~
\sum_{i=1}^M
\left\langle
\mathbf{A}_i^{(t)}
+
\eta \mathbf{R}_i^{(t-1)},
\mathbf{R}_i
\right\rangle .
\end{equation}
Let
\begin{equation}
\label{eq:dmma_Hi_def}
\mathbf{H}_i^{(t)}
:=
\mathbf{A}_i^{(t)}
+
\eta\mathbf{R}_i^{(t-1)}
=
\sum_{j=1}^M
w_{ij}k_{ij}^{(t-1)}
\mathbf{R}_j^{(t-1)}
\boldsymbol{\mu}_j
\boldsymbol{\mu}_i^\top
+
\eta\mathbf{R}_i^{(t-1)} .
\end{equation}
Then the subproblem separates over $i$ as
\begin{equation}
\label{eq:dmma_procrustes}
\mathbf{R}_i^{(t)}
=
\mathop{\arg\max}_{\mathbf{R}_i^\top\mathbf{R}_i=\mathbf{I}_{\bar d}}
\left\langle
\mathbf{H}_i^{(t)},
\mathbf{R}_i
\right\rangle .
\end{equation}
This is the classical orthogonal Procrustes problem \citep{schonemann1966generalized}. Let the singular value decomposition of $\mathbf{H}_i^{(t)}$ be
\begin{equation}
\label{eq:dmma_Hi_svd}
\mathbf{H}_i^{(t)}
=
\mathbf{U}_i
\mathbf{S}_i
\mathbf{V}_i^\top .
\end{equation}
Then the optimal solution is
\begin{equation}
\label{eq:dmma_R_update}
\mathbf{R}_i^{(t)}
=
\mathbf{U}_i\mathbf{V}_i^\top .
\end{equation}
This gives the update rule used in Algorithm~\ref{alg_DMMA}. 

\section{Proof for Theorem \ref{theorem_convergence}}

\begin{proof}
Define
\begin{equation}
\phi_{ij}(\mathbf R)
=
w_{ij}\exp\!\left(
2\gamma
\langle \mathbf R_i\boldsymbol\mu_i,
\mathbf R_j\boldsymbol\mu_j\rangle
\right).
\end{equation}
Since every $\mathbf R_i$ is orthogonal, we have
\begin{equation}
0<\phi_{ij}(\mathbf R)
=
\exp\!\left(
-\gamma\|\mathbf R_i\boldsymbol\mu_i-
\mathbf R_j\boldsymbol\mu_j\|_2^2
\right)
\leq 1 .
\end{equation}
The Euclidean gradient of $\mathcal L$ with respect to $\mathbf R_i$ is
\begin{equation}
\nabla_{\mathbf R_i}\mathcal L(\mathbf R)
=
4\gamma
\sum_{j=1}^M
\phi_{ij}(\mathbf R)
\mathbf R_j\boldsymbol\mu_j\boldsymbol\mu_i^\top .
\end{equation}

We first bound the Lipschitz constant of $\nabla\mathcal L$. For two feasible points $\mathbf R$ and $\mathbf S$, let
\begin{equation}
\delta_i=\|\mathbf R_i-\mathbf S_i\|_F,
\qquad
D=\left(\sum_{i=1}^M\delta_i^2\right)^{1/2}.
\end{equation}
By the mean-value theorem and the bound $\phi_{ij}(\mathbf R)\leq1$,
\begin{equation}
|\phi_{ij}(\mathbf R)-\phi_{ij}(\mathbf S)|
\leq
2\gamma\nu^2(\delta_i+\delta_j).
\end{equation}
Therefore,
\begin{equation}
\begin{aligned}
&
\left\|
\nabla_{\mathbf R_i}\mathcal L(\mathbf R)
-
\nabla_{\mathbf R_i}\mathcal L(\mathbf S)
\right\|_F \\
&\leq
4\gamma
\sum_{j=1}^M
\left[
|\phi_{ij}(\mathbf R)-\phi_{ij}(\mathbf S)|
\|\mathbf R_j\boldsymbol\mu_j\boldsymbol\mu_i^\top\|_F
+
\phi_{ij}(\mathbf S)
\|(\mathbf R_j-\mathbf S_j)
\boldsymbol\mu_j\boldsymbol\mu_i^\top\|_F
\right] \\
&\leq
4\gamma
\sum_{j=1}^M
\left[
2\gamma\nu^4(\delta_i+\delta_j)
+
\nu^2\delta_j
\right] \\
&\leq
\left(
16\gamma^2\nu^4M
+
4\gamma\nu^2\sqrt{M}
\right)D .
\end{aligned}
\end{equation}
Taking the Frobenius norm over all blocks gives
\begin{equation}
\begin{aligned}
\|\nabla\mathcal L(\mathbf R)-\nabla\mathcal L(\mathbf S)\|_F
&=
\left(
\sum_{i=1}^M
\left\|
\nabla_{\mathbf R_i}\mathcal L(\mathbf R)
-
\nabla_{\mathbf R_i}\mathcal L(\mathbf S)
\right\|_F^2
\right)^{1/2}  \\
&\leq
\left(
16\gamma^2\nu^4M^{3/2}
+
4\gamma\nu^2M
\right)
\|\mathbf R-\mathbf S\|_F .
\end{aligned}
\end{equation}
Thus $\nabla\mathcal L$ is $L_\nabla$-Lipschitz continuous with
\begin{equation}
L_\nabla
=
16\gamma^2\nu^4M^{3/2}
+
4\gamma\nu^2M .
\end{equation}

By the smoothness of $\mathcal L$, we have
\begin{equation}
\begin{aligned}
\mathcal L(\mathbf R^{(t)})
\geq\;&
\mathcal L(\mathbf R^{(t-1)})
+
\sum_{i=1}^M
\left\langle
\nabla_{\mathbf R_i}\mathcal L(\mathbf R^{(t-1)}),
\mathbf R_i^{(t)}-\mathbf R_i^{(t-1)}
\right\rangle \\
&-
\frac{L_\nabla}{2}
\sum_{i=1}^M
\|\mathbf R_i^{(t)}-\mathbf R_i^{(t-1)}\|_F^2 .
\end{aligned}
\end{equation}
Let
\begin{equation}
\mathbf A_i^{(t-1)}
=
\sum_{j=1}^M
w_{ij}k_{ij}^{(t-1)}
\mathbf R_j^{(t-1)}
\boldsymbol\mu_j
\boldsymbol\mu_i^\top .
\end{equation}
Then
\begin{equation}
\nabla_{\mathbf R_i}\mathcal L(\mathbf R^{(t-1)})
=
4\gamma \mathbf A_i^{(t-1)} .
\end{equation}
Algorithm~\ref{alg_DMMA} computes
\begin{equation}
\mathbf R_i^{(t)}
=
\arg\max_{\mathbf R_i^\top\mathbf R_i=\mathbf I}
\left\langle
\mathbf A_i^{(t-1)}+\eta\mathbf R_i^{(t-1)},
\mathbf R_i
\right\rangle .
\end{equation}
Hence,
\begin{equation}
\left\langle
\mathbf A_i^{(t-1)}+\eta\mathbf R_i^{(t-1)},
\mathbf R_i^{(t)}-\mathbf R_i^{(t-1)}
\right\rangle
\geq 0 .
\end{equation}
Since both $\mathbf R_i^{(t)}$ and $\mathbf R_i^{(t-1)}$ are orthogonal,
\begin{equation}
\left\langle
\mathbf R_i^{(t-1)},
\mathbf R_i^{(t)}-\mathbf R_i^{(t-1)}
\right\rangle
=
-\frac12
\|\mathbf R_i^{(t)}-\mathbf R_i^{(t-1)}\|_F^2 .
\end{equation}
Therefore,
\begin{equation}
\left\langle
\mathbf A_i^{(t-1)},
\mathbf R_i^{(t)}-\mathbf R_i^{(t-1)}
\right\rangle
\geq
\frac{\eta}{2}
\|\mathbf R_i^{(t)}-\mathbf R_i^{(t-1)}\|_F^2 .
\end{equation}
Substituting this into the smoothness inequality gives
\begin{equation}
\mathcal L(\mathbf R^{(t)})
\geq
\mathcal L(\mathbf R^{(t-1)})
+
\left(
2\gamma\eta-\frac{L_\nabla}{2}
\right)
\sum_{i=1}^M
\|\mathbf R_i^{(t)}-\mathbf R_i^{(t-1)}\|_F^2 .
\end{equation}
Since $\eta>L_\nabla/(4\gamma)$, we have
\begin{equation}
2\gamma\eta-\frac{L_\nabla}{2}>0 .
\end{equation}
Thus $\mathcal L(\mathbf R^{(t)})$ is non-decreasing.

Moreover, because $0<\phi_{ij}(\mathbf R)\leq1$,
\begin{equation}
0<\mathcal L(\mathbf R)\leq M^2 .
\end{equation}
Hence $\{\mathcal L(\mathbf R^{(t)})\}_{t\geq0}$ converges. Summing the ascent inequality over $t$ gives
\begin{equation}
\sum_{t=1}^{\infty}
\sum_{i=1}^M
\|\mathbf R_i^{(t)}-\mathbf R_i^{(t-1)}\|_F^2
<\infty .
\end{equation}
Therefore,
\begin{equation}
\sum_{i=1}^M
\|\mathbf R_i^{(t)}-\mathbf R_i^{(t-1)}\|_F^2
\rightarrow 0 .
\end{equation}

Finally, the feasible set is compact because it is a product of orthogonal groups. Hence the sequence has accumulation points. Let $\mathbf R^\star$ be any accumulation point. Since
\begin{equation}
\mathbf R_i^{(t)}-\mathbf R_i^{(t-1)}\rightarrow \mathbf 0 ,
\end{equation}
the optimality condition of the Procrustes update passes to the limit and gives
\begin{equation}
\mathbf R_i^\star
=
\arg\max_{\mathbf R_i^\top\mathbf R_i=\mathbf I}
\left\langle
\mathbf A_i^\star+\eta\mathbf R_i^\star,
\mathbf R_i
\right\rangle .
\end{equation}
This implies the first-order stationarity condition for the original constrained maximization problem. Therefore, every accumulation point is stationary.
\end{proof}

\section{Proof for Theorem \ref{theorem_geb}}\label{appendix_proof_geb}

\subsection{Main Proof}

Since \eqref{eq_SCL} does not explicitly show the error related to classification or metric learning, here we consider the following pair-wise loss function $\ell$ instead. An example is as
\begin{equation}
    \ell({G_u,G_v})=1-C_{uv}\cdot\zeta(\mathbf{g}_u,\mathbf{g}_v)
\end{equation}
where $\zeta(\mathbf{g}_u,\mathbf{g}_v)=\frac{\mathbf{g}_u^\top\mathbf{g}_v}{\|\mathbf{g}_u\|\|\mathbf{g}_v\|}$ and $C_{uv}=1$ if $G_u$ and $G_u$ are in the same class and $C_{uv}=-1$ if they are in different classes. Note that $\mathbf{g}=F(G)$, where $F\in\mathcal{F}$.
The empirical risk is
\begin{equation}
    \hat{\mathcal{L}}_{\mathcal{D}}(F)=\frac{1}{M}\sum_{j=1}^M\frac{1}{N(N-1)}\sum_{u\neq v}\ell(G_{u}^{(j)},G_{v}^{(j)})\triangleq \frac{1}{M}\sum_{j=1}^M\bar{\mathcal{L}}_{G_j}(F)
\end{equation}
where we have assumed $N_1=N_2=\cdots=N_M=N$ for convenience and $\bar{\mathcal{L}}_{G_j}(F)=\frac{1}{N(N-1)}\sum_{u\neq v}\ell(G_{u}^{(j)},G_{v}^{(j)})$.
The true risk is
\begin{equation}
    \mathcal{L}(F)=\mathbb{E}_{G,G\sim\mathbb{G}}[\ell(G,G')]
\end{equation}
We would like to bound
\begin{equation}
\begin{aligned}
\sup_{F\in\mathcal{F}}\{\mathcal{L}(F)-\hat{\mathcal{L}}_{\mathcal{D}}(F)\}
\end{aligned}
\end{equation}

For any $\mathcal{D}=\{\mathcal{G}_1,\ldots,\mathcal{G}_j,\ldots,\mathcal{G}_M\}$ and $\tilde{\mathcal{D}}=\{\mathcal{G}_1,\ldots,\tilde{\mathcal{G}}_j,\ldots,\mathcal{G}_M\}$, where
${\mathcal{G}}_j=\{G_1^{(j)},\ldots,{G}_i^{(j)},\ldots,G_N^{(j)}\}$ and $\tilde{\mathcal{G}}_j=\{G_1^{(j)},\ldots,\tilde{G}_i^{(j)},\ldots,G_N^{(j)}\}$,
we have
\begin{equation}
    \begin{aligned}
        &\left|\sup_{F\in\mathcal{F}}\{\mathcal{L}(F)-\hat{\mathcal{L}}_{\mathcal{D}}(F)\}-\sup_{F\in\mathcal{F}}\{\mathcal{L}(F)-\hat{\mathcal{L}}_{\tilde{\mathcal{D}}}(F)\}\right|\\
        \leq&\sup_{F\in\mathcal{F}}\left|\hat{\mathcal{L}}_{\mathcal{D}}(F) -\hat{\mathcal{L}}_{\tilde{\mathcal{D}}}(F)\right| \\ 
        = &\sup_{F\in\mathcal{F}}\left|\frac{1}{MN(N-1)} \left(\sum_{\mathcal{G}_j:u \neq v}{\ell}(G_u^{(j)},G_v^{(j)})-\sum_{\tilde{\mathcal{G}}_j:u \neq v}{\ell}(G_u^{(j)},G_v^{(j)})\right)\right|\\
        \leq &\sup_{F\in\mathcal{F}}\left|\frac{1}{MN(N-1)} \left(\sum_{v\neq i}\left({\ell}({G}_i^{(j)},G_v^{(j)})-{\ell}(\tilde{G}_i^{(j)},G_v^{(j)})\right)\right)\right|\\
        \leq &\sup_{F\in\mathcal{F}}\left|\frac{1}{MN(N-1)} \left(\sum_{v\neq i}\left|{\ell}({G}_i^{(j)},G_v^{(j)})-{\ell}(\tilde{G}_i^{(j)},G_v^{(j)})\right|\right)\right|\\
        \leq &\frac{1}{MN}
    \end{aligned}
\end{equation}
where the last inequality holds due to the fact that $0\leq \ell\leq 1$.
Applying the McDiarmid’s inequality (Lemma \ref{lemma_mcdiarmid}) to $\sup_{F\in\mathcal{F}}\{\hat{\mathcal{L}}_{\mathcal{D}}(F) - \mathcal{L}(F)\}$, with probability at least $1-\delta$, we have
\begin{equation}\label{eq_result_of_mcdiarmid}
    \sup_{F\in\mathcal{F}}\{\mathcal{L}(F)-\hat{\mathcal{L}}_{\mathcal{D}}(F)\}\leq \mathbb{E}_{\mathcal{D}}\left(\sup_{F\in\mathcal{F}}\{\mathcal{L}(F)-\hat{\mathcal{L}}_{\mathcal{D}}(F)\}\right)+\sqrt{\frac{\ln(1/\delta)}{2MN}}
\end{equation}

For convenience, we let $\bar{\ell}(G_u,G_v)=\mathcal{L}(F)-\ell(G_u,G_v)$, we have the following derivation
\begin{equation}
    \begin{aligned}
& \mathbb{E}_{\mathcal{D}}\left(\sup _{F \in \mathcal{F}} \frac{1}{M}\sum_{j=1}^M\frac{1}{N(N-1)} \sum_{u \neq v}\bar{\ell}(G_u^{(j)},G_v{(j)})\right) \\
=&\mathbb{E}_{\mathcal{D}} \left(\sup _{F \in \mathcal{F}} \frac{1}{M}\sum_{j=1}^M\frac{1}{N!} \sum_\pi \frac{1}{\lfloor N / 2\rfloor} \sum_{i=1}^{\lfloor N / 2\rfloor} \bar{\ell}\left(G^{(j)}_{\pi(i)}, G^{(j)}_{\pi(\lfloor N / 2\rfloor+i)}\right)\right) \\
\leq &\mathbb{E}_{\mathcal{D}} \left(\frac{1}{N!} \sum_\pi \sup _{F \in \mathcal{F}} \frac{1}{M}\sum_{j=1}^M\frac{1}{\lfloor N / 2\rfloor} \sum_{i=1}^{\lfloor N / 2\rfloor} \bar{\ell}\left(G^{(j)}_{\pi(i)}, G^{(j)}_{\pi(\lfloor N / 2\rfloor+i)}\right)\right) \\
\leq &\frac{1}{N!} \sum_\pi \mathbb{E}_{\mathcal{D}}\left(\sup _{F \in \mathcal{F}} \frac{1}{M}\sum_{j=1}^M\frac{1}{\lfloor N/ 2\rfloor} \sum_{i=1}^{\lfloor N / 2\rfloor} \bar{\ell}\left(G^{(j)}_{\pi(i)}, G^{(j)}_{\pi(\lfloor N / 2\rfloor+i)}\right)\right) \\
=&\mathbb{E}_{\mathcal{D}}\left(\sup _{F \in \mathcal{F}} \frac{1}{M}\sum_{j=1}^M\frac{1}{\lfloor N/ 2\rfloor} \sum_{i=1}^{\lfloor N / 2\rfloor} \bar{\ell}\left(G^{(j)}_{\pi(i)}, G^{(j)}_{\pi(\lfloor N / 2\rfloor+i)}\right)\right)\\
=&\mathbb{E}_{\mathcal{D}}\left(\sup _{F \in \mathcal{F}}\left\{\mathcal{L}(F)- \tilde{\mathcal{L}}_{\mathcal{D}}(F)\right\}\right)
\end{aligned}
\end{equation}
where $\tilde{\mathcal{L}}_{\mathcal{D}}(F)=\frac{1}{M}\sum_{j=1}^M\frac{1}{\lfloor N/ 2\rfloor} \sum_{i=1}^{\lfloor N / 2\rfloor} {\ell}\left(G^{(j)}_{\pi(i)}, G^{(j)}_{\pi(\lfloor N / 2\rfloor+i)}\right)$.

For convenience, we let $S=M\lfloor N/ 2\rfloor$ and rename the graph-pair as $(G_{s}, \bar{G}_{s})$. So we have $S$ independent samples.
By introducing a virtual dataset $\mathcal{D}'\subset\mathbb{G}$ with size $S$, we obtain
\begin{equation}
    \begin{aligned}
        &\mathbb{E}_{\mathcal{D}}\left(\sup _{F \in \mathcal{F}}\left\{\mathcal{L}(F)-\tilde{\mathcal{L}}_{\mathcal{D}}(F)\right\}\right)\\
        =&\mathbb{E}_{\mathcal{D}}\left(\sup _{F \in \mathcal{F}}\left\{\mathcal{L}(F)-\frac{1}{S}\sum_{s=1}^S\ell\left(G_{s}, \bar{G}_{s}\right)\right\}\right)\\
        =&\mathbb{E}_{\mathcal{D}}\left(\sup _{F \in \mathcal{F}}\left\{ \mathbb{E}_{\mathcal{D}'}\left(\frac{1}{S}\sum_{s=1}^S\ell\left(G_{s}', \bar{G}_{s}'\right)\right)-\frac{1}{S}\sum_{s=1}^S\ell\left(G_{s}, \bar{G}_{s}\right)\right\}\right)\\
        \leq&\mathbb{E}_{\mathcal{D},\mathcal{D}'}\left(\sup _{F \in \mathcal{F}}\frac{1}{S}\sum_{s=1}^S\left
[\ell\left(G_{s}', \bar{G}_{s}'\right)-\ell\left(G_{s}, \bar{G}_{s}\right)\right]\right)\\
    \end{aligned}
\end{equation}
where the inequality holds due to Jensen's inequality. By introducing the Rademacher variable $\epsilon_s\in\{-1,1\}$, we have
\begin{equation}
    \begin{aligned} 
    &\mathbb{E}_{\mathcal{D}}\left(\sup _{F \in \mathcal{F}}\left\{\mathcal{L}(F)- \tilde{\mathcal{L}}_{\mathcal{D}}(F)\right\}\right)\\
    \leq&\mathbb{E}_{\mathcal{D},\mathcal{D}'}\mathbb{E}_{\epsilon}\left(\sup _{F \in \mathcal{F}}\frac{1}{S}\sum_{s=1}^S\epsilon_s\left
[\ell\left(G_{s}', \bar{G}_{s}'\right)-\ell\left(G_{s}, \bar{G}_{s}\right)\right]\right)\\
\leq&\mathbb{E}_{\mathcal{D},\mathcal{D}'}\mathbb{E}_{\epsilon}\left(\sup _{F \in \mathcal{F}}\left\{\frac{1}{S}\sum_{s=1}^S\epsilon_s\ell\left(G_{s}, \bar{G}_{s}\right)\right\}+\sup _{F \in \mathcal{F}}\left\{\frac{1}{S}\sum_{s=1}^S(-\epsilon_s)\left(G_{s}', \bar{G}_{s}'\right)\right\}\right)\\
\leq&2\mathbb{E}_{\mathcal{D},\epsilon}\left(\sup _{F \in \mathcal{F}}\frac{1}{S}\sum_{s=1}^S\epsilon_s\ell\left(G_{s}, \bar{G}_{s}\right)\right)\\
=& 2\mathbb{E}_S(\hat{\mathcal{R}}_S(\mathcal{F}))
    \end{aligned}
\end{equation}
where $\mathcal{R}_S(\mathcal{F}):=\mathbb{E}_S(\hat{\mathcal{R}}_S(\mathcal{F}))$ is the Rademacher complexity.

Combining \eqref{eq_result_of_mcdiarmid}, we arrive at
\begin{equation}\label{eq_RSF_bound}
    \sup_{F\in\mathcal{F}}\{\mathcal{L}(F)-\hat{\mathcal{L}}_{\mathcal{D}}(F)\}\leq 2\mathcal{R}_S(\mathcal{F})+\sqrt{\frac{\ln(1/\delta)}{2MN}}
\end{equation}

According to Lemma \ref{lemma_lip_reference}, Lemma \ref{lemma_lip_gin}, and Lemma \ref{lemma_lip_gt}, the Lipschitz constants of the GIN, GT, and reference layer are 
\begin{equation}
    \begin{aligned}
        L_{\text{GIN}}=&\max _{(i,p)\in[N]\times[M]}\|\mathbf{A}_{i}^{(p)}\|_2^\vartheta\prod_{j=1}^{\vartheta'}\|\mathbf{W}_j\|_2\\
        L_{\text{Ref}}=&4\sqrt{\frac{\gamma R}{{n}}}\\
        L_{\text{GT}}=&\mu^{\vartheta}\prod_{j=1}^{\vartheta \vartheta'}\|\mathbf{W}_j\|_2
    \end{aligned}
\end{equation}
Since there are $Q$ parallel GINs, according to Lemma \ref{lemma_concate_lip}, the Lipschitz constant of their combinations is 
\begin{equation}
    L_{\text{QGIN}}=\max_q{L_\text{GIN}^{(q)}}
\end{equation}
where $L_{\text{GIN}}^{(q)}=\max _{(i,p)\in[N]\times[M]}\|\mathbf{A}_{i}^{(p)}\|_2^\vartheta\prod_{j=1}^{\vartheta'}\|\mathbf{W}_j^{(q)}\|_2$.
Based on the composition of these network components and their specific configurations, the Lipschitz constant of $F$ is 
\begin{equation}
    L_F=\left(4\sqrt{\frac{\gamma R}{{n}}}+\frac{1}{\sqrt{n}}\right)\left(\max _{(i,p)\in[N]\times[M]}\|\mathbf{A}_{i}^{(p)}\|_2^{\vartheta_1}\max_{q\in[Q]}\prod_{j=1}^{\kappa_1}\|\mathbf{W}_j^{\text{GIN}_q}\|_2\right)\left(\mu^{\vartheta_2}\prod_{j}^{\kappa_2}\|\mathbf{W}_j^{\text{GT}}\|_2\right)
\end{equation}
where $\kappa_1$ is the maximum number of MLP layers in each GIN and $\kappa_2$ is the total number of weight matrices excluding those in the attention maps of the transformer. Suppose the loss function $\ell$ is $\tau$-Lipschitz, then the Lipschitz constant of $\ell\circ\mathcal{F}$ is $L_{\ell\circ F}=\tau L_F$.

Let $\tilde{\mathbf{Z}}^{(j)}=[\bar{\mathbf{A}}^{(j)}{\mathbf{Z}}_1^{(j)},\ldots,\bar{\mathbf{A}}^{(j)}{\mathbf{Z}}_Q^{(j)}]$, where $\bar{\mathbf{A}}^{(j)}=\text{diag}({\mathbf{A}}^{(j)}_1,\ldots,{\mathbf{A}}^{(j)}_N)\in\mathbb{R}^{Nn\times Nn}$. We further form $\hat{\mathbf{Z}}=[\tilde{\mathbf{Z}}^{(1)};\tilde{\mathbf{Z}}^{(2)};\ldots;\tilde{\mathbf{Z}}^{(M)}]\in\mathbb{R}^{MNn\times Q\bar{d}}$.
According to Lemma \ref{lemma_cover_matrix}, the covering number of $\mathcal{Z}=\{\hat{\mathbf{Z}}\in\mathbb{R}^{MNn\times Q\bar{d}}:\|\hat{\mathbf{Z}}\|_F\leq \beta\}$ is bounded as
   \begin{equation}
        \ln\mathcal{N}\left(\mathcal{Z}, \epsilon, \|\cdot\|_F\right) \leq \frac{\beta^2Q^2\bar{d}^2\ln{(2Q\bar{d})}}{\epsilon^2}
    \end{equation}
Therefore, using Lemma \ref{lemma_cov_lip}, the covering number of $\ell\circ \mathcal{F}\times \mathcal{Z}$ is bounded as
\begin{equation}\label{eq_NLF}
            \ln\mathcal{N}\left(\ell\circ \mathcal{F}, \epsilon, \|\cdot\|_F\right) \leq \frac{\tau^2L_F^2\beta^2Q^2\bar{d}^2\ln{(2Q\bar{d})}}{\epsilon^2}\triangleq  \frac{\varphi}{\epsilon^2}
\end{equation}

Using Lemma \ref{lemma_Dudley}, we can bound the Rademacher complexity of our model class as
\begin{equation}\label{eq_RSLF}
    \begin{aligned}
            \mathcal{R}_{S}(\ell\circ\mathcal{F}) \leq& \inf_{\alpha > 0} \left(\frac{4\alpha}{\sqrt{S}}+\frac{12}{S}\int_{\alpha}^{\sqrt{S}} \frac{\sqrt{\varphi}}{\epsilon}\,d\epsilon \right) \\  
        \leq& \inf_{\alpha > 0} \left(\frac{4\alpha}{\sqrt{S}}+\frac{12\sqrt{\varphi}}{S}\ln\left(\frac{\sqrt{S}}{\alpha}\right) \right)\\ 
        \leq & \frac{4+12\sqrt{\varphi}\ln(S)}{S}
    \end{aligned}
\end{equation}
where in the last inequality we have let $\alpha=1/\sqrt{S}$.

Now combining \eqref{eq_RSLF}, \eqref{eq_NLF}, and \eqref{eq_RSF_bound}, we arrive at
\begin{equation}
     \mathcal{L}(F) \leq \hat{\mathcal{L}}_{\mathcal{D}}(F)+\frac{16+48\tau L_F\beta Q\bar{d}\sqrt{\ln{(2Q\bar{d})}}\ln(MN/2)}{MN}+\sqrt{\frac{\ln(1/\delta)}{2MN}}
\end{equation}
This completes the proof.

\subsection{Supporting Lemmas and Their Proofs}
\begin{lemma}[McDiarmid's inequality \citep{mcdiarmid1989method}]\label{lemma_mcdiarmid}
Suppose $f: \prod_{k=1}^m \Omega_k \rightarrow \mathbb{R}$ with bounded differences $\left\{c_k\right\}_{k=1}^m$ then, for all $\epsilon>0$, there holds
$$
\operatorname{Pr}_{\mathbf{z}}\left\{f(\mathbf{z})-\mathbb{E}_{\mathbf{z}} f(\mathbf{z}) \geq \epsilon\right\} \leq e^{-\frac{2 \epsilon^2}{\sum_{k=1}^m c_k^2}}
$$
\end{lemma}

\begin{lemma}[Dudley entropy integral bound \citep{bartlett2017spectrally}]\label{lemma_Dudley}
    Let $\mathcal{F}$ be a real-valued function class taking values in $[0, 1]$, and assume that $\mathbf{0} \in \mathcal{F}$. Then
    \begin{equation*}
        \mathcal{R}_{S}(\mathcal{F}) \leq \inf_{\alpha > 0} \left(\frac{4\alpha}{\sqrt{S}}+\frac{12}{S}\int_{\alpha}^{\sqrt{S}} \sqrt{\ln \mathcal{N}\left(\epsilon, \mathcal{F}, \rho\right)} \,d\epsilon \right).
    \end{equation*}
\end{lemma}

\begin{lemma}\label{lemma_lip_reference}
    The Lipschitz constant of the reference layer is 
        $L_{\text{ref}}=4\sqrt{\frac{\theta R}{{n}}}$.
\end{lemma}

\begin{proof}
According to the definition of MMD, we have
\begin{align*}
    \left|\mathrm{MMD}^2\left(\mathbf{H}, \mathbf{V}\right) - \mathrm{MMD}^2\left(\mathbf{H}', \mathbf{V}\right)\right| 
    \leq &
    \left|\frac{1}{n^2} \sum_{i, j=1}^n \left[\exp{\left(-\theta\|\mathbf{h}_i - \mathbf{h}_j\|_2^2\right)} - \exp{\left(-\theta\|\mathbf{h'}_i - \mathbf{h'}_j\|_2^2\right)}\right]\right| \\
    & +
    \left| \frac{2}{mn} \sum_{i=1}^n \sum_{j=1}^m \left[\exp{\left(-\theta\|\mathbf{h}_i - \mathbf{v}_j\|_2^2\right)} - \exp{\left(-\theta\|\mathbf{h'}_i - \mathbf{v}_j\|_2^2\right)}\right]\right| \\
    \overset{(a)}{\leq} &
    \frac{\sqrt{\theta}}{n^2} \sum_{i, j=1}^n \left|\|\mathbf{h}_i - \mathbf{h}_j\|_2 - \|\mathbf{h'}_i - \mathbf{h'}_j\|_2\right| 
    +
    \frac{2\sqrt{\theta}}{mn} \sum_{i=1}^n \sum_{j=1}^m \left|\|\mathbf{h}_i - \mathbf{v}_j\|_2 - \|\mathbf{h'}_i - \mathbf{v}_j\|_2\right| \\
    \overset{(b)}{\leq} &
    \frac{\sqrt{\theta}}{n^2} \sum_{i, j=1}^n \|\left(\mathbf{h}_i - \mathbf{h'}_i\right) - \left(\mathbf{h}_j-\mathbf{h'}_j\right)\|_2 +
    \frac{2\sqrt{\theta}}{mn} \sum_{i=1}^n \sum_{j=1}^m \|\left(\mathbf{h}_i - \mathbf{h'}_i\right) - \left(\mathbf{v}_j-\mathbf{v}_j\right)\|_2 \\
    \leq &
    \frac{4\sqrt{\theta}}{n} \sum_{i=1}^n \|\mathbf{h}_i - \mathbf{h'}_i\|_2\\ 
    \overset{(c)}{\leq} &4\sqrt{\frac{\theta}{{n}}}\|\mathbf{H} - \mathbf{H}'\|_F
\end{align*}
In the above derivation, (a) holds due to $\left|\exp{(-x^2)} - \exp{(-y^2)}\right| \leq |x-y|$ for any $x, y \geq 0$, (b) holds due to the triangle inequality, and (c) holds by the Cauchy–Schwarz inequality.

The output of the layer is $\mathbf{S}$, for which we have
 \begin{align*}
        \|\mathbf{S} - \mathbf{S}'\|_2 
        &= \sqrt{\sum_{i=1}^N\sum_{j=1}^R|s_{ij} - s_{ij}'|^2} \\
        &\leq 
        4\sqrt{\frac{\theta}{{n}}}\sqrt{\sum_{i=1}^N\sum_{j=1}^R\|\mathbf{H}_i - \mathbf{H}_i'\|_F^2} \\
        &=
        4\sqrt{\frac{\theta R}{{n}}}\|\mathbf{H} - \mathbf{H}'\|_F
    \end{align*}
This finished the proof.

\end{proof}

\begin{lemma}\label{lemma_lip_gin}
    Suppose the GIN $f$ has $Q$ layers and each layer has an MLP of $Q'$ layers. Then the Lipschitz constant of $f$ is 
        $L_{\text{GIN}}=\max _{i\in[N]}\|\mathbf{A}_{i}\|_2^Q\prod_{j=1}^{QQ'}\|\mathbf{W}_j\|_2$.
\end{lemma}

\begin{proof}
   Recall that the $l$-th layer of the GIN can be formulated as
\begin{equation}
f^{(l)}\left(\mathbf{A}, \mathbf{Z}^{(l-1)}\right) = \mathrm{MLP}^{(l)}\left(\left({\mathbf{A}} + \epsilon\mathbf{I}\right)\cdot\mathbf{Z}^{(l-1)}\right)
\end{equation}
where $\mathbf{Z}^{(0)}=\mathbf{X}$. For convenience, let $\epsilon=0$. We put all adjacency matrices together to form a big block diagonal matrix $\bar{\mathbf{A}}$ of size $Nn\times Nn$. Then the spectral norm of $\bar{\mathbf{A}}$ is $\max _{i\in[N]}\|\mathbf{A}_{i}\|_2$. Similarly, we form a big matrix $\bar{\mathbf{Z}}$ of size $Nn\times d$. Then we have
\begin{equation}
\bar{Z}^{(l)}=f^{(l)}\left(\bar{\mathbf{A}}, \bar{\mathbf{Z}}^{(l-1)}\right) = \mathrm{MLP}^{(l)}\left(\bar{\mathbf{A}}\bar{\mathbf{Z}}^{(l-1)}\right)
\end{equation}
Then the Lipschitz constant of $f^{(l)}$ is $\max _{i\in[N]}\|\mathbf{A}_{i}\|_2\prod_{j=1}^Q\rho_j\|\mathbf{W}_j\|_2$, where $W_j$ is the weight matrix and $\rho_j$ is the Lipschitz constant of the layer. 
Since most activation functions such as ReLu and Sigmoid are $1$-Lipschitz, we let $\rho_i=1$ $\forall i$. Given that $f$ has $Q$ layers, we conclude that the Lipschitz constant is $\max _{i\in[N]}\|\mathbf{A}_{i}\|_2^Q\prod_{j=1}^{QQ'}\|\mathbf{W}_j\|_2$.
\end{proof}

\begin{lemma}\label{lemma_lip_gt}
    Suppose the graph transformer $g$ is composed of $Q$ blocks and each block has an MLP of $Q'$ layers. Suppose the attention map is $\mu$-Lipschitz. Then the Lipschitz constant of $g$ is 
    $L_{\text{GT}}=\mu^{QQ'}\prod_{j=1}^{QQ'}\|\mathbf{W}_j\|_2$.
\end{lemma}

\begin{proof}
Recall that the self-attention is
    \begin{equation}
    \mathrm{attn}\left(\mathbf{\Gamma}_i^{(j)}\right) = \mathrm{softmax}\left(\frac{(\mathbf{\Gamma}_i^{(j)}\mathbf{W}_Q)(\mathbf{\Gamma}_i^{(j)}\mathbf{W}_K)^\top}{\sqrt{d'}}\right)(\mathbf{\Gamma}_i^{(j)}\mathbf{W}_V)
\end{equation}
Assume that the softmax operation is $\mu$-Lipschitz with respect to the input $\mathbf{\Gamma}_i^{(j)}$. The Lipschitz constant of the self-attention mechanism is $\mu\|\mathbf{W}_V\|_2$.  The self-attention is then followed by a residual connection, layer normalization, and MLP of $Q$-layers. We omit the residual connection and the layer normalization since they have a tiny impact on the analysis. For the MLP, the Lipschitz constant is $\prod_{j=1}^{Q'}\|\mathbf{W}_j\|_2$, where $\mathbf{W}_j$ is the weight matrix of layer $j$ and the activation functions are assumed to be $1$-Lipschitz. Since $g$ has $Q$ sequential blocks, the total Lipschitz constant is $\mu^{QQ'}\prod_{j=1}^{QQ'}\|\mathbf{W}_j\|_2$.
\end{proof}

\begin{lemma}[Lemma 3.2 in \citep{bartlett2017spectrally}]\label{lemma_cover_matrix}
    Let conjugate exponents $(p, q)$ and $(r, s)$ be given with $p \leq 2$, as well as positive reals $(a, b, \epsilon)$ and positive integer m. Let matrix $\mathbf{X} \in \mathbb{R}^{n\times d}$ be given with $\|\mathbf{X}\|_p \leq b$. Then
    \begin{equation*}
        \ln\mathcal{N}\left(\left\{\mathbf{XA}: \mathbf{A} \in \mathbb{R}^{d \times m}, \|\mathbf{A}\|_{q, s} \leq a\right\}, \epsilon, \|\cdot\|_F\right) \leq \Bigl \lceil \frac{a^2b^2m^{2/r}}{\epsilon^2}\Bigr \rceil\ln{(2dm)}
    \end{equation*}
\end{lemma}

\begin{lemma}\label{lemma_concate_lip}
    Suppose $\mathbf{Z}_i\in\mathbb{R}^{n\times d}$ and $f_i(\mathbf{Z}_i)$ is $L_i$-Lipschitz continuous with respect to $\mathbf{Z}_i$, where $i=1,\ldots,Q$. Let $\bar{\mathbf{Z}}=[\mathbf{Z}_i;\ldots;\mathbf{Z}_Q]\in\mathbb{R}^{n\times dQ}$. Let $F=[f_1,f_2,\ldots,f_Q]$. Then the Lipschitz constant of $F(\bar{\mathbf{Z}})$ with respect to $\bar{\mathbf{Z}}$ is 
 $L_F=\max_{i}\alpha_i$.
\end{lemma}
\begin{proof}
    Based on the settings, we have
    \begin{equation}
        \begin{aligned}
            &\|F(\bar{\mathbf{Z}})-F(\bar{\mathbf{Z}}')\|_F\\
            =&\left\|
            \begin{matrix}
                f_1({\mathbf{Z}}_1)-f_1({\mathbf{Z}}_1') & \ldots &f_Q({\mathbf{Z}}_Q)-f_Q({\mathbf{Z}}_Q')
            \end{matrix}
            \right\|_F\\
            =& \sqrt{\sum_{i=1}^Q\|f_i({\mathbf{Z}}_i)-f_i({\mathbf{Z}}_i')\|_F^2}\\
            \leq& \sqrt{\sum_{i=1}^Q\alpha_i^2\|{\mathbf{Z}}_i-{\mathbf{Z}}_i'\|_F^2}\\
            \leq& \max_{i}\alpha_i\sqrt{\sum_{i=1}^Q\|{\mathbf{Z}}_i-{\mathbf{Z}}_i'\|_F^2}\\
            =&\max_{i}\alpha_i\|\bar{\mathbf{Z}}-\bar{\mathbf{Z}}'\|_F\\
        \end{aligned}
    \end{equation}
\end{proof}

\begin{lemma}\label{lemma_cov_lip}
Suppose $\phi$ is an $\alpha$-Lipschitz continuous function, then $\ln \mathcal{N}(\epsilon, \phi \circ \mathcal{F}, \rho) \leq \ln \mathcal{N}(\epsilon / \alpha, \mathcal{F}, \rho)$.   
\end{lemma}
\begin{proof}
    This is a well-known result, and we will not repeat the proof.
\end{proof}

The theorem shows the impacts of model architecture, input data size, and weight matrices on the generalization ability of our model:

\begin{itemize}
    \item When the total number of training graphs $MN$ is larger, the bound is tighter, which is further verified by the experiments in Figure \ref{fig:number_pretrain}. Note that if we use the unsupervised contrastive loss to train the model, due to the data augmentation (though the samples are not independent), the generalization could be stronger. 
    \item Although $\beta$ often scales with $\sqrt{n}$, we have a factor $\tfrac{1}{\sqrt{n}}$ in $L_F$. This means that the number of nodes in each graph does not have a significant impact on the generalization, provided that the spectral norms of $\mathbf{A}_i^{(j)}$ increase slowly with $n$. As a result, our model will generalize well to both small graphs (e.g., ENZYMES) and large graphs (e.g., REDDIT), as shown by Tables \ref{tab:fewshot_res} and \ref{tab:fewshot_results2}.
    \item Since $L_F$ scales with $\mathcal{O}(\sqrt{\gamma R})$, we could use a relatively large $R$ to enrich the final vector representation for each graph, thereby improving the expressiveness. Moreover, $L_F$ is not very sensitive to $\gamma$, which is learned adaptively.
\end{itemize}

\section{Details of GIN and Graph Transformer based Model}

To design a universal graph representation model $F$, 
we incorporate two main components: a GNN module $f$ and a graph transformer module $g$.
We build a GIN encoder followed by a graph transformer encoder $g_{\psi}\circ f_{\theta}\left(\cdot\right)$. 
The GNN encoder specializes in learning
local representations of the structure of a node's immediate neighborhood, 
while the transformer computes all pairwise node interactions, 
enabling global reasoning through attention mechanisms.
Specifically, we adopt the Graph Isomorphism Network (GIN) \citep{xu2019powerful} as the GNN encoder, 
and its $l$-th layer can be formulated as
\begin{equation}
f^{(l)}\left(\mathbf{A}_i^{(j)}, \mathbf{Z}_i^{(j)}\right) = \mathrm{MLP}^{(l)}\left(\left(\tilde{\mathbf{A}}_i^{(j)} + \epsilon\mathbf{I}\right)\cdot\mathbf{Z}_i^{(j)}\right)
\end{equation}
where $\tilde{\mathbf{A}}_i^{(j)}$ is the adjacency matrix of $G_i^{(j)}$ with self-loops, 
$\epsilon$ is a hyperparameter, $\mathrm{MLP}^{(l)}$ is a multilayer perceptron (MLP) in layer $l$, and the parameters to optimize are denoted as $\theta$.

The graph transformer (GT) module consists of a self-attention mechanism and a feed-forward network, which is usually an MLP. 
Let $\mathbf{\Gamma}_i^{(j)} \in \mathbb{R}^{n_i\times d}$ represent the matrix of hidden states, and 
$\mathbf{W}_Q$, $\mathbf{W}_K$, and $\mathbf{W}_V$ of size $d\times d'$ be projection matrices, the self-attention mechanism is
\begin{equation}
    \mathrm{attn}\left(\mathbf{\Gamma}_i^{(j)}\right) = \mathrm{softmax}\left(\frac{(\mathbf{\Gamma}_i^{(j)}\mathbf{W}_Q)(\mathbf{\Gamma}_i^{(j)}\mathbf{W}_K)^\top}{\sqrt{d'}}\right)(\mathbf{\Gamma}_i^{(j)}\mathbf{W}_V)
\end{equation}
which is further transformed to $\hat{\mathbf{\Gamma}}_i^{(j)} = \mathrm{Norm}\left(\mathbf{\Gamma}_i^{(j)}+\mathrm{attn}\left(\mathbf{\Gamma}_i^{(j)}\right)\right)$.
Then the $l$-th transformer block can be formulated as
\begin{equation}
    g^{(l)}\left(\mathbf{\Gamma}_i^{(j)}\right) = \mathrm{Norm}\left(\hat{\mathbf{\Gamma}}_i^{(j)} + \mathrm{FFN}\left(\hat{\mathbf{\Gamma}}_i^{(j)}\right)\right) 
\end{equation}
We denote the parameters of the transformer module as $\psi$.
Finally, we concatenate the outputs of the GIN and GT, leading to the following node representations of $G_i^{(j)}$:
\begin{equation}
    \mathbf{H}_i^{(j)}=g_{\psi}\circ f_{\theta}\left(\mathbf{A}_i^{(j)},\mathbf{Z}_i^{(j)}\right) \Big\|f_\theta\left(\mathbf{A}_i^{(j)},\mathbf{Z}_i^{(j)}\right), \quad i\in[N_j],\quad j\in[M].
\end{equation}
For convenience, we let $\mathcal{W}=\{\psi,\theta\}$, which is the set of all parameters of the GIN and GT.

\section{Details about Experimental Settings}

\subsection{Datasets}
\label{datasets}
The basic information and statistics of the graph datasets we used in the experiments are shown in Table \ref{tab_datasets}. In our experiments, the concatenation of the original node attributes and node labels in the datasets is used as initial input node features. 

\begin{table}[h!]
    \centering
        \caption{Dataset Statistics.}
    \label{tab_datasets}
    \resizebox{\textwidth}{!}{
    \begin{tabular}{ccccccc}
    \toprule
       Dataset  &Domain& \#Graphs & \#Avg.Nodes & \#Features & \#Classes &  Task \\
    \midrule
       ENZYMES  & Bioinformatics&600 & 32.63 & 21 & 6 & Graph Classification/Graph Clustering  \\
        NCI1 & Small molecules&4110 & 29.87 & 37 & 2 &  Graph Classification/Graph Clustering \\
        NCI109 & Small molecules&4127 & 29.68 & 38 & 2 &  Graph Classification \\
        DD & Bioinformatics&1178 & 284.32 & 89 & 2 & Graph Classification  \\
        Mutagenicity & Small molecules&4337 & 30.32 & 14 & 2 &  Graph Classification \\
        COLLAB & Social networks&5000 & 74.49 & 0 & 2 &  Graph Classification/Graph Clustering \\
        REDDIT-BINARY & Social networks&2000 & 429.63 & 0 & 2 &  Graph Classification/Graph Clustering \\
        REDDIT-MULTI & Social networks&4999 & 508.52 & 0 & 5 &  Graph Clustering \\
        IMDB-BINARY & Social networks&1000 & 19.77 & 0 & 2 &  Graph Classification \\
        IMDB-MULTI & Social networks&1500 & 13.00 & 0 & 3 &  Graph Classification \\
        Letter-med & Computer vision&2250 & 4.67 & 2 & 15 &  Graph Classification \\
        COIL-RAG & Computer vision&3900 & 3.01 & 64 & 100 &  Graph Classification \\
        Cuneiform & Computer vision&267 & 21.27 & 10 & 30 &  Graph Classification \\
    \bottomrule
    \end{tabular}}
\end{table}

\subsection{Details of Model Testing in Few-Shot Graph Classification}
\label{test_detail}
Specifically, let the dataset in the downstream task be $\mathcal{G}^{\text{Down}} = \left\{\mathcal{G}^{\text{train}}, \mathcal{G}^{\text{test}}\right\}$, where $\mathcal{G}^{\text{train}}=\left\{\left(\mathbf{A}^{\text{train}}_i, \mathbf{X}^{\text{train}}_i\right)\right\}_{i=1}^{N_\text{train}}$ and $\mathcal{G}^{\text{test}}=\left\{\left(\mathbf{A}^{\text{test}}_i, \mathbf{X}^{\text{test}}_i\right)\right\}_{i=1}^{N_\text{test}}$. For $\mathcal{G}^{\text{train}}$, applying \eqref{eq_K_construct}, \eqref{eq_K_construct_Q}, \eqref{eq_Z_KSVD}, and \eqref{eq_ZZZ} sequentially, we obtain $\mathbf{Z}^{\text{train}}$, the aligned node feature matrix of the training set, which is further modified by using Algorithm \ref{alg_DMMA}. Now we apply the pretrained model to $\mathbf{Z}^{\text{train}}$ to obtain the embedding vector of each training graph, i.e., $\mathbf{g}_i^{\text{train}}=F_{\mathcal{W},\mathcal{V}, \gamma}(\mathbf{A}_i^{\text{train}},\mathbf{Z}_i^{\text{train}})$, $i\in N_{\text{train}}$.



Let the kernel matrix of the training set be $\mathbf{K}_{\lambda_q}^{\text{train}}=\mathbf{U}\boldsymbol{\Sigma}\mathbf{V}^\top$, and the cross-kernel matrix between the test and training sets be $\mathbf{K}_{\lambda_q}^{\text{test}}$. $\mathbf{Z}_{\lambda_q}^{\text{test}} = \mathbf{K}_{\lambda_q}^{\text{test}}\mathbf{V}_{\bar{d}}\boldsymbol{\Sigma}_{\bar{d}}^{-1/2}$, $q\in[Q]$. Then we obtain $\mathbf{Z}^{\text{test}}=\left[\mathbf{Z}_{\lambda_1}^{\text{test}},\mathbf{Z}_{\lambda_2}^{\text{test}},\ldots,\mathbf{Z}_{\lambda_Q}^{\text{test}}\right]$, the aligned node feature matrix of the testing set, which is further modified by using Algorithm \ref{alg_DMMA}.
Now, similar to the training data, we have $\mathbf{g}_i^{\text{test}}=F_{\mathcal{W},\mathcal{V}, \gamma}(\mathbf{A}_i^{\text{test}},\mathbf{Z}_i^{\text{test}})$, $i\in N_{\text{test}}$. These steps are summarized in Algorithm \ref{alg_few_shot}, where the underlined values are frozen in Algorithm 1.


\begin{algorithm}[H]
    \caption{Few-shot graph classification}
    \label{alg_few_shot}
\small{
\begin{algorithmic}[1]
    \REQUIRE {$\mathcal{G}^{\text{Down}} = \big\{\mathcal{G}^{\text{train}}, \mathcal{G}^{\text{test}}\big\}$, $\big\{\mathbf{R}_\text{pre}^{(j)}\big\}_{j=1}^{M}$, $\big\{\boldsymbol{\mu}_\text{pre}^{(j)}\big\}_{j=1}^{M}$}
    \vspace{-3pt}
    \STATE Compute $\mathbf{Z}^\text{train}$, $\mathbf{Z}^\text{test}$ using \eqref{eq_Z_KSVD} and \eqref{eq_ZZZ}.
    \STATE {$\mathbf{R}^\text{train}\gets$Algorithm~\ref{alg_DMMA}$\big(\boldsymbol{\mu}^\text{train}, \big\{\boldsymbol{\mu}_\text{pre}^{(j)}\big\}_{j=1}^{M}, \big\{\underline{\mathbf{R}}_\text{pre}^{(j)}\big\}_{j=1}^{M}\big)$\\
$\mathbf{R}^\text{test}\gets$Algorithm~\ref{alg_DMMA}$\big(\boldsymbol{\mu}^\text{test}, \boldsymbol{\mu}^\text{train}, \underline{\mathbf{R}}^\text{train}, \big\{\boldsymbol{\mu}_\text{pre}^{(j)}\big\}_{j=1}^{M}, \big\{\underline{\mathbf{R}}_\text{pre}^{(j)}\big\}_{j=1}^{M}\big)$}
    \STATE {Mean alignment: $\mathbf{Z}^\text{train}\gets\mathbf{Z}^\text{train}\mathbf{R}^{\text{train}^\top}$, $\mathbf{Z}^\text{test}\gets\mathbf{Z}^\text{test}\mathbf{R}^{\text{test}^\top}$}.
    \STATE{Representation: $\mathbf{g}_i^\text{train}\gets F_{\mathcal{W}, \mathcal{V}, \gamma}\big(\mathbf{A}_i^\text{train}, \mathbf{Z}_i^\text{train}\big)$, $i\in[|\mathcal{G}^{\text{train}}|]$}
    \STATE{Train the softmax classifier $f_c$ on $\{\mathbf{g}_i^\text{train}\}$. }
    \STATE $\hat{\mathbf{y}}_i^\text{test} = f_c\circ F_{\mathcal{W}, \mathcal{V}, \gamma}\big(\mathbf{A}_i^\text{test}, \mathbf{Z}_i^\text{test}\big)$, $i\in[|\mathcal{G}^{\text{test}}|]$.    
    \ENSURE Predicted graph labels $\{\hat{\mathbf{y}}_i^\text{test}\}$
\end{algorithmic}}
\end{algorithm}

\subsection{Inplementation Details}
\label{exp_set}
In our experiments, 
we use 6 Gaussian kernels with different $\lambda_q\in\{0.25, 0.5, 1, 2, 5, 10\}$. 
For all kernel matrices and the adjacency matrix, the truncated dimension of SVD $\bar{d}$ is set as 32. 
For each global graph obtained by Gaussian kernels, we use 6-GIN encoder to encode node features from different global graphs respectively. We implement each GIN encoder with 3 graph convolutional layers. The size of each hidden layer in GIN is set to 128. 
The graph transformer module consists of 3 equally wide layers, each containing 4 attention heads, 
with the dimension of each attention head set as 48. 
In the pretraining stage, all modules are optimized using Adam optimizer \citep{kinga2015method} with fixed learning rate $\alpha_1=0.0005$ and a weight decay factor of $10^{-5}$, trained for 50 epochs. 
The Gaussian kernel parameter $\gamma$ in the reference layer employs a separate learning rate $\alpha_2=0.1$. 
The batch size for all datasets is fixed to 64.

\paragraph{Few-shot learning settings} In the downstream tasks of few-shot graph classification, 
the classifier is a softmax classifier, which follows the setting in EdgePrompt \citep{fu2025edge}. Regarding data splitting, we randomly choose 50 graphs in each class for training, and the remaining samples are used for testing. 
The number of epochs is set to 500, and the learning rate of the classifier is set to 0.001 for graph few-shot training. 
As the k-shot tasks are balanced classification, we employ accuracy as the evaluation metric following EdgePrompt. \\
For ProNoG, we used the provided checkpoint from the official open-source repository as the pretrained model. For BRIDGE, GFT, and RiemannGFM, we followed the recommended settings in their paper to pretrain the model. The official repository of RiemannGFM does not support graph classification, we extend it to graph level task by using mean pooling.
In the downstream adaptation stage, we adopted the recommended hyperparameters for both methods. Experiments on COLLAB, REDDIT-B, IMDB-B, IMDB-M, and Letter-med are conducted under 50-shot setting following experiments in our paper. For COIL-RAG and Cuneiform, due to a lack of enough samples per class, we adopt 5-shot and 1-shot settings, respectively. Since the three baselines do not handle datasets without node attributes, to ensure fair comparison, we handle social network datasets without node attributes (COLLAB, REDDIT-B, IMDB-B, IMDB-M) uniformly across all methods. Following our proposed approach, we generate node attributes using truncated SVD on A+I (adjacency matrix with self-loops) as input for all baseline models. All the results of our method are obtained from models trained on 5 bio-chemical datasets (ENZYMES, DD, NCI1, NCI109, Mutagenicity) mentioned in the main part of the paper, which differ significantly from social networks and computer vision data in both semantics and structure. \\

In cross domain experiments, for graph-prompt-based baselines, which are designed to train and test within a single dataset, we preprocess the raw node attributes of each target dataset by PCA truncation to 32 dimensions; if the feature dimension is smaller than 32, we zero-pad it to 32 dimensions. Other baselines follow the same cross-domain protocol as GraphVec: they are pretrained on the five bio-chemical datasets (ENZYMES, DD, NCI1, NCI109, Mutagenicity) and evaluated on target graphs from other domains. In the downstream adaptation stage, we adopt the recommended hyperparameters for each baseline. Experiments on COLLAB, REDDIT-B, IMDB-B, IMDB-M, and Letter-med are conducted under the 50-shot setting. For COIL-RAG and Cuneiform, due to a lack of enough samples per class, we adopt 5-shot and 1-shot settings, respectively. SAMGPT is trained for 20 epochs, GOFA is fine-tuned for one epoch from the official checkpoint, and the remaining baselines are trained for 50 epochs. For LLM-based methods, raw graphs must first be converted into text-attributed graphs. Since the node features in our datasets are numerical rather than natural-language attributes, this conversion may produce highly similar text embeddings. In practice, we observe different degrees of representation collapse, which may explain the weak performance of LLM-based methods on several datasets. For datasets without node attributes, we handle them uniformly across all methods by generating structural node attributes using truncated SVD on $\mathbf{A}+\mathbf{I}$, following the same setting as GraphVec. \\

\textbf{Unsupervised pretraining settings}~~Following \citet{you2020graph}, we construct 3 augmentations using dropping nodes with a ratio of 0.1, permuting edges with a ratio of 0.1, and extracting subgraph for each graph before the global multi-graph construction and max-density mean alignment. 

We conduct all experiments on a 14 vCPU Intel(R) Xeon(R) Gold 6348 CPU with one Nvidia A800-80G GPU, CUDA 11.8. We repeat five times with different random seeds and report the average results with standard deviation calculated by the numpy library function.

\section{More Results}

\subsection{Intuitive Example of the Global Graph Construction}
Here we provide an intuitive example of synthetic data to show that our graph construction could be domain-agnostic. Suppose we have four datasets $\mathcal{D}_1, \mathcal{D}_2, \mathcal{D}_3, \mathcal{D}_4$ drawn from the following four distributions respectively: 1) $\mathcal{N}(\mathbf{0}, \mathbf{I}_2)$ (2D Gaussian); 2) $\mathcal{N}(\mathbf{0}, \mathbf{I}_2)$ (2D Gaussian); 3) $\mathcal{N}(\mathbf{1}, 2\mathbf{I}_3)$ (3D Gaussian); 4) $\mathcal{N}(-\mathbf{2}, \mathbf{I}_2) + \mathcal{N}(\mathbf{2}, \mathbf{I}_2)$ (2D Gaussian mixture model). Thus, $\mathcal{D}_2$ can be regarded as a dataset from the same domain as $\mathcal{D}_1$, while $\mathcal{D}_3$ and $\mathcal{D}_4$ are from different domains. We calculate the Gromov-Wasserstein distances between the weighted graphs constructed from the four datasets using the method proposed in our paper. The results are shown in the following table (average of 5 runs). We see that the distance between $\mathcal{D}_1$ and $\mathcal{D}_3$ is close to that between $\mathcal{D}_1$ and $\mathcal{D}_2$, meaning that the features generated by our multi-graph alignment method are indeed domain agnostic. The distance between $\mathcal{D}_1$ and $\mathcal{D}_4$ is much larger than that between $\mathcal{D}_1$ and $\mathcal{D}_3$, meaning that our method can effectively identify the topological difference between the datasets.

\begin{table}[h!]
\centering
\caption{Gromov-Wasserstein distances between synthetic datasets (average of 5 runs)}
\begin{tabular}{lrrrr}
\toprule
 & \textbf{$\mathcal{D}_1 \sim \mathcal{N}(\mathbf{0}, \mathbf{I}_2)$} & \textbf{$\mathcal{D}_2 \sim \mathcal{N}(\mathbf{0}, \mathbf{I}_2)$} & \textbf{$\mathcal{D}_3 \sim \mathcal{N}(\mathbf{1}, 2\mathbf{I}_3)$} & \textbf{$\mathcal{D}_4 \sim \mathcal{N}(-\mathbf{2}, \mathbf{I}_2) + \mathcal{N}(\mathbf{2}, \mathbf{I}_2)$} \\
\midrule
$\mathcal{D}_1 \sim \mathcal{N}(\mathbf{0}, \mathbf{I}_2)$ & 0 & 0.004 & 0.015 & 0.069 \\
$\mathcal{D}_2 \sim \mathcal{N}(\mathbf{0}, \mathbf{I}_2)$ & --- & 0 & 0.015 & 0.069 \\
$\mathcal{D}_3 \sim \mathcal{N}(\mathbf{1}, 2\mathbf{I}_3)$ & --- & --- & 0 & 0.081 \\
$\mathcal{D}_4 \sim \mathcal{N}(-\mathbf{2}, \mathbf{I}_2) + \mathcal{N}(\mathbf{2}, \mathbf{I}_2)$ & --- & --- & --- & 0 \\
\bottomrule
\end{tabular}
\end{table}

\subsection{Full Results of Table \ref{tab:fewshot_res}}
The full compared numbers in the baseline of Table \ref{tab:fewshot_res} are provided in Table \ref{tab:fullres}.
\label{app_fullres}
\begin{table*}
    \centering
        \caption{50-shot graph classification performance comparison with different pretrained models. We color the \textcolor{red}{\textbf{best}} and \textcolor{orange}{\textbf{second best}} models. The compared numbers of in-domain experiments are from EdgePrompt \citep{fu2025edge}. }
        \label{tab:fullres}
    \resizebox{\textwidth}{!}{
\begin{tabular}{c|c|c|c|c|c|c|c}
\toprule pretraining & Tuning Methods & ENZYMES & DD & NCI1 & NCI109 & Mutagenicity  &Average\\ 
\midrule
\multirow{7}{*}{GraphCL} & Classifier Only & $30.50_{ \pm 1.16}$ & $62.89_{ \pm 2.19}$ & $62.49_{ \pm 1.95}$ & $61.68_{ \pm 0.93}$ & $66.62_{ \pm 1.87}$  &$56.84$
\\
 & GraphPrompt \citep{liu2023graphprompt}& $27.83_{ \pm 1.61}$ & $64.33_{ \pm 1.79}$ & $63.19_{ \pm 1.71}$ & $62.18_{ \pm 0.48}$ & $6 7 . 6 2_{ \pm 0 . 6 5}$ &$57.03$
\\
 & ALL-in-one \citep{sun2023all}& $25.92_{ \pm 0.55}$ & $66.54_{ \pm 1.82}$ & $57.52_{ \pm 2.61}$ & $62.74_{ \pm 0.78}$ & $63.43_{ \pm 2.53}$  &$55.23$
\\
 & GPF \citep{fang2023universal}& $30.08_{ \pm 1.25}$ & $64.54_{ \pm 2.22}$ & $62.66_{ \pm 1.83}$ & $62.29_{ \pm 0.90}$ & $66.54_{ \pm 1.85}$  &$57.22$
\\
 & GPF-plus \citep{fang2023universal}& $31.00_{ \pm 1.50}$ & $67.26_{ \pm 2.29}$ & $64.56_{ \pm 1.10}$ & $62.84_{ \pm 0.22}$ & $66.82_{ \pm 1.63}$  &$58.50$
\\
& EdgePrompt \citep{fu2025edge} & $29.50_{ \pm 1.57}$ & $64.16_{ \pm 2.13}$ & $63.05_{ \pm 2.11}$ & $62.59_{ \pm 0.93}$ & $66.87_{ \pm 1.88}$  &$57.23$
\\
 & EdgePrompt+ \citep{fu2025edge}& ${3 4 . 0 0}_{ \pm 1 . 2 5}$& $6 7 . 9 8_{ \pm 2 . 0 5}$& $6 6 . 3 0_{ \pm 2 . 5 4}$ & $66.52_{ \pm 0.91}$ & $67.47_{ \pm 2.37}$ &$60.45$
\\
\midrule \multirow{7}{*}{SimGRACE} & Classifier Only & $27.07_{ \pm 1.04}$ & $61.77_{ \pm 2.40}$ & $61.27_{ \pm 3.64}$ & $62.12_{ \pm 1.10}$ & $67.36_{ \pm 0.71}$  &$55.92$
\\
 & GraphPrompt \citep{liu2023graphprompt}& $26.87_{ \pm 1.47}$ & $62.58_{ \pm 1.84}$ & $62.45_{ \pm 1.52}$& $62.41_{ \pm 0.69}$ & $68.03_{ \pm 0.78}$  &$56.47$
\\
 & ALL-in-one \citep{sun2023all}& $25.73_{ \pm 1.18}$ & $65.16_{ \pm 1.47}$ & $58.52_{ \pm 1.59}$ & $62.01_{ \pm 0.66}$ & $64.43_{ \pm 1.00}$  &$55.17$
\\
 & GPF \citep{fang2023universal}& $28.53_{ \pm 1.76}$ & $65.64_{ \pm 0.70}$ & $61.45_{ \pm 3.13}$ & $61.90_{ \pm 1.26}$ & $67.19_{ \pm 0.74}$  &$56.94$
\\
 & GPF-plus \citep{fang2023universal}& $27.33_{ \pm 2.01}$ & $67.20_{ \pm 1.56}$& $61.61_{ \pm 2.89}$ & $62.84_{ \pm 0.23}$ & $67.69_{ \pm 0.64}$  &$57.33$
\\
 & EdgePrompt \citep{fu2025edge}& $29.33_{ \pm 2.30}$ & $63.97_{ \pm 2.14}$ & $62.02_{ \pm 3.02}$ & $62.02_{ \pm 1.03}$ & $67.55_{ \pm 0.85}$  &$56.98$
\\
 & EdgePrompt+ \citep{fu2025edge}& $3 2 . 6 7_{ \pm 2 . 5 3}$& $67 . 72_{ \pm 1.62}$& $\textcolor{orange}{\mathbf{67.07}}_{ \pm 1 . 9 6}$& $\textcolor{orange}{\mathbf{66.53}}_{ \pm 1.30}$& ${{68.31}}_{ \pm 1 .36}$ &$60.46$
\\
\midrule \multirow{7}{*}{EP-GPPT} & Classifier Only & $29.08_{ \pm 1.35}$ & $62.12_{ \pm 2.82}$ & $56.85_{\pm 4.35}$& $62.27_{ \pm 0.78}$ & $66.30_{ \pm 1.78}$  &$55.32$
\\
 & GraphPrompt \citep{liu2023graphprompt}& $26.67_{ \pm 1.60}$ & $61.61_{ \pm 1.91}$ & $58.77_{ \pm 0.97}$ & $62.16_{ \pm 0.89}$ & $66.37_{ \pm 1.17}$  &$55.12$
\\
 & ALL-in-one \citep{sun2023all}& $24.92_{ \pm 1.33}$ & $63.61_{ \pm 2.12}$ & $59.14_{ \pm 2.12}$ & $59.70_{ \pm 1.37}$ & $64.86_{ \pm 1.60}$  &$54.45$
\\
 & GPF \citep{fang2023universal}& $28.33_{ \pm 1.73}$ & $63.48_{ \pm 2.08}$ & $58.14_{ \pm 4.16}$ & $62.52_{ \pm 1.39}$ & $66.10_{ \pm 0.96}$  &$55.71$
\\
 & GPF-plus \citep{fang2023universal}& $29.25_{ \pm 1.30}$ & $66.92_{ \pm 2 .34}$& $62.93_{ \pm 3.23}$ & $64.13_{ \pm 1.42}$ & $67.57_{ \pm 1.45}$  &$58.16$
\\
 & EdgePrompt \citep{fu2025edge}& $28.33_{ \pm 3.41}$ & $64.03_{ \pm 2.26}$ & $59.85_{ \pm 3.15}$ & $62.98_{ \pm 1.44}$ & $66.36_{ \pm 1.22}$  &$56.31$
\\
 & EdgePrompt+ \citep{fu2025edge}& $3 2 . 7 5_{ \pm 2 . 2 6}$& $66.16_{ \pm 1.60}$& $63.58_{ \pm 2.07}$ & $6 5 . 1 5_{ \pm 1 . 6 0}$& ${{6 8 . 3 5}}_{ \pm 1 . 5 7}$&$59.20$
\\
\midrule \multirow{7}{*}{\makecell{EP-\\GraphPrompt}} & Classifier Only & $31.33_{ \pm 3.22}$ & $62.58_{ \pm 2.40}$ & $62.09_{ \pm 2.31}$ & $60.19_{ \pm 1.71}$ & $65.13_{ \pm 0.81}$  &$55.32$
\\
 & GraphPrompt \citep{liu2023graphprompt}& $30.20_{ \pm 1.93}$ & $64.72_{ \pm 1.98}$ & $62.57_{ \pm 1.45}$ & $62.32_{ \pm 0.95}$ & $65.85_{ \pm 0.65}$  &$57.13$
\\
 & ALL-in-one \citep{sun2023all}& $29.07_{ \pm 1.16}$ & $65.60_{ \pm 2.38}$ & $58.67_{ \pm 2.42}$ & $57.69_{ \pm 1.08}$ & $64.66_{ \pm 0.76}$  &$55.14$
\\
 & GPF \citep{fang2023universal}& $30.93_{ \pm 1.76}$ & $66.21_{ \pm 1.66}$ & $61.80_{ \pm 2.78}$ & $62.27_{ \pm 1.18}$ & $65.61_{ \pm 0.59}$  &$57.36$
\\
 & GPF-plus \citep{fang2023universal}& $30.67_{ \pm 3.06}$ & $67.50_{ \pm 2.45}$& $62.59_{ \pm 2.09}$ & $61.98_{ \pm 1.60}$ & $65.51_{ \pm 1.10}$  &$57.65$
\\
 & EdgePrompt \citep{fu2025edge}& $30.80_{ \pm 2.09}$ & $65.87_{ \pm 1.35}$ & $61.75_{ \pm 2.49}$ & $62.33_{ \pm 1.65}$ & $65.77_{ \pm 0.90}$  &$57.30$
\\
 & EdgePrompt+ \citep{fu2025edge}& $33.27_{ \pm 2.71}$ & $67.47_{ \pm 2.14}$& $6 5 . 0 6_{ \pm 1.84}$& $6 4 . 6 4_{ \pm 1.57}$& $66 . 42_{ \pm 1.31}$ &$59.37$\\
\midrule
\multirow{4}{*}{cross-domain}&{GCN \citep{kipf2016semi}} & $43.33_{ \pm 1.05}$ & $65.84_{ \pm 2.77}$ & $61.36_{ \pm 2.00}$& $62.17_{ \pm 0.66}$& $60.46_{ \pm 1.75}$  &$58.63$\\
&{BRIDGE \citep{yuan2025much}} & $36.67_{ \pm 5.96}$ & $64.95_{ \pm 3.38}$ & $63.50_{ \pm 2.27}$& $61.78_{ \pm 1.63}$& $65.12_{ \pm 2.83}$  &$58.40$\\
&{GFT \citep{wang2024gft}} & $34.61_{ \pm 3.12}$ & $56.00_{ \pm 1.77}$ & $59.16_{ \pm 6.25}$& $60.50_{ \pm 2.71}$& $67.82_{ \pm 3.18}$  &$55.61$\\
&{RiemannGFM \citep{sun2025riemanngfm}} & $34.27_{ \pm 1.72}$ & $68.74_{ \pm 1.31}$ & $55.10_{ \pm 2.24}$& $59.86_{ \pm 1.30}$& $62.56_{ \pm 4.04}$  &$56.11$\\

\midrule
\multicolumn{2}{c|}{\textbf{GraphVec}} & $\textcolor{red}{\mathbf{51.00}}_{ \pm 3.22}$ & $\textcolor{red}{\mathbf{75.94}}_{ \pm 2.70}$ & $\textcolor{red}{\mathbf{67.32}}_{ \pm 1.51}$& $\textcolor{red}{\mathbf{67.90}}_{ \pm 1.67}$& $\textcolor{red}{\mathbf{68.57}}_{ \pm 1.62}$  &$\textcolor{red}{\mathbf{66.14}}$\\
\multicolumn{2}{c|}{\textbf{Unsupervised GraphVec}} & ${{48.33}}_{ \pm 2.36}$ & $\textcolor{orange}{\mathbf{74.02}}_{ \pm 1.26}$ & $66.11_{ \pm 2.30}$& $64.34_{ \pm 2.38}$& $\textcolor{orange}{\mathbf{68.38}}_{ \pm 2.88}$  &$\textcolor{orange}{\mathbf{64.23}}$\\

\multicolumn{2}{c|}{\textbf{GraphVec} w/o mean alignment} & $\textcolor{orange}{\mathbf{49.33}}_{ \pm 1.48}$ & ${73.14}_{ \pm 1.16}$ & $65.80_{ \pm 1.40}$& $65.21_{ \pm 2.05}$& $67.00_{ \pm 1.31}$  &${{64.09}}$
\\
\bottomrule
\end{tabular}}
\vspace{-5pt}
\end{table*}

\subsection{Extension to Node Classification Task}
\label{app_nodecls}
Our model can also be extended to node-level tasks by retraining it with a node-level contrastive loss objective. Specifically, we maintain the construction of global multi-graphs and the mean alignment module from our original framework, while removing the reference distribution layers since graph-level representations are not required here. The node embeddings are obtained directly from the outputs of both the Graph Transformer and GIN modules. These embeddings are then fed into a linear classifier to perform the downstream node classification task.

The model was evaluated on 4 node classification datasets: Cora, CiteSeer, PubMed, and ogbn-arxiv. We follow a leave-one-dataset-out protocol: for each target dataset, GraphVec is pretrained on the remaining three datasets and then evaluated on the target dataset under the 5-shot setting. The node embeddings produced by the pretrained encoder are used to train a linear classifier with the labeled nodes, and the remaining nodes are used for testing. The results are shown as follows. This adaptation demonstrates that GraphVec can also be effectively extended to node-level tasks. 

\begin{table}[h!]
\centering
\caption{5-shot node classification results on node-level tasks}
\begin{tabular}{lrrrr}
\toprule
\textbf{Methods} & \textbf{Cora} & \textbf{CiteSeer} & \textbf{Pubmed} & \textbf{ogbn-arxiv} \\
\midrule
GPPT \citep{sun2022gppt} & 41.28 $\pm$ 6.24 & 35.32 $\pm$ 1.27 & 53.41 $\pm$ 3.99 & 17.73 $\pm$ 1.66 \\
GraphPrompt \citep{liu2023graphprompt} & 31.65 $\pm$ 3.33 & 26.98 $\pm$ 1.24 & 44.18 $\pm$ 5.57 & 16.11 $\pm$ 1.42 \\
ALL-in-one \citep{sun2023all} & 31.57 $\pm$ 2.86 & 29.76 $\pm$ 1.53 & 46.89 $\pm$ 5.35 & 17.89 $\pm$ 1.21 \\
GPF \citep{fang2023universal} & 37.56 $\pm$ 3.81 & 29.74 $\pm$ 1.73 & 48.16 $\pm$ 3.32 & 17.64 $\pm$ 1.18 \\
GPF-plus \citep{fang2023universal} & 28.87 $\pm$ 3.18 & 26.65 $\pm$ 1.91 & 43.02 $\pm$ 4.59 & 17.39 $\pm$ 1.27 \\
EdgePrompt \citep{fu2025edge} & 37.26 $\pm$ 4.53 & 29.83 $\pm$ 1.01 & 45.49 $\pm$ 3.27 & 17.82 $\pm$ 1.59 \\
EdgePrompt + \citep{fu2025edge} & 56.41 $\pm$ 3.62 & 43.49 $\pm$ 2.62 & 61.51 $\pm$ 4.91 & 17.78 $\pm$ 2.12 \\
\midrule
\textbf{GraphVec} & \textbf{58.66 $\pm$ 1.51} & \textbf{45.45 $\pm$ 1.26} & \textbf{65.72 $\pm$ 2.43} & \textbf{23.75 $\pm$ 1.39} \\
\bottomrule
\end{tabular}
\end{table}
\subsection{The impact of Nyström Approximation }
\label{app_nystrom}
To address the computational complexity associated with large-scale graphs, we employ the Nyström approximation during pretraining. To systematically evaluate its impact, we pre-train GraphVec using the Nyström method with varying sample sizes, while keeping the pretraining datasets consistent with our main experiments. In the following table, we specifically report the wall-clock time required for constructing the global multi-graphs, and evaluate downstream performance via 50-shot graph classification accuracy on the Letter-med dataset. We see that, when the sample size is less than 2000, the time used for constructing global multi-graphs is acceptable. When adding the sample size to 4000, the wall-clock time increases sharply while the performance improvement is marginal (less than 0.1\%).
\begin{table}[h!]
\centering
\caption{Accuracy--efficiency trade-off of Nystr\"om approximation on Letter-med.}
\label{tab:nystrom_tradeoff}
\begin{tabular}{lcccc}
\toprule
\# Nystr\"om samples & 100 & 1000 & 2000 & 4000 \\
\midrule
Wall-clock time (s) & 34.82 & 48.57 & 63.55 & 716.61 \\
Accuracy (\%) & $81.50 \pm 2.21$ & $82.47 \pm 0.92$ & $84.27 \pm 1.10$ & $84.33 \pm 1.60$ \\
\bottomrule
\end{tabular}
\end{table}

\subsection{Impact of Number of pretraining Datasets}\label{app_impact}
Figure \ref{fig:number_pretrain} shows the change of classification accuracy when the number of datasets used in pretraining increases from $1$ to $4$. We can see that with more datasets used in pretraining, the performance in downstream tasks becomes better. This result indicates that the generalization ability of graph embeddings generated by our GraphVec can benefit from the increase in the number of training datasets, which is an important capability for cross domain pretrained graph model. It can also be observed that even using model pretrained on only 1 dataset, GraphVec still outperforms other baselines shown in Table \ref{tab:fewshot_res}.
\begin{figure}[h!]
    \centering
    \includegraphics[width=0.6\linewidth]{ 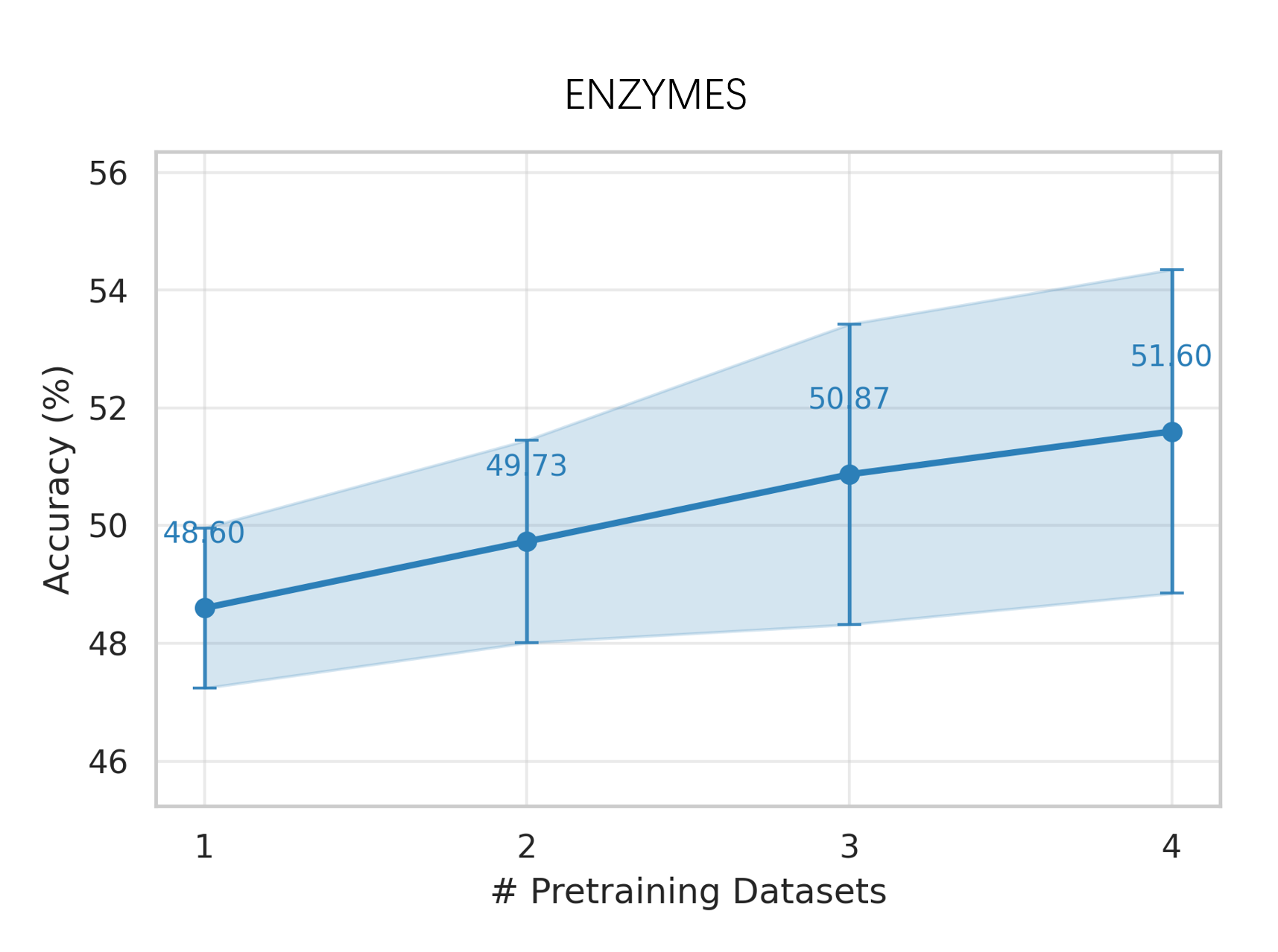}
    \caption{The change of classification accuracy in ENZYMES when the number of datasets used in pretraining increases from $1$ to $4$.}
    \label{fig:number_pretrain}
\end{figure}

\subsection{Few-shot Learning with Fewer Labeled Samples}\label{app_1_shot}
As shown in Figure \ref{fig:app_1_shot}, the classification accuracy of our method GraphVec increases as the number of labeled samples increases. Our GraphVec with 20-shot even outperforms the competitors with 50-shot in Table \ref{tab:fewshot_res} of the main paper. 
\begin{figure}[h!]
    \centering
    \includegraphics[width=0.7\linewidth]{ 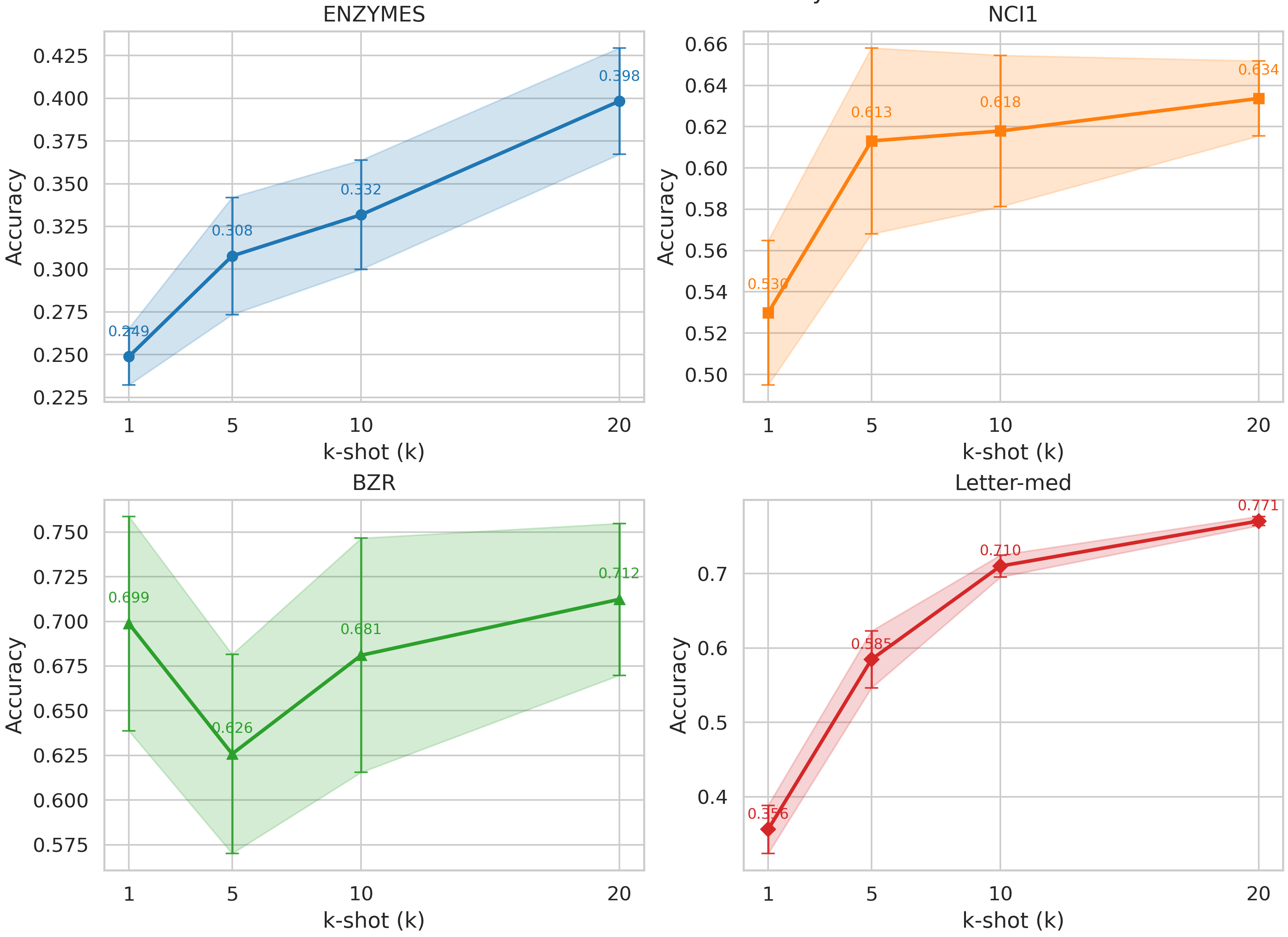}
    \caption{Classification accuracy trends of our method GraphVec with varying $k$ values in few-shot learning across four datasets (PROTEINS, NCI109, DD, and Mutagenicity), shaded area represents standard deviation}
    \label{fig:app_1_shot}
\end{figure}

\subsection{Comparison between Original Attributes and Similarity-based Features}
\label{app_attr_semantics}
\begin{table}[ht]
\centering
\caption{Supervised graph classification accuracy of a standard GIN using original node attributes and our similarity-based features. The train/validation/test split is 8/1/1.}
\label{tab:semantic_feature_compare}
\begin{tabular}{lcccc}
\toprule
Feature Type & PROTEINS & NCI1 & NCI109 & ENZYMES \\
\midrule
Original node attributes & $72.81 \pm 4.08$ & $78.88 \pm 1.50$ & $81.60 \pm 1.62$ & $66.33 \pm 6.09$ \\
Similarity-based features & $71.43 \pm 2.33$ & $80.00 \pm 1.99$ & $78.31 \pm 2.24$ & $67.00 \pm 5.81$ \\
\bottomrule
\end{tabular}
\end{table}

Table \ref{tab:semantic_feature_compare} shows that replacing the original attributes with the features produced by our global multi-graph construction leads to comparable supervised performance on datasets with semantically meaningful node attributes. The transformed features slightly outperform the raw attributes on NCI1 and ENZYMES, while remaining competitive on PROTEINS and NCI109. This result suggests that the proposed feature construction does not simply discard useful information; rather, it preserves a substantial portion of the task-relevant signal while converting heterogeneous attributes into a more transferable relational representation. In these experiments, mean alignment is applied to remove the sign ambiguity induced by SVD.

\subsection{Comparison with Backbone-only Variants}

\begin{table}[ht]
\centering
\caption{50-shot cross-dataset graph classification accuracy on the five biochemical datasets. GraphVec denotes the full model, while GIN and GT are trained under the same evaluation protocol.}
\label{tab:backbone_compare}
\begin{tabular}{lccccc}
\toprule
Method & ENZYMES & DD & NCI1 & NCI109 & Mutagenicity \\
\midrule
GIN & 43.53 & 60.53 & 59.45 & 60.84 & 64.34 \\
GT & 44.53 & 65.52 & 54.43 & 55.20 & 61.48 \\
GraphVec & 51.00 & 75.94 & 67.32 & 67.90 & 68.57 \\
\bottomrule
\end{tabular}
\end{table}

Table \ref{tab:backbone_compare} addresses a potential confound: whether the performance gain mainly comes from the expressive backbone rather than from the proposed pretraining and alignment design. The answer is negative. Although GraphVec is built on top of GIN and graph transformer components, the full model consistently outperforms either backbone alone on all five datasets, with particularly large margins on DD, NCI1, and NCI109. Therefore, the improvement cannot be attributed only to backbone choice; the multi-graph feature construction, mean alignment, and reference distribution module contribute materially to the final performance.

\subsection{Comparison with Simpler Cross-domain Feature Alternatives}

\begin{table}[ht]
\centering
\caption{Cross-domain feature ablation on computer vision datasets. All models use the same mean alignment, backbone, reference layers, and training strategy; only the input feature construction is changed.}
\label{tab:cross_domain_feature_ablation}
\begin{tabular}{lccc}
\toprule
Feature Type & Letter-med & COIL-RAG & Cuneiform \\
\midrule
Original node attributes & $55.33 \pm 1.78$ & $58.12 \pm 1.97$ & $33.56 \pm 3.52$ \\
PCA features & $56.73 \pm 1.93$ & $10.29 \pm 0.45$ & $28.05 \pm 7.07$ \\
Global multi-graph features & $85.60 \pm 1.44$ & $74.20 \pm 0.77$ & $55.86 \pm 8.15$ \\
\bottomrule
\end{tabular}
\end{table}

Table \ref{tab:cross_domain_feature_ablation} further shows that the advantage of GraphVec is not explained by dimensional alignment alone. When the global multi-graph features are replaced by zero-padded original attributes or PCA-based features, while keeping all other components unchanged, cross-domain performance drops substantially on all three computer vision datasets. The degradation is especially severe for PCA on COIL-RAG and Cuneiform, indicating that simple linear dimension reduction does not yield a transferable relational space. In contrast, the proposed global multi-graph construction provides much stronger cross-domain invariance, which supports the role of kernelized sample relationships as the key mechanism for bridging domain-specific feature spaces.

\subsection{Evaluation on Generated Node Attributes $\text{SVD}(\mathbf{A}+\mathbf{I})$}
For the dataset without original node attributes, the classification mainly relies on discriminating between different structures of graphs. While the topologically derived features may not carry explicit domain semantics like chemical properties or pixel coordinates, they can be viewed as a form of generic node attribute derived from the graph connectivity. The core of our methodology, specifically the global multi-graph, is designed to bridge the inherent semantic gaps between different domains, regardless of whether the original features are rich in semantics or purely structural.

To evaluate the impact of the feature generation method , we conducted a controlled experiment to replace truncated SVD with two kinds of node centrality. The results are shown in the Table \ref{tab:gen_nodeattr}.

\begin{table}[ht]
\centering
\caption{Evaluation on Generated Node Attributes.}
\label{tab:gen_nodeattr}
\begin{tabular}{lcccc}
\toprule
 & Degree Centrality & Betweenness Centrality & Degree + Betweenness & Our Method \\
\midrule
REDDIT - B & 73.95 $\pm$ 2.30 & 74.92 $\pm$ 1.68 & 77.79 $\pm$ 2.76 & 81.52 $\pm$ 1.50 \\
\bottomrule
\end{tabular}
\end{table}

As shown in the table, our method demonstrates a clear advantage over using only a single type of centrality, and also achieves a marginal improvement compared to combining both centrality measures. It is worth noting that the attributes generated by our approach can be viewed as a form of structural encoding. While other types of structural encodings may also be effective, the performance gain observed here can be largely attributed to our proposed global multi-graph construction, which enhances the model's ability to capture feature information from different spaces.

\subsection{More Graph Clustering Results}
\label{app_more_clus}
To further evaluate the performance of GraphVec in graph clustering task, we conduct experiments on 4 more datasets. The results are shown in Table \ref{tab:more_cluster_res}. We also provide the full version of Table \ref{tab:cluster_res} with NMI in Table \ref{tab:full_cluster_res}
\begin{table}[h!]
\centering
\caption{Graph clustering results on PTC-MM, MUTAG, COX2 and BZR}
\label{tab:more_cluster_res}
\resizebox{1.1\textwidth}{!}{ 
\begin{tabular}{l *{12}{c}}
\toprule
\multirow{2}{*}{Dataset} & \multicolumn{3}{c}{PTC-MM} & \multicolumn{3}{c}{MUTAG} & \multicolumn{3}{c}{COX2} & \multicolumn{3}{c}{BZR} \\
\cmidrule(lr){2-4} \cmidrule(lr){5-7} \cmidrule(lr){8-10} \cmidrule(lr){11-13}
 & ACC & NMI & ARI & ACC & NMI & ARI & ACC & NMI & ARI & ACC & NMI & ARI \\
\midrule
GraphCL+SC & $62.09 \pm 0.56$ & $2.14 \pm 0.43$ & $3.36 \pm 0.87$ & $73.22 \pm 2.66$ & $\mathbf{32.19} \pm 2.05$ & $23.44 \pm 2.45$ & $75.01 \pm 2.12$ & $1.24 \pm 0.37$ & $\mathbf{2.39} \pm 2.28$ & $72.88 \pm 1.66$ & $\mathbf{1.90} \pm 0.38$ & $\mathbf{3.47} \pm 0.59$ \\
GWF \citep{xu2022representing}+SC & $53.02 \pm 1.66$ & $0.36 \pm 0.28$ & $0.21 \pm 0.09$ & $73.92 \pm 4.30$ & $18.35 \pm 3.85$ & $24.48 \pm 4.69$ & $58.83 \pm 4.46$ & $1.16 \pm 0.41$ & $1.45 \pm 1.21$ & $52.76 \pm 0.80$ & $3.47 \pm 1.16$ & $-0.71 \pm 0.32$ \\
GLCC \citep{ju2023glcc}& $61.61 \pm 0.24$ & $0.63 \pm 0.41$ & $1.24 \pm 1.38$ & $71.99 \pm 3.08$ & $13.18 \pm 6.93$ & $16.89 \pm 8.28$ & $77.37 \pm 1.11$ & $0.02 \pm 0.03$ & $-0.30 \pm 0.42$ & $63.62 \pm 9.79$ & $1.18 \pm 0.60$ & $1.12 \pm 0.97$ \\
Our Method & $\mathbf{65.74} \pm 0.00$ & $\mathbf{4.35} \pm 0.00$ & $\mathbf{6.31} \pm 0.00$ & $\mathbf{81.38} \pm 0.00$ & $31.00 \pm 0.00$ & $\mathbf{38.96} \pm 0.00$ & $\mathbf{78.58} \pm 0.00$ & $\mathbf{2.37} \pm 0.00$ & $2.19 \pm 0.00$ & $\mathbf{77.28} \pm 0.00$ & $0.16 \pm 0.00$ & $1.50 \pm 0.00$ \\
\bottomrule
\end{tabular}
}
\end{table}
\begin{table*}[htbp]
\centering
\caption{Graph clustering performance on ENZYMES, NCI1, COLLAB, REDDIT-BINARY, REDDIT-MULTI. The comparison numbers are from AMGC \citep{yang2025towards}. }
\label{tab:full_cluster_res}
\resizebox{\textwidth}{!}{
\begin{tabular}{l|ccc|ccc|ccc|ccc|ccc}
\toprule
Method & \multicolumn{3}{c|}{ENZYMES} & \multicolumn{3}{c|}{NCI1} & \multicolumn{3}{c|}{COLLAB} & \multicolumn{3}{c|}{REDDIT-BINARY} & \multicolumn{3}{c}{REDDIT-MULTI} \\
 & ACC & NMI & ARI & ACC & NMI & ARI & ACC & NMI & ARI & ACC & NMI & ARI & ACC & NMI & ARI \\
\midrule
RW +SC & $17.0 _{\pm 0.0}$ & $0.7 _{\pm 0.0}$ & $0.3 _{\pm 0.0}$ & N/A & N/A & N/A & N/A & N/A & N/A & N/A & N/A & N/A & N/A & N/A & N/A \\
WL +SC & $21.0 _{\pm 0.0}$ & $3.1 _{\pm 0.0}$ & $1.5 _{\pm 0.0}$ & $50.1 _{\pm 0.0}$ & $0.0 _{\pm 0.0}$ & $0.0 _{\pm 0.0}$ & $53.2 _{\pm 0.0}$ & $2.0 _{\pm 0.0}$ & $0.5 _{\pm 0.0}$ & $57.6 _{\pm 0.0}$ & $9.0 _{\pm 0.0}$ & $2.2 _{\pm 0.0}$ & $18.7 _{\pm 0.0}$ & $9.0 _{\pm 0.0}$ & $4.0 _{\pm 0.0}$ \\
WL-OA +SC & $20.0 _{\pm 0.0}$ & $1.4 _{\pm 0.0}$ & $0.3 _{\pm 0.0}$ & $53.2 _{\pm 0.0}$ & $0.9 _{\pm 0.0}$ & $0.8 _{\pm 0.0}$ & $54.2 _{\pm 0.0}$ & $0.2 _{\pm 0.0}$ & $2.6 _{\pm 0.0}$ & $53.8 _{\pm 0.0}$ & $5.6 _{\pm 0.0}$ & $3.8 _{\pm 0.0}$ & $20.9 _{\pm 0.0}$ & $9.6 _{\pm 0.0}$ & $3.2 _{\pm 0.0}$ \\
SP +SC & $22.0 _{\pm 0.0}$ & $2.6 _{\pm 0.0}$ & $1.7 _{\pm 0.0}$ & $50.1 _{\pm 0.0}$ & $0.1 _{\pm 0.0}$ & $0.0 _{\pm 0.0}$ & $48.7 _{\pm 0.0}$ & $17.9 _{\pm 0.0}$ & $\textcolor{orange}{\mathbf{13.9}} _{\pm 0.0}$ & $57.8 _{\pm 0.0}$ & $2.2 _{\pm 0.0}$ & $2.2 _{\pm 0.0}$ & $20.3 _{\pm 0.0}$ & $6.1 _{\pm 0.0}$ & $0.1 _{\pm 0.0}$ \\
LT +SC & $17.0 _{\pm 0.0}$ & $0.4 _{\pm 0.0}$ & $0.0 _{\pm 0.0}$ & N/A & N/A & N/A & N/A & N/A & N/A & N/A & N/A & N/A & N/A & N/A & N/A \\
GK +SC & $17.1 _{\pm 0.1}$ & $0.8 _{\pm 0.3}$ & $0.0 _{\pm 0.0}$ & $52.9 _{\pm 0.9}$ & $0.7 _{\pm 1.4}$ & $0.3 _{\pm 0.6}$ & $56.8 _{\pm 1.4}$ & $15.5 _{\pm 1.9}$ & $9.3 _{\pm 2.1}$ & $50.3 _{\pm 0.3}$ & $0.2 _{\pm 0.1}$ & $0.0 _{\pm 0.0}$ & $18.7 _{\pm 0.9}$ & $7.2 _{\pm 0.3}$ & $0.3 _{\pm 0.1}$ \\
\midrule
InfoGraph +KM & $22.1 _{\pm 1.0}$ & $2.4 _{\pm 0.5}$ & $1.3 _{\pm 0.5}$ & $54.1 _{\pm 2.2}$ & $1.3 _{\pm 1.1}$ & $0.9 _{\pm 0.9}$ & $59.6 _{\pm 1.8}$ & $14.4 _{\pm 3.0}$ & $6.6 _{\pm 2.3}$ & $51.3 _{\pm 2.1}$ & $2.3 _{\pm 0.4}$ & $0.6 _{\pm 0.2}$ & $20.3 _{\pm 0.9}$ & $0.5 _{\pm 0.2}$ & $0.0 _{\pm 0.0}$ \\
InfoGraph +SC & $23.8 _{\pm 0.5}$ & $4.6 _{\pm 0.7}$ & $2.2 _{\pm 0.4}$ & $54.9 _{\pm 1.7}$ & $0.9 _{\pm 0.6}$ & $1.0 _{\pm 0.8}$ & $60.9 _{\pm 2.5}$ & $15.4 _{\pm 3.3}$ & $9.3 _{\pm 3.5}$ & $50.8 _{\pm 1.3}$ & $1.6 _{\pm 0.6}$ & $0.6 _{\pm 0.0}$ & $24.7 _{\pm 1.3}$ & $4.8 _{\pm 0.6}$ & $3.2 _{\pm 0.6}$ \\
GraphCL +KM & $21.5 _{\pm 0.2}$ & $1.6 _{\pm 0.1}$ & $0.9 _{\pm 0.1}$ & $55.4 _{\pm 1.7}$ & $0.5 _{\pm 0.3}$ & $1.0 _{\pm 0.9}$ & $58.0 _{\pm 1.2}$ & $17.8 _{\pm 2.0}$ & $11.3 _{\pm 0.6}$ & $51.9 _{\pm 3.3}$ & $3.4 _{\pm 1.2}$ & $0.2 _{\pm 0.0}$ & $25.3 _{\pm 0.9}$ & $5.3 _{\pm 0.3}$ & $4.3 _{\pm 0.6}$ \\
GraphCL +SC & $25.3 _{\pm 0.3}$ & $4.8 _{\pm 0.4}$ & $2.0 _{\pm 0.3}$ & $50.8 _{\pm 1.6}$ & $0.6 _{\pm 0.6}$ & $1.1 _{\pm 0.8}$ & $57.8 _{\pm 0.6}$ & $17.0 _{\pm 1.3}$ & $10.1 _{\pm 0.7}$ & $55.9 _{\pm 2.1}$ & $3.2 _{\pm 1.0}$ & $0.3 _{\pm 0.2}$ & $27.3 _{\pm 1.3}$ & $5.4 _{\pm 0.8}$ & $4.2 _{\pm 1.1}$ \\
JOAO + KM & $21.7 _{\pm 0.4}$ & $4.9 _{\pm 0.4}$ & $2.1 _{\pm 0.2}$ & $51.1 _{\pm 0.4}$ & $0.4 _{\pm 0.2}$ & $0.1 _{\pm 0.0}$ & $58.3 _{\pm 1.5}$ & $18.7 _{\pm 2.6}$ & $11.1 _{\pm 1.8}$ & $54.3 _{\pm 2.9}$ & $4.2 _{\pm 1.8}$ & $0.8 _{\pm 0.3}$ & $26.6 _{\pm 0.6}$ & $3.6 _{\pm 1.2}$ & $2.5 _{\pm 0.2}$ \\
JOAO + SC & $24.4 _{\pm 1.4}$ & $3.2 _{\pm 0.7}$ & $1.7 _{\pm 0.8}$ & $51.5 _{\pm 3.0}$ & $0.9 _{\pm 1.2}$ & $0.4 _{\pm 1.2}$ & $58.2 _{\pm 0.9}$ & $17.1 _{\pm 2.1}$ & $10.6 _{\pm 0.8}$ & $55.9 _{\pm 1.2}$ & $6.7 _{\pm 2.0}$ & $1.4 _{\pm 0.6}$ & $25.6 _{\pm 0.6}$ & $2.5 _{\pm 0.2}$ & $3.4 _{\pm 0.3}$ \\
GLCC  & $24.4 _{\pm 1.4}$ & $3.2 _{\pm 0.7}$ & $1.7 _{\pm 0.8}$ & $60.9 _{\pm 2.3}$ & $5.3 _{\pm 1.9}$ & $3.6 _{\pm 2.6}$ & $60.3 _{\pm 0.6}$ & $18.2 _{\pm 1.3}$ & $12.1 _{\pm 0.9}$ & $\textcolor{orange}{\mathbf{67.6}} _{\pm 3.4}$ & $9.2 _{\pm 2.6}$ & $8.7 _{\pm 1.7}$ & $32.4 _{\pm 2.1}$ & $11.8 _{\pm 1.3}$ & $8.2 _{\pm 1.6}$ \\
\midrule
AMGC & $\textcolor{orange}{\mathbf{26.7}} _{\pm 2.0}$ & $\textcolor{orange}{\mathbf{5.2}} _{\pm 1.3}$ & $\textcolor{orange}{\mathbf{2.8}} _{\pm 0.7}$ & $\textcolor{orange}{\mathbf{62.7}} _{\pm 3.0}$ & $\textcolor{orange}{\mathbf{6.4}} _{\pm 1.9}$ & $\textcolor{orange}{\mathbf{6.4}} _{\pm 3.6}$ & $\textcolor{orange}{\mathbf{61.2}} _{\pm 1.0}$ & $\textcolor{orange}{\mathbf{20.5}} _{\pm 1.6}$ & $12.9 _{_{\pm 0.9}}$ & $64.3 _{\pm 1.9}$ & $\textcolor{orange}{\mathbf{12.1}} _{\pm 3.3}$ & $\textcolor{orange}{\mathbf{10.5}} _{\pm 2.7}$ & $\textcolor{orange}{\mathbf{35.5}} _{\pm 2.3}$ & $\textcolor{orange}{\mathbf{16.1}} _{\pm 0.9}$ & $\textcolor{orange}{\mathbf{12.0}} _{\pm 0.7}$ \\
\midrule
GraphVec & $\textcolor{red}{\mathbf{29.1}} _{\pm 0.4}$ & $\textcolor{red}{\mathbf{7.7}} _{\pm 0.3}$ & $\textcolor{red}{\mathbf{3.5}} _{\pm 0.2}$ & $\textcolor{red}{\mathbf{64.8}} _{\pm 0.0}$ & $\textcolor{red}{\mathbf{6.5}} _{\pm 0.0}$ & $\textcolor{red}{\mathbf{8.7}} _{\pm 0.0}$ & $\textcolor{red}{\mathbf{61.8}} _{\pm 0.0}$ & $\textcolor{red}{\mathbf{21.2}} _{\pm 0.0}$ & $\textcolor{red}{\mathbf{18.9}}_{\pm 0.0}$ & $\textcolor{red}{\mathbf{71.6}} _{\pm 0.0}$ & $\textcolor{red}{\mathbf{20.7}} _{\pm 0.0}$ & $\textcolor{red}{\mathbf{18.6}} _{\pm 0.0}$ & $\textcolor{red}{\mathbf{40.0}} _{\pm 0.2}$ & $\textcolor{red}{\mathbf{17.3}} _{\pm 0.1}$ & $\textcolor{red}{\mathbf{12.1}} _{\pm 0.2}$ \\
\bottomrule
\end{tabular}}
\end{table*}
\subsection{Visualization of aligned node embeddings}
To qualitatively inspect whether the proposed alignment places node embeddings from heterogeneous datasets into a comparable coordinate system, we visualize the aligned node embeddings from biochemical, social-network, and computer-vision datasets using t-SNE. As shown in Figure \ref{fig:tsne_aligned_node_features}, the embeddings from different datasets are not separated into completely isolated domain-specific regions. Instead, several datasets occupy overlapping or adjacent regions in the projected space, suggesting that the global multi-graph construction and density-maximization mean alignment reduce the gross distributional mismatch caused by heterogeneous raw attributes and feature dimensions.

The visualization also shows that the aligned space does not collapse all datasets into an indistinguishable distribution. Some datasets still form relatively compact local clusters or dataset-specific bands, which is expected because different domains preserve different structural and semantic regularities. Thus, the goal of the alignment is not to erase domain information, but to map node-level features into a shared relational space where cross-domain training becomes feasible. Since t-SNE is a qualitative projection and can distort global distances, this figure should be viewed as supporting evidence complementary to the downstream classification, clustering, and ablation results.

\begin{figure}[h!]
    \centering
    \includegraphics[width=0.5\linewidth]{ 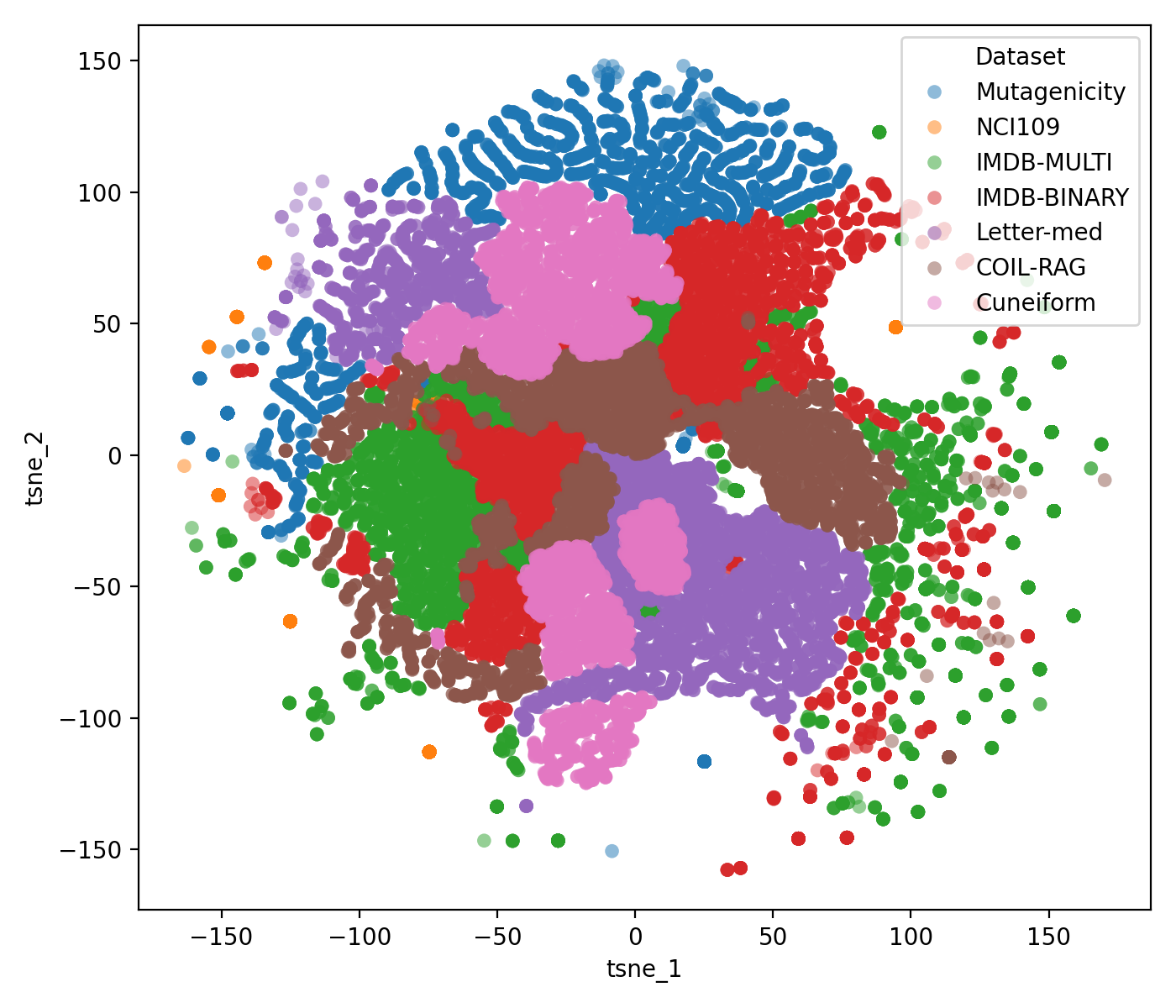}
    \caption{T-SNE visualization of aligned node embeddings of datasets from different domains.}
    \label{fig:tsne_aligned_node_features}
\end{figure}
We further compare the node embeddings learned by GraphVec-FM with those obtained from two representative non-LLM-based GFMs, RiemannGFM and GFT. Figure \ref{fig:tsne_model_node_embeddings} visualizes node embeddings from NCI109 and IMDB-MULTI, where colors denote datasets and markers denote graph classes. Compared with RiemannGFM and GFT, GraphVec-FM produces embeddings with stronger cross-dataset mixing while still preserving visible class-level structure. This qualitative comparison suggests that the proposed global multi-graph construction and density-based alignment better reduce dataset-specific feature mismatch, rather than merely separating nodes according to their source dataset. As with all t-SNE visualizations, this result is qualitative and should be interpreted together with the quantitative transfer results.

\begin{figure}[h!]
    \centering
    \includegraphics[width=\linewidth]{ 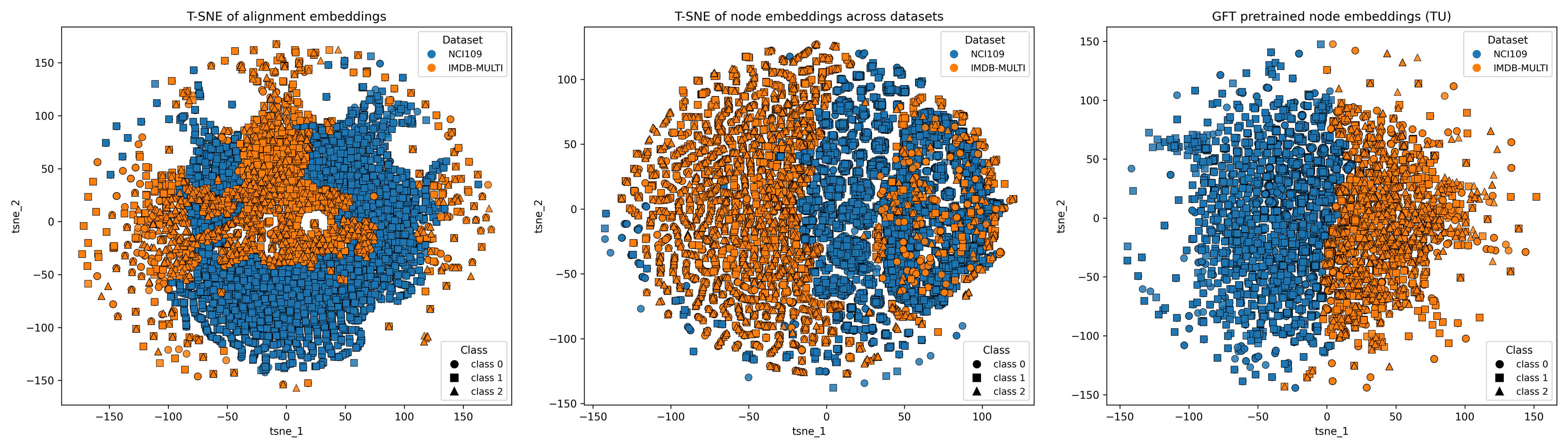}
    \caption{T-SNE visualization of node embeddings obtained by different pretrained models. From left to right are node embeddings obtained from GraphVec-FM, RiemannGFM, and GFT, respectively.}
    \label{fig:tsne_model_node_embeddings}
\end{figure}

\begin{table}[h!]
\centering
\caption{Resource consumption for large-scale dataset evaluation. Note that COLOR-3 exhibits higher RAM consumption compared to reddit\_threads, primarily due to its larger average number of nodes.}
\label{tab:scalability}

\begin{tabular}{lcccc}
\toprule
\textbf{Dataset} & \textbf{Graphs} & \textbf{Avg. Nodes} &
\textbf{RAM} & \textbf{Time} \\
\midrule
COLOR-3 & 10500 & 61.31 & 23.17GB & 626.84s \\
reddit\_threads & 203088 & 23.93 & 6.01GB & 3139.53s \\
\bottomrule
\end{tabular}
\end{table}

\subsection{Scalability to Large-scale Dataset}
\vspace{-5pt}
The construction and decomposition of global multi-graphs become computationally intensive when the downstream dataset contains a large number of graphs and nodes. To ensure scalability, except for employing the Nyström approximation (mentioned in Section \ref{sec_graph}) to handle large graphs efficiently and reduce the complexity of kernel matrix operations, we also reduce the computation cost by splitting datasets into small batches and computing the mini-batch global graphs.

To further validate the scale ability during evaluation, we test the time and memory consumption to evaluate on two large datasets COLOR-3 and reddit\_threads  \citep{Morris+2020} with more than 10k graphs and 20k graphs respectively. For these experiments, each dataset is divided into blocks of 128 graphs to build the corresponding multi‑graphs. The resulting resource usage is summarized in Table \ref{tab:scalability}, demonstrating that the overhead remains manageable at this scale. 
\subsection{Ablation Study}

\label{app_abl}
To verify the effectiveness of our proposed methods and modules, we conduct ablation study on global multi-graph construction, mean alignment algorithm, and reference layer. For the global multi-graph, we vary the number of multi-graphs from 1-6. Figure \ref{fig:attr_acc_vs_multigraphs} and Figure \ref{fig:woattr_acc_vs_multigraphs} demonstrate the impact of the number of kernel parameters on downstream few-shot graph classification accuracy. These results were obtained by incrementally increasing the set of Gaussian kernel bandwidths from $[0.25]$ to the full set $[0.25, 0.5, 1, 2, 5, 10]$ used in our main experiments. It can be observed that classification accuracy improves with a greater number of global multi-graphs, particularly for datasets with original continuous node attributes. This observation further illustrates that the global multi-graphs constructed by using different kernel parameters help capture patterns from original features. 

\begin{figure}[h!]
    \centering
    \includegraphics[width=0.8\linewidth]{ 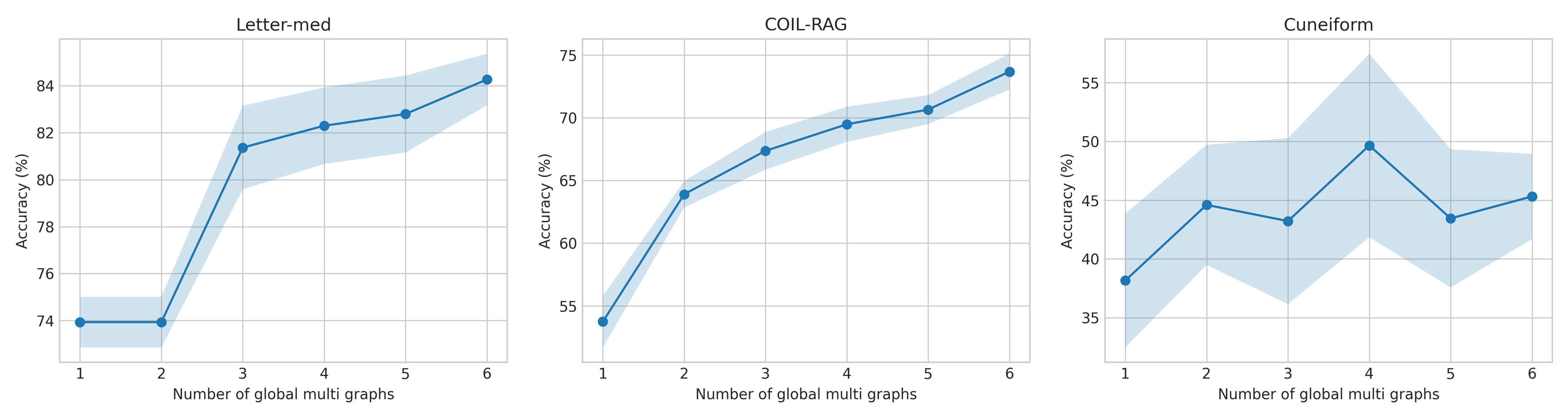}
    \caption{The few-shot graph classification accuracy in datasets with node attributes when the number of global multi-graphs increases from 1 to 6.}
    \label{fig:attr_acc_vs_multigraphs}
\end{figure}

\begin{figure}[h!]
    \centering
    \includegraphics[width=0.99\linewidth]{ 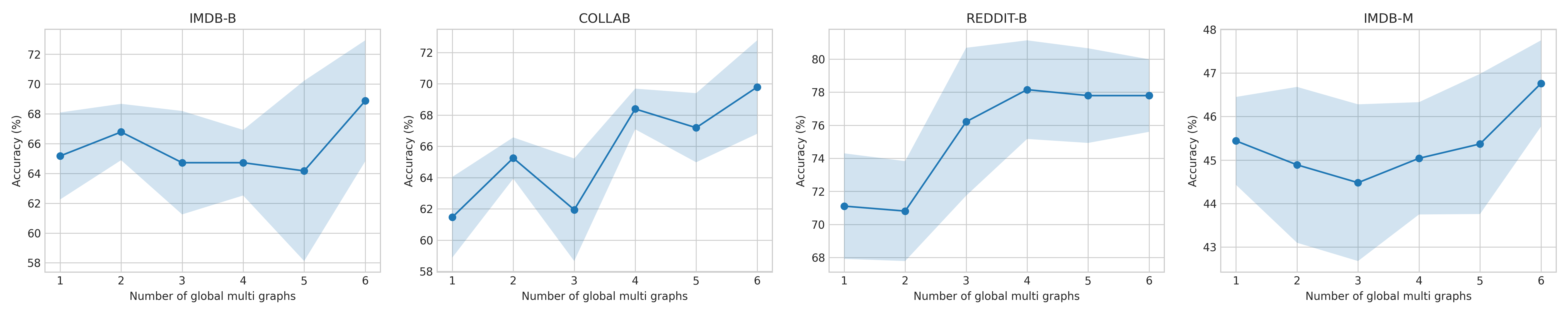}
    \caption{The few-shot graph classification accuracy in datasets without node attributes when the number of global multi-graphs increases from 1 to 6.}
    \label{fig:woattr_acc_vs_multigraphs}
\end{figure}

We conducted experiments using simple pooling without reference layers, and the results are presented in Table \ref{tab:model_ablation}. After removing the reference layer module, the performance on all datasets shows degradation, especially on COLLAB, IMDB-BINARY, and Cuneiform. By removing the mean alignment, the performance on COLLAB, Letter-Med, and Cuneiform shows an evident decrease. Similarly, we also conduct experiments that remove the alignment module and both the alignment module and the reference layer. The overall impact of the two modules is shown in Table \ref{tab:model_ablation_2}. 

\begin{table}[h!]
\centering
\small
\caption{Ablation study of reference layer and mean alignment module.}
\label{tab:model_ablation}
\resizebox{0.99\textwidth}{!}{
\begin{tabular}{l *{7}{c}}
\toprule
& \multicolumn{7}{c}{\textbf{Dataset}} \\
\cmidrule{2-8}
Model Variant & COLLAB & REDDIT-B & IMDB-B & IMDB-M & Letter-med & COIL-RAG & Cuneiform \\
& (50-shot) & (50-shot) & (50-shot) & (50-shot) & (50-shot) & (5-shot) & (1-shot) \\
\midrule
Original Model & $\mathbf{68.09} \pm 2.99$ & $\mathbf{81.52} \pm 1.50$ & $\mathbf{68.39} \pm 4.06$ & $\mathbf{46.70} \pm 0.99$ & $\mathbf{85.60} \pm 1.44$ & $\mathbf{74.20} \pm 0.77$ & $\mathbf{55.86} \pm 8.15$ \\ 
Mean Readout Only & $65.04 \pm 2.75$ & $77.79 \pm 3.11$ & $61.78 \pm 2.59$ & $46.07 \pm 2.60$ & $83.17 \pm 0.99$ & $72.87 \pm 1.26$ & $41.04 \pm 3.64$ \\
w/o alignment & $64.90\pm1.81$ & $76.74\pm5.43$ & $66.06\pm4.44$ & $45.78\pm1.46$ & $81.87\pm1.31$ & $72.14\pm0.35$ & $42.87\pm5.18$ \\
\bottomrule
\end{tabular}}
\end{table}

\begin{table}[h!]
\centering

\caption{The individual effect of alignment algorithm and reference layer on downstream cross-domain graph classification. The reported performance is averaged on 7 datasets.}
\label{tab:model_ablation_2}
\begin{tabular}{ccc}
\hline
Alignment & Reference layer & Avg. Acc. (\%) \\
\hline
$\checkmark$ & $\checkmark$ & 68.62 \\
$\checkmark$ & $\times$      & 63.95 \\
$\times$      & $\checkmark$ & 64.37 \\
$\times$      & $\times$     & 63.16 \\ 
\hline
\end{tabular}
\end{table}


\subsection{Kernel Function Analysis}

\begin{table}[ht]
\centering
\caption{Analysis of different kernel functions used in the multi-graph construction module.
The RBF kernel is used as the default choice in GraphVec.}
\label{tab:kernel_analysis}
\resizebox{\linewidth}{!}{
\begin{tabular}{lccccccc}
\toprule
Kernel
& COLLAB
& IMDB-BINARY
& IMDB-MULTI
& REDDIT-BINARY
& Letter-med
& COIL-RAG
& Cuneiform \\
\midrule
Laplacian
& $65.28 \pm 1.38$ & $69.83 \pm 1.79$ & $46.33 \pm 2.09$ & $86.03 \pm 3.23$ & $83.47 \pm 2.05$ & $74.74 \pm 0.96$ & $47.00 \pm 4.92$ \\
Polynomial
& $56.33 \pm 3.81$ & $60.56 \pm 2.06$ & $40.56 \pm 3.24$ & $71.53 \pm 4.65$ & $82.57 \pm 1.42$ & $62.55 \pm 2.53$ & $41.43 \pm 4.79$ \\
RBF
& $68.09 \pm 2.99$ & $68.39 \pm 4.06$ & $46.70 \pm 0.99$ & $81.52 \pm 1.50$ & $85.60 \pm 1.44$ & $74.20 \pm 0.77$ & $55.86 \pm 8.15$ \\
\bottomrule
\end{tabular}
}
\end{table}

We further study the effect of different kernel functions in the multi-graph construction module. 
Specifically, we compare the Laplacian kernel, polynomial kernel, and RBF kernel for constructing the global graphs over node attributes. 
As shown in Table~\ref{tab:kernel_analysis}, the RBF kernel achieves the best or competitive performance on most datasets, including COLLAB, IMDB-MULTI, Letter-med, and Cuneiform. 
Although the Laplacian kernel performs better on IMDB-BINARY, REDDIT-BINARY, and COIL-RAG, its performance drops substantially on Cuneiform compared with the RBF kernel. 
The polynomial kernel performs consistently worse than the other two kernels, suggesting that its induced similarity may be less suitable for capturing local node-attribute relationships across heterogeneous graph domains.

Overall, the RBF kernel provides the most stable performance across datasets. 
This is consistent with the motivation of our multi-graph construction strategy: the RBF kernel captures smooth local similarity in the node attribute space and is less dependent on the absolute scale or linear structure of the original attributes. 
Therefore, we use the RBF kernel as the default kernel function in GraphVec.

\subsection{Comparison with Other Alignment Approaches}
We further compare our density-maximization mean alignment with the alignment strategy adopted in AnyGraph under the same backbone and evaluation protocol. Although both methods rely on SVD-based representations, the two alignment mechanisms are fundamentally different. AnyGraph applies SVD directly to original node attributes and is mainly designed to handle heterogeneous feature dimensions, whereas our method first constructs kernelized global graphs and then aligns the resulting relational embeddings by explicitly maximizing the density of dataset means under orthogonal transformations.

The results in Table~\ref{tab:anygraph_alignment} show that our alignment strategy improves performance on all six cross-domain benchmarks. The gain is modest on IMDB-MULTI, but it becomes substantial on datasets with larger domain gaps, especially the computer vision datasets Letter-med, COIL-RAG, and Cuneiform. In particular, the large margin on COIL-RAG and Cuneiform suggests that directly aligning raw attribute spaces is not sufficient when datasets differ strongly in feature semantics, while our density-based alignment is better at matching datasets in a shared relational space. These results provide empirical support that the proposed alignment module is an important contributor to the transferability of GraphVec across domains.

\begin{table}[h!]
\centering
\caption{Comparison between the alignment strategy used in AnyGraph and our density-maximization mean alignment.
Our alignment consistently improves cross-domain graph classification performance, especially on computer vision graph datasets.}
\label{tab:anygraph_alignment}
\resizebox{\linewidth}{!}{
\begin{tabular}{lcccccc}
\toprule
Alignment Method
& COLLAB
& IMDB-BINARY
& IMDB-MULTI
& Letter-med
& COIL-RAG
& Cuneiform \\
\midrule
AnyGraph's Alignment
& $62.93 \pm 2.27$ & $63.61 \pm 2.92$ & $45.33 \pm 1.93$ & $72.17 \pm 1.25$ & $29.05 \pm 3.35$ & $10.46 \pm 2.09$ \\
Our Alignment
& $68.09 \pm 2.99$ & $68.39 \pm 4.06$ & $46.70 \pm 0.99$ & $85.60 \pm 1.44$ & $74.20 \pm 0.77$ & $55.86 \pm 8.15$ \\
\bottomrule
\end{tabular}
}
\end{table}

\subsection{Hyperparameter Analysis}

\begin{table}[ht]
\centering
\caption{Hyperparameter sensitivity analysis of the contrastive temperature $\tau$. 
The underlined column denotes the default setting used in our experiments.}
\label{tab:sensitivity_tau}
\resizebox{\linewidth}{!}{
\begin{tabular}{lccccc}
\toprule
Dataset 
& $\tau=0.05$ 
& $\underline{\tau=0.1}$ 
& $\tau=0.15$ 
& $\tau=0.2$ 
& $\tau=0.5$ \\
\midrule
IMDB-MULTI     
& $46.93 \pm 0.83$ & $46.70 \pm 0.99$ & $46.78 \pm 1.11$ & $48.11 \pm 0.55$ & $46.81 \pm 1.17$ \\
IMDB-BINARY    
& $67.83 \pm 1.53$ & $68.39 \pm 4.06$ & $68.22 \pm 3.15$ & $67.06 \pm 2.01$ & $68.89 \pm 4.09$ \\
COLLAB         
& $66.10 \pm 2.35$ & $68.09 \pm 2.99$ & $68.06 \pm 1.32$ & $66.63 \pm 0.96$ & $66.46 \pm 2.19$ \\
Letter-med     
& $85.37 \pm 0.73$ & $85.60 \pm 1.44$ & $85.20 \pm 0.83$ & $85.73 \pm 0.98$ & $86.03 \pm 1.62$ \\
REDDIT-BINARY  
& $86.19 \pm 2.29$ & $81.52 \pm 1.50$ & $86.06 \pm 1.29$ & $84.14 \pm 2.24$ & $86.71 \pm 1.13$ \\
COIL-RAG       
& $74.35 \pm 0.46$ & $74.20 \pm 0.77$ & $72.86 \pm 1.10$ & $71.11 \pm 1.45$ & $73.97 \pm 1.36$ \\
Cuneiform      
& $51.73 \pm 3.25$ & $55.86 \pm 8.15$ & $50.13 \pm 4.97$ & $48.95 \pm 5.25$ & $51.48 \pm 3.65$ \\
\bottomrule
\end{tabular}
}
\end{table}

\begin{table}[ht]
\centering
\caption{Hyperparameter sensitivity analysis of the SVD embedding dimension $d$.
The underlined column denotes the default setting used in our experiments.}
\label{tab:sensitivity_d}
\resizebox{\linewidth}{!}{
\begin{tabular}{lccccc}
\toprule
Dataset
& $d=8$
& $d=16$
& $\underline{d=32}$
& $d=64$
& $d=128$ \\
\midrule
IMDB-MULTI
& $46.00 \pm 1.97$ & $47.37 \pm 1.06$ & $46.70 \pm 0.99$ & $45.37 \pm 1.40$ & $43.30 \pm 3.93$ \\
IMDB-BINARY
& $67.06 \pm 2.60$ & $65.72 \pm 1.24$ & $68.39 \pm 4.06$ & $68.56 \pm 1.85$ & $68.00 \pm 1.38$ \\
COLLAB
& $64.64 \pm 1.96$ & $66.74 \pm 1.51$ & $68.09 \pm 2.99$ & $67.92 \pm 1.31$ & $69.53 \pm 1.37$ \\
Letter-med
& $80.50 \pm 2.14$ & $82.10 \pm 1.70$ & $85.60 \pm 1.44$ & $85.50 \pm 1.27$ & $86.60 \pm 1.01$ \\
REDDIT-BINARY
& $79.42 \pm 1.59$ & $83.01 \pm 4.56$ & $81.52 \pm 1.50$ & $82.80 \pm 2.07$ & $78.75 \pm 3.24$ \\
COIL-RAG
& $64.97 \pm 1.43$ & $70.72 \pm 1.12$ & $74.20 \pm 0.77$ & $40.05 \pm 29.18$ & $3.37 \pm 0.96$ \\
Cuneiform
& $52.91 \pm 5.25$ & $51.81 \pm 6.39$ & $55.86 \pm 8.15$ & $20.17 \pm 19.50$ & $6.67 \pm 2.04$ \\
\bottomrule
\end{tabular}
}
\end{table}

\begin{table}[ht]
\centering
\caption{Hyperparameter sensitivity analysis of $\gamma$.
The underlined column denotes the default setting used in our experiments.}
\label{tab:sensitivity_gamma}
\resizebox{\linewidth}{!}{
\begin{tabular}{lccccc}
\toprule
Dataset
& $\underline{\gamma=0.01}$
& $\gamma=0.1$
& $\gamma=0.5$
& $\gamma=1.0$
& $\gamma=10.0$ \\
\midrule
IMDB-MULTI
& $46.70 \pm 0.99$ & $44.35 \pm 1.20$ & $45.93 \pm 2.22$ & $42.22 \pm 0.56$ & $46.57 \pm 1.02$ \\
IMDB-BINARY
& $68.39 \pm 4.06$ & $68.19 \pm 0.14$ & $67.78 \pm 0.28$ & $68.89 \pm 1.94$ & $68.33 \pm 0.83$ \\
COLLAB
& $68.09 \pm 2.99$ & $63.20 \pm 0.21$ & $65.41 \pm 2.63$ & $65.77 \pm 0.31$ & $66.91 \pm 0.93$ \\
Letter-med
& $85.60 \pm 1.44$ & $85.17 \pm 0.17$ & $85.50 \pm 0.50$ & $83.58 \pm 1.25$ & $85.50 \pm 0.33$ \\
REDDIT-BINARY
& $81.52 \pm 1.50$ & $82.63 \pm 2.26$ & $84.16 \pm 0.72$ & $80.53 \pm 2.61$ & $83.82 \pm 2.23$ \\
COIL-RAG
& $74.20 \pm 0.77$ & $74.19 \pm 2.58$ & $73.96 \pm 0.65$ & $75.31 \pm 1.31$ & $76.27 \pm 0.27$ \\
Cuneiform
& $55.86 \pm 8.15$ & $52.32 \pm 0.00$ & $56.12 \pm 1.27$ & $59.92 \pm 2.95$ & $55.27 \pm 0.42$ \\
\bottomrule
\end{tabular}
}
\end{table}

\begin{table}[h!]
\centering
\caption{Hyperparameter sensitivity analysis of the number of reference distributions $R$.
The underlined column denotes the default setting used in our experiments.}
\label{tab:sensitivity_R}
\resizebox{\linewidth}{!}{
\begin{tabular}{lccccc}
\toprule
Dataset
& $R=8$
& $R=16$
& $R=32$
& $\underline{R=64}$
& $R=128$ \\
\midrule
IMDB-MULTI
& $45.41 \pm 0.75$ & $45.11 \pm 2.52$ & $44.59 \pm 0.81$ & $46.70 \pm 0.99$ & $47.00 \pm 1.30$ \\
IMDB-BINARY
& $68.56 \pm 1.57$ & $67.17 \pm 1.35$ & $67.72 \pm 1.82$ & $68.39 \pm 4.06$ & $70.39 \pm 1.59$ \\
COLLAB
& $66.04 \pm 2.08$ & $66.10 \pm 2.96$ & $67.09 \pm 1.12$ & $68.09 \pm 2.99$ & $66.35 \pm 0.72$ \\
Letter-med
& $84.67 \pm 1.18$ & $84.57 \pm 1.51$ & $85.27 \pm 1.83$ & $85.60 \pm 1.44$ & $85.33 \pm 1.48$ \\
REDDIT-BINARY
& $81.92 \pm 2.18$ & $82.76 \pm 4.63$ & $83.51 \pm 1.55$ & $81.52 \pm 1.50$ & $84.90 \pm 3.41$ \\
COIL-RAG
& $73.03 \pm 1.43$ & $74.14 \pm 2.46$ & $73.78 \pm 1.40$ & $74.20 \pm 0.77$ & $74.08 \pm 0.65$ \\
Cuneiform
& $53.76 \pm 5.30$ & $54.68 \pm 3.89$ & $53.67 \pm 4.05$ & $55.86 \pm 8.15$ & $55.36 \pm 2.77$ \\
\bottomrule
\end{tabular}
}
\end{table}

\begin{table}[h!]
\centering
\caption{Hyperparameter sensitivity analysis of the backbone depths $(l_{\mathrm{GIN}}, l_{\mathrm{GT}})$.
The underlined column denotes the default setting used in our experiments.}
\label{tab:sensitivity_depth}
\resizebox{\linewidth}{!}{
\begin{tabular}{lccccc}
\toprule
Dataset
& $\underline{l_{\mathrm{GIN}}=3,\ l_{\mathrm{GT}}=3}$
& $l_{\mathrm{GIN}}=4,\ l_{\mathrm{GT}}=3$
& $l_{\mathrm{GIN}}=5,\ l_{\mathrm{GT}}=3$
& $l_{\mathrm{GIN}}=3,\ l_{\mathrm{GT}}=4$
& $l_{\mathrm{GIN}}=3,\ l_{\mathrm{GT}}=5$ \\
\midrule
IMDB-MULTI
& $46.70 \pm 0.99$ & $46.73 \pm 1.22$ & $46.91 \pm 1.37$ & $46.98 \pm 1.52$ & $46.67 \pm 1.32$ \\
IMDB-BINARY
& $68.39 \pm 4.06$ & $67.22 \pm 2.16$ & $60.56 \pm 2.56$ & $69.44 \pm 1.18$ & $69.44 \pm 1.20$ \\
COLLAB
& $68.09 \pm 2.99$ & $62.37 \pm 6.01$ & $65.91 \pm 2.74$ & $66.43 \pm 0.42$ & $68.56 \pm 0.72$ \\
Letter-med
& $85.60 \pm 1.44$ & $83.78 \pm 1.73$ & $81.89 \pm 0.48$ & $86.11 \pm 0.98$ & $87.28 \pm 0.61$ \\
REDDIT-BINARY
& $81.52 \pm 1.50$ & $82.70 \pm 1.99$ & $84.43 \pm 0.56$ & $82.49 \pm 2.77$ & $84.24 \pm 2.74$ \\
COIL-RAG
& $74.20 \pm 0.77$ & $69.44 \pm 1.13$ & $65.46 \pm 0.38$ & $71.97 \pm 0.32$ & $72.41 \pm 1.65$ \\
Cuneiform
& $55.86 \pm 8.15$ & $49.23 \pm 5.60$ & $46.27 \pm 1.70$ & $55.84 \pm 5.02$ & $53.16 \pm 5.21$ \\
\bottomrule
\end{tabular}
}
\end{table}

We analyze the sensitivity of GraphVec with respect to five key hyperparameters: the contrastive temperature $\tau$, the SVD embedding dimension $d$, the Gaussian kernel parameter $\gamma$, the number of reference distributions $R$, and the backbone depths $(l_{\mathrm{GIN}}, l_{\mathrm{GT}})$. Specifically, $\tau$ controls the sharpness of the supervised contrastive objective, $d$ denotes the dimensionality of the SVD-based node embeddings used in the multi-graph feature alignment module, $\gamma$ controls the bandwidth of the Gaussian kernel used in the density-based alignment and the reference-distribution similarity, $R$ is the number of learnable reference distributions in the reference distribution module, and $l_{\mathrm{GIN}}$ and $l_{\mathrm{GT}}$ denote the numbers of GIN and graph transformer layers, respectively.

For the contrastive temperature $\tau$, the performance is generally stable across a broad range. The default value $\tau=0.1$ achieves competitive results on most datasets, e.g., $68.39$ on IMDB-BINARY, $68.09$ on COLLAB, $85.60$ on Letter-med, and $74.20$ on COIL-RAG. Although a few datasets obtain slightly higher scores with other values, such as IMDB-MULTI at $\tau=0.2$ and REDDIT-BINARY at $\tau=0.5$, the differences are mostly moderate. This indicates that the contrastive objective is not overly sensitive to the temperature, and $\tau=0.1$ provides a reasonable default trade-off across datasets.

The embedding dimension $d$ has a more pronounced effect. Moderate dimensions, especially $d=32$, perform consistently well and are used as the default setting. Increasing $d$ beyond this point does not consistently improve performance and can severely degrade results on some datasets. For example, COIL-RAG drops from $74.20$ at $d=32$ to $40.05$ at $d=64$ and $3.37$ at $d=128$, while Cuneiform drops from $55.86$ to $20.17$ and $6.67$, respectively. This suggests that overly high-dimensional SVD features may introduce noise or instability, particularly on small or structurally sparse datasets. Therefore, $d=32$ is a justified default: it is sufficiently expressive while avoiding the instability observed with larger dimensions.

For $\gamma$, the model is relatively robust within the tested range. In GraphVec, $\gamma$ is the hyperparameter in the density-maximization mean alignment objective, which controls the strength of the exponential weighting over the distances between aligned mean embeddings. The default value $\gamma=0.01$ gives strong results on IMDB-BINARY, COLLAB, Letter-med, COIL-RAG, and Cuneiform. Larger values occasionally improve individual datasets, such as COIL-RAG at $\gamma=10.0$ and Cuneiform at $\gamma=1.0$, but they do not yield consistent gains across all datasets. This suggests that using a very large $\gamma$ may over-emphasize small pairwise differences among mean embeddings during alignment, making the alignment process more dataset-dependent. The results therefore support using a small default value $\gamma=0.01$, which provides stable cross-domain alignment across heterogeneous graph datasets.

The number of reference distributions $R$ is also not highly sensitive. The default $R=64$ performs competitively across datasets, achieving $68.09$ on COLLAB, $85.60$ on Letter-med, $74.20$ on COIL-RAG, and $55.86$ on Cuneiform. Increasing $R$ to $128$ improves some datasets, such as IMDB-BINARY and REDDIT-BINARY, but slightly hurts others such as COLLAB and Letter-med. This suggests that a larger reference set can improve expressiveness, but the gain is dataset-dependent. We therefore choose $R=64$ as a balanced setting between representation capacity and robustness.

Finally, we evaluate the backbone depths $(l_{\mathrm{GIN}}, l_{\mathrm{GT}})$. The default setting $(3,3)$ is competitive, while deeper or more asymmetric configurations provide mixed results. For instance, $(3,5)$ improves Letter-med and REDDIT-BINARY but degrades COIL-RAG and Cuneiform compared with the default. Similarly, increasing $l_{\mathrm{GIN}}$ to $4$ or $5$ does not consistently improve performance and can substantially hurt COLLAB, COIL-RAG, and Cuneiform. These results suggest that excessive depth may introduce over-smoothing or optimization difficulty, especially under cross-domain transfer. Thus, the default shallow configuration $(l_{\mathrm{GIN}}, l_{\mathrm{GT}})=(3,3)$ is a reasonable choice for stable generalization.

Overall, GraphVec is relatively insensitive to $\tau$, $\gamma$, and $R$, while the embedding dimension $d$ and backbone depth require moderate settings. The selected default configuration $\tau=0.1$, $d=32$, $\gamma=0.01$, $R=64$, and $(l_{\mathrm{GIN}}, l_{\mathrm{GT}})=(3,3)$ provides a robust balance across social network and computer vision datasets, rather than being tuned for a single dataset.

\subsection{Robustness Evaluation on Noisy Input Graphs}
\label{app_robust}
To validate our model's performance on noisy graph data, we randomly added/deleted 10\% edges to 50\% of the test graphs during the few-shot test phase, and the results are shown in Table \ref{tab:robustness}. It can be observed that there is only a slight decrease in terms of accuracy when the input graphs are perturbed or noisy, which demonstrates the robustness of our model.
\begin{table}[h!]
\centering
\caption{Few-shot graph classification performance comparison between original input and perturbed input}
\label{tab:robustness}
\resizebox{0.99\textwidth}{!}{
\begin{tabular}{l *{6}{c}}
\toprule
Dataset  & REDDIT-B & IMDB-B & IMDB-M & Letter-med & COIL-RAG & Cuneiform \\
& 50-shot  & 50-shot & 50-shot & 50-shot & 5-shot & 1-shot \\
\midrule
Original Graphs  & $\mathbf{81.52} \pm 1.50$ & $\mathbf{68.39} \pm 4.06$ & $\mathbf{46.70} \pm 0.99$ & $\mathbf{85.60} \pm 1.44$ & $\mathbf{74.20} \pm 0.77$ & $\mathbf{55.86} \pm 8.15$ \\ 
50\% Perturbed Graphs  & $77.68 \pm 2.30$ & $65.22 \pm 1.07$ & $46.07 \pm 1.55$ & $83.47 \pm 3.22$ & $72.13 \pm 1.93$ & $44.14 \pm 3.82$ \\
\bottomrule
\end{tabular}}
\end{table}
\subsection{Time and Memory Consumption}
\label{app_time}
To evaluate the computational cost and runtime of model pretraining, we conducted experiments on two datasets of different scales: the larger deezer\_ego\_net dataset (9,629 graphs) and the smaller ENZYMES dataset (600 graphs), each for 10 epochs. The wall-clock time and peak GPU/RAM memory usage are presented in the Table \ref{tab:timemem}. We also compared the wall-clock time and memory cost of pretraining with ProNoG \citep{yu2025non} and BRIDGE \citep{yuan2025much} on the same dataset. It can be observed that our GraphVec requires less training time, especially on relatively large datasets. GraphVec demands more memory consumption, which is primarily due to the computation and storage of the global graph. 
\begin{table}[h!]
\centering
\caption{Performance and resource utilization comparison}
\label{tab:timemem}
\resizebox{1.0\textwidth}{!}{
\begin{tabular}{l *{6}{c}}
\toprule
& \multicolumn{3}{c}{\textbf{deezer\_ego\_net}} & \multicolumn{3}{c}{\textbf{ENZYMES}} \\ 
\cmidrule(lr){2-4} \cmidrule(lr){5-7}
\multirow{2}{*}{\textbf{Method}} & \makecell{wall-clock \\ time (s)} & \makecell{GPU Peak \\ Memory (GB)} & \makecell{RAM Peak \\ Memory (GB)} & \makecell{wall-clock \\ time (s)} & \makecell{GPU Peak \\ Memory (GB)} & \makecell{RAM Peak \\ Memory (GB)} \\
\midrule
ProNoG & 1194.61 & 0.12 & 1.1 & 69.54 & 0.05 & 0.96 \\
BRIDGE & 1600.28 & 0.31 & 1.27 & 116.99 & 0.18 & 1.17 \\
Our Method & 899.22 & 1.04 & 46.42 & 76.15 & 0.26 & 8.35 \\
\bottomrule
\end{tabular}}
\end{table}

To ensure fair comparison, since ProNoG and BRIDGE process 4 graphs at once, we set our model's batch size to 4. All experiments are conducted on 14 vCPU Intel(R) Xeon(R) Gold 6348 CPU with one Nvidia A800-80G GPU, CUDA 11.8.






\newpage

\section*{NeurIPS Paper Checklist}

The checklist is designed to encourage best practices for responsible machine learning research, addressing issues of reproducibility, transparency, research ethics, and societal impact. Do not remove the checklist: {\bf The papers not including the checklist will be desk rejected.} The checklist should follow the references and follow the (optional) supplemental material.  The checklist does NOT count towards the page
limit. 

Please read the checklist guidelines carefully for information on how to answer these questions. For each question in the checklist:
\begin{itemize}
    \item You should answer \answerYes{}, \answerNo{}, or \answerNA{}.
    \item \answerNA{} means either that the question is Not Applicable for that particular paper or the relevant information is Not Available.
    \item Please provide a short (1--2 sentence) justification right after your answer (even for \answerNA). 
\end{itemize}

{\bf The checklist answers are an integral part of your paper submission.} They are visible to the reviewers, area chairs, senior area chairs, and ethics reviewers. You will also be asked to include it (after eventual revisions) with the final version of your paper, and its final version will be published with the paper.

The reviewers of your paper will be asked to use the checklist as one of the factors in their evaluation. While \answerYes{} is generally preferable to \answerNo{}, it is perfectly acceptable to answer \answerNo{} provided a proper justification is given (e.g., error bars are not reported because it would be too computationally expensive'' or ``we were unable to find the license for the dataset we used''). In general, answering \answerNo{} or \answerNA{} is not grounds for rejection. While the questions are phrased in a binary way, we acknowledge that the true answer is often more nuanced, so please just use your best judgment and write a justification to elaborate. All supporting evidence can appear either in the main paper or the supplemental material, provided in appendix. If you answer \answerYes{} to a question, in the justification please point to the section(s) where related material for the question can be found.

IMPORTANT, please:
\begin{itemize}
    \item {\bf Delete this instruction block, but keep the section heading ``NeurIPS Paper Checklist"},
    \item  {\bf Keep the checklist subsection headings, questions/answers and guidelines below.}
    \item {\bf Do not modify the questions and only use the provided macros for your answers}.
\end{itemize}


\begin{enumerate}

\item {\bf Claims}
    \item[] Question: Do the main claims made in the abstract and introduction accurately reflect the paper's contributions and scope?
    \item[] Answer: \answerYes{} 
    \item[] Justification: The main claims made in the abstract and introduction accurately reflect the paper's contributions and scope

\item {\bf Limitations}
    \item[] Question: Does the paper discuss the limitations of the work performed by the authors?
    \item[] Answer: \answerYes{} 
    \item[] Justification: It has been discussed in conclusion.

\item {\bf Theory assumptions and proofs}
    \item[] Question: For each theoretical result, does the paper provide the full set of assumptions and a complete (and correct) proof?
    \item[] Answer: \answerYes{} 
    \item[] Justification: All assumptions are clearly stated or referenced in the statement of any theorems.

    \item {\bf Experimental result reproducibility}
    \item[] Question: Does the paper fully disclose all the information needed to reproduce the main experimental results of the paper to the extent that it affects the main claims and/or conclusions of the paper (regardless of whether the code and data are provided or not)?
    \item[] Answer: \answerYes{} 
    \item[] Justification: The paper fully disclose all the information needed to reproduce the main experimental results.

\item {\bf Open access to data and code}
    \item[] Question: Does the paper provide open access to the data and code, with sufficient instructions to faithfully reproduce the main experimental results, as described in supplemental material?
    \item[] Answer: \answerYes{} 
    \item[] Justification: The paper provide open access to the data and code.

\item {\bf Experimental setting/details}
    \item[] Question: Does the paper specify all the training and test details (e.g., data splits, hyperparameters, how they were chosen, type of optimizer) necessary to understand the results?
    \item[] Answer: \answerYes{} 
    \item[] Justification: The paper specify all the training and test details necessary to understand the results.

\item {\bf Experiment statistical significance}
    \item[] Question: Does the paper report error bars suitably and correctly defined or other appropriate information about the statistical significance of the experiments?
    \item[] Answer: \answerYes{} 
    \item[] Justification: The paper report error bars.

\item {\bf Experiments compute resources}
    \item[] Question: For each experiment, does the paper provide sufficient information on the computer resources (type of compute workers, memory, time of execution) needed to reproduce the experiments?
    \item[] Answer: \answerYes{} 
    \item[] Justification: They have been reported in the Appendix.
    
\item {\bf Code of ethics}
    \item[] Question: Does the research conducted in the paper conform, in every respect, with the NeurIPS Code of Ethics \url{https://neurips.cc/public/EthicsGuidelines}?
    \item[] Answer: \answerYes{} 
    \item[] Justification: The research conducted in the paper conform, in every respect, with the NeurIPS Code of Ethics.

\item {\bf Broader impacts}
    \item[] Question: Does the paper discuss both potential positive societal impacts and negative societal impacts of the work performed?
    \item[] Answer: \answerNA{} 
    \item[] Justification: There is no societal impact of the work performed.
    
\item {\bf Safeguards}
    \item[] Question: Does the paper describe safeguards that have been put in place for responsible release of data or models that have a high risk for misuse (e.g., pre-trained language models, image generators, or scraped datasets)?
    \item[] Answer: \answerNA{} 
    \item[] Justification: The paper poses no such risks.

\item {\bf Licenses for existing assets}
    \item[] Question: Are the creators or original owners of assets (e.g., code, data, models), used in the paper, properly credited and are the license and terms of use explicitly mentioned and properly respected?
    \item[] Answer: \answerYes{} 
    \item[] Justification: Yes, they are properly credited.

\item {\bf New assets}
    \item[] Question: Are new assets introduced in the paper well documented and is the documentation provided alongside the assets?
    \item[] Answer: \answerNA{} 
    \item[] Justification: The paper does not release new assets.

\item {\bf Crowdsourcing and research with human subjects}
    \item[] Question: For crowdsourcing experiments and research with human subjects, does the paper include the full text of instructions given to participants and screenshots, if applicable, as well as details about compensation (if any)? 
    \item[] Answer: \answerNA{} 
    \item[] Justification: The paper does not involve crowdsourcing nor research with human subjects.

\item {\bf Institutional review board (IRB) approvals or equivalent for research with human subjects}
    \item[] Question: Does the paper describe potential risks incurred by study participants, whether such risks were disclosed to the subjects, and whether Institutional Review Board (IRB) approvals (or an equivalent approval/review based on the requirements of your country or institution) were obtained?
    \item[] Answer: \answerNA{} 
    \item[] Justification: The paper does not involve crowdsourcing nor research with human subjects.

\item {\bf Declaration of LLM usage}
    \item[] Question: Does the paper describe the usage of LLMs if it is an important, original, or non-standard component of the core methods in this research? Note that if the LLM is used only for writing, editing, or formatting purposes and does \emph{not} impact the core methodology, scientific rigor, or originality of the research, declaration is not required.
    \item[] Answer: \answerNA{} 
    \item[] Justification: The core method development in this research does not involve LLMs.

\end{enumerate}

\end{document}